\documentclass[letterpaper]{article} % DO NOT CHANGE THIS
\usepackage[preprint]{aaai2027}  % DO NOT CHANGE THIS
\usepackage[hyphens]{url}  % DO NOT CHANGE THIS
\usepackage{graphicx} % DO NOT CHANGE THIS
\usepackage{natbib}  % DO NOT CHANGE THIS AND DO NOT ADD ANY OPTIONS TO IT
\usepackage{caption} % DO NOT CHANGE THIS AND DO NOT ADD ANY OPTIONS TO IT
\usepackage{algorithm}

\usepackage{amsthm}
\newtheorem{theorem}{Theorem}

\usepackage{newfloat}
\usepackage{listings}
\DeclareCaptionStyle{ruled}{labelfont=normalfont,labelsep=colon,strut=off} % DO NOT CHANGE THIS
\floatstyle{ruled}
\newfloat{listing}{tb}{lst}{}
\floatname{listing}{Listing}

\usepackage{booktabs}
\usepackage[american]{babel}
\usepackage{mathtools}
\usepackage{amssymb}
\usepackage{subcaption}
\usepackage{xcolor}
\usepackage{float}

\usepackage{graphicx}
\usepackage[utf8]{inputenc}
\usepackage{amssymb,amsmath,array}
\usepackage{algorithm}
\usepackage{algpseudocode}
\usepackage{subcaption}
\usepackage{svg}
\usepackage{multirow} 
\usepackage{arydshln}% For dashed lines
\usepackage{xspace}
\usepackage{todonotes}
\usepackage{adjustbox}
\usepackage{colortbl}
\usepackage{xcolor}
\usetikzlibrary{arrows.meta}
\usepackage{siunitx}
\usepackage{xcolor}
\usepackage{array}
\usepackage{caption}
\usepackage{pgfplots}
\pgfplotsset{compat=1.18}
\usepackage{todonotes}
\definecolor{secondaryblue}{HTML}{007eb9}
\definecolor{textgreen}{HTML}{5e8400}
\colorlet{textblue}{secondaryblue}
\usepackage{pgfplots}
\pgfplotsset{compat=1.18}
\usepackage{placeins}

\title{Explaining Image Similarity with Automatically \\Extracted Concept Activation Vectors}
\author{
    Isaac Roberts\textsuperscript{\rm 1}\corresponding,
    Petra Bevandi\'c\textsuperscript{\rm 2},
    Alexander Schulz\textsuperscript{\rm 1},
    Barbara Hammer\textsuperscript{\rm 1} 
}
\affiliations{
    \textsuperscript{\rm 1}Bielefeld University - Faculty of Technology\\
    \textsuperscript{\rm 2}University of Zagreb - Faculty of Electrical Engineering and Computing\\
{iroberts,aschulz,bhammer}@techfak.uni-bielefeld.de, petra.bevandic@fer.hr
}

\newcommand{\compactSubSection}[1]{\vspace{0.05cm}\noindent\textbf{#1}}

\newcommand{\localRel}{e_l}
\newcommand{\groupRel}{e_G}

\begin{document}

\maketitle

\begin{abstract}
Image similarity underlies many computer vision applications, yet it is often unclear why two images receive a high or low similarity score. Existing explainability methods often rely on gradient-based attribution maps to provide local justifications for similarity. These approaches struggle to provide global insights into what specifically drives similarity in regions of an embedding space, such as texture, shape, or color. We introduce a model- and metric-agnostic framework that explains image similarity using Concept Activation Vectors (CAVs) extracted automatically via Sparse Autoencoders (SAEs). Given a pair of images, we perturb their embeddings along discovered concept directions and measure the resulting change in a chosen similarity function, yielding concept importances. For image pairs, we provide localization with concept attribution maps. We extend this procedure to group-level settings, explaining what drives similarity across a cluster of images rather than a single pair, and further, we introduce Exemplar Retrieval, aiming to recover samples with similar reasons contributing to similarity. Our experiments show that our latent perturbations are more faithful to the underlying data distribution than pixel-space baselines, and that concept importances linearly recover the true similarity score. Qualitative results further confirm the usefulness of our methods in understanding a model's individual and group similarity judgments. 

% Qualitative results further show that our method surfaces interpretable, human-recognizable concepts underlying both individual- and group-level similarity judgments. 
\end{abstract}

\section{Introduction}

\begin{figure}[t!]
    \centering
    % First subfigure
    \begin{subfigure}[c]{0.23\textwidth}
        \centering
        \includegraphics[width=0.99\textwidth]{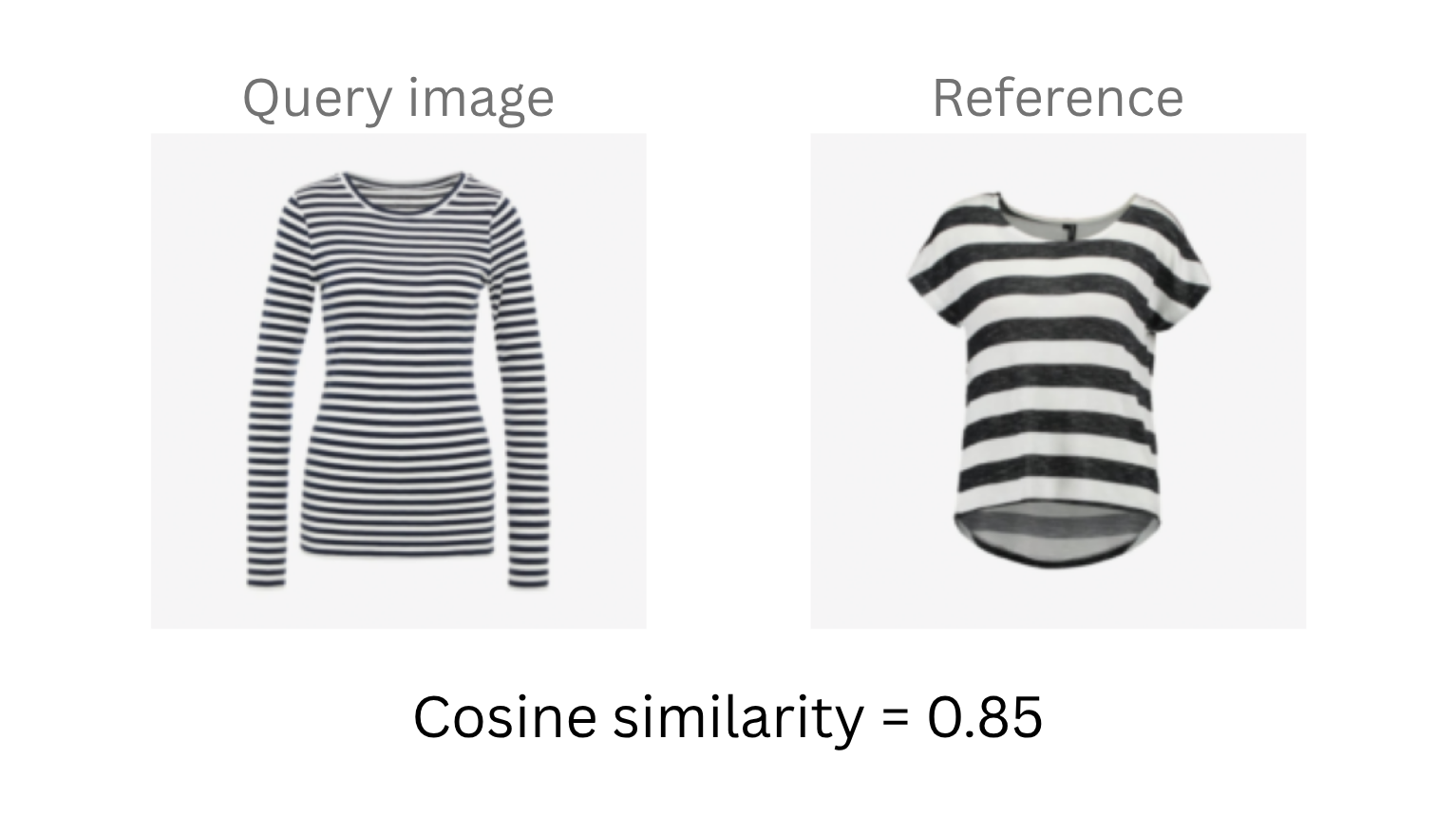}%45
        \caption{Input images}
        %\includegraphics[width=0.4\textwidth]{assets/local_exp_figure/attribute_saliency_map.pdf}
        %\caption{Input images \hspace{1.5cm} Saliency map + attribute}
    \end{subfigure}
    \begin{subfigure}[c]{0.23\textwidth}
    \centering
    \includegraphics[width=0.99\textwidth]{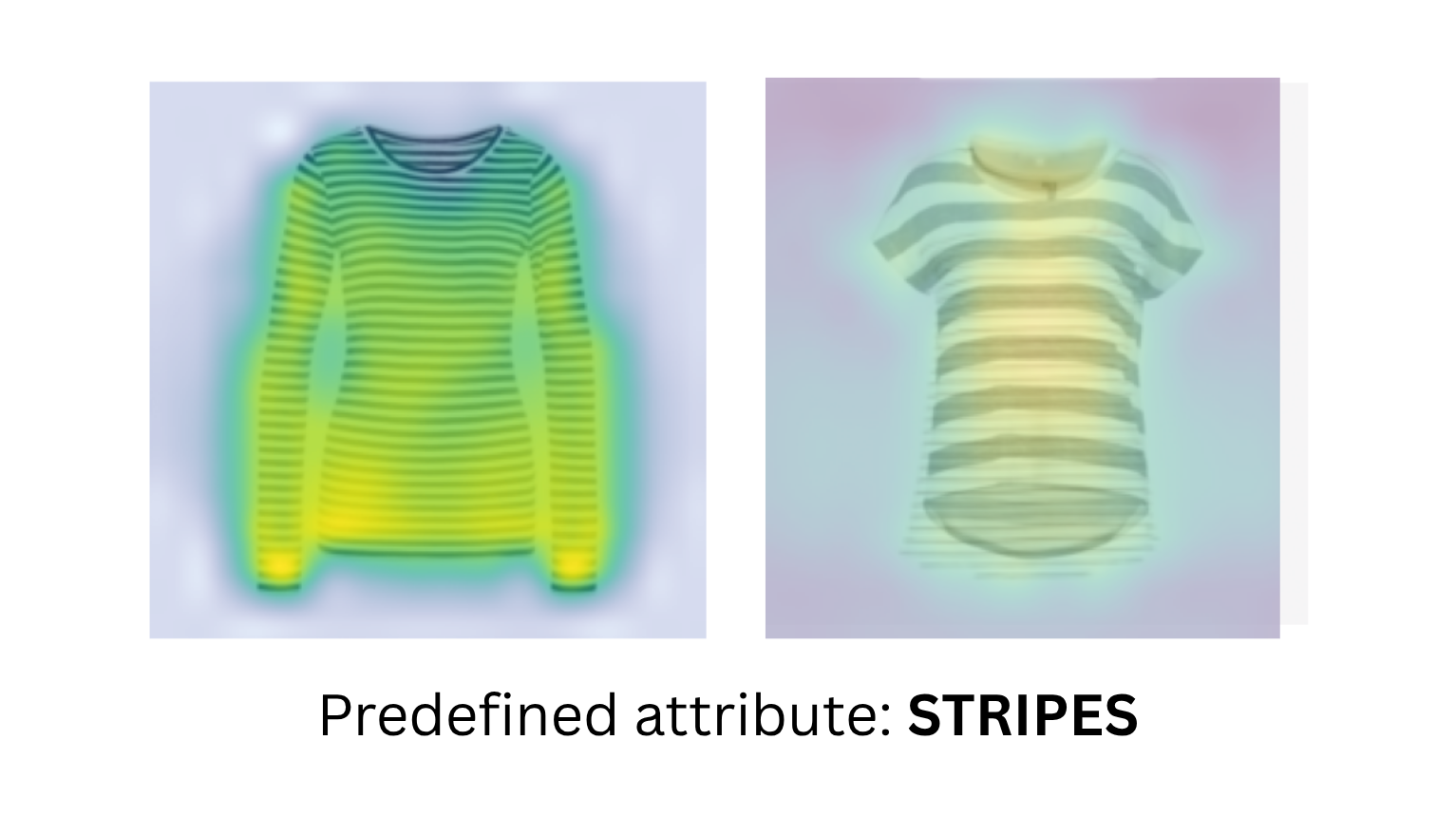}%4
    \caption{Saliency map + attribute}
    \end{subfigure}\\
    \begin{subfigure}[c]{0.48\textwidth}
        \centering
        \includegraphics[width=0.99\textwidth]{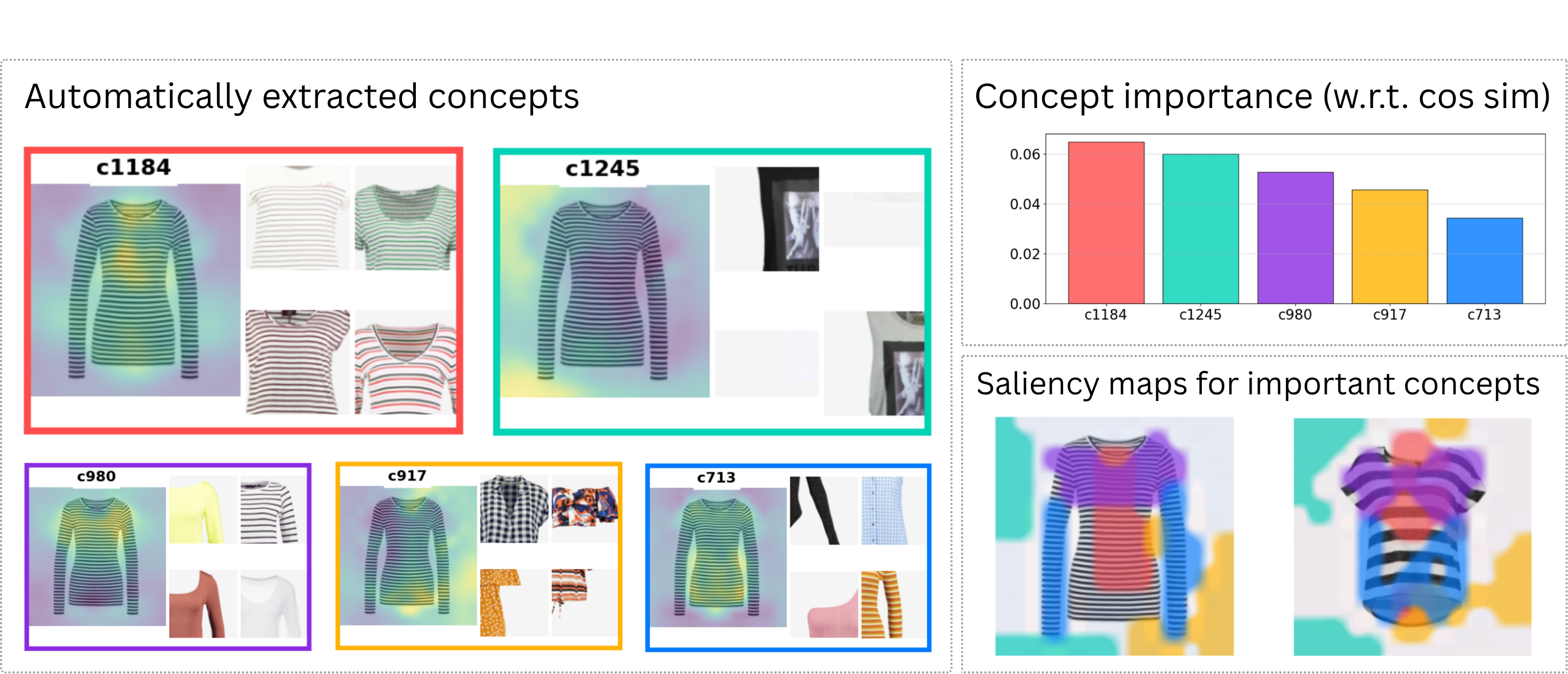}
        \caption{Concept Activation Vector based explanation}
    \end{subfigure}  
    \caption{%We illustrate our approach to explaining image similarity metrics. 
    (a) Similarity between samples is usually measured in pretrained embedding spaces. Suitable scores include euclidian distance or cosine similarity. (b) What exactly gives rise to particular similarity is not always obvious. Prior work~\cite{plummer2020these,chen2023sim2word} typically relies on identifying relevant predefined attributes and building corresponding saliency maps. (c) We propose to extract relevant concepts dynamically given a model and set of images. Semantic meaning of each concept is visualized via image crops that most strongly activate it. We rank the concepts by their impact on a chosen similarity score. We may ground the most important concepts with saliency maps.
    }
    \vspace{-0.3cm}
    \label{fig:sim_example}
\end{figure}

% importance of image similarity
Automated similarity estimation between input instances is a central tool for various AI tasks. In the vision domain, these include image retrieval \cite{radenovic2018revisiting,chen2022deep}, face recognition \cite{sun2014deep,gururaj2024comprehensive},
%zero-shot object detection,
and fashion compatibility \cite{tan2019learning,Hsiao_2018_CVPR},
with numerous further applications in the text domain.
% often done with deep metric learning
Image similarity estimation is often done in the latent space of pretrained deep networks 
%trained with deep metric learning
\cite{mohan2023deep},
where we assume a more structured and semantically
meaningful representation of image content.
% overview on deep metric learning in vision

% prior work explained primarily classification models or models with one input
Though closer inspections of pretrained deep networks may provide some insights into their behavior with respect to similarity,
%and building trust 
standard explainability (xAI) techniques often cannot be directly applied to any given similarity function, which typically require two inputs. 
Classical xAI methods, such as input attribution, were designed primarily for single-input functions such as classification. 

% local vs group explanations
%Although there do exist some extensions to also explain similarity functions
Currently, limited work  extends xAI methods to similarity functions,
\cite{opitz2025interpretable,eberle2020building},
%(see related work for more references),
%the amount of work on this is limited, 
and existing approaches are often restricted to saliency-based explanations. Notable exceptions 
%utilize concepts to 
augment saliency maps with more information
%on what is actually important in the salient regions 
about the observed region
\cite{plummer2020these,chen2023sim2word}. 
Typically, this information is derived from manually defined attributes learned from purposefully labeled data sets. Furthermore, such approaches are confined to local explanations for individual image pairs.

In this work, we propose a novel explainability method for similarity metrics based on pretrained deep models using CAVs. We perform faithful perturbations in the latent space using CAVs and derive importance by their estimated contribution to the similarity function. Our approach can be applied to any pretrained model and similarity metric defined in its embedding space, without predefined concept databases
or additional labeling effort. Finally, we demonstrate that our approach may be extended to group-level explanations, such as explaining clusters in an embedding space and enables actionable insights, which importance is highlighted in recent high-level work \cite{longo2024explainable,moshkovitz2026position}.
%Also, we do not rely on predefined concepts but extract them automatically.
Our contributions are as follows:\begin{enumerate}
    \item A \textbf{novel explainability framework} of image similarity metrics based on automatically extracted CAVs, enabling model- and metric-agnostic pairwise explanations, easily extendable for group analysis.
    \item \textbf{Similarity-based concept attribution} with the importance of each concept derived from the similarity score.
    \item Experiments demonstrating the \textbf{superiority of latent-space perturbations}, \textbf{faithfulness of our explanations}, and actionable \textbf{downstream application(s) usefulness}.
\end{enumerate}

\section{Related work} \label{sec:rel_work}
Explainability in AI has transitioned from a desirable feature to a necessity, primarily to enhance human understanding of model decision-making \cite{burkart2021XAIsurvey}. Methods are generally categorized as either ante-hoc (interpretable by design) or post-hoc (explaining a pre-trained model). Our work focuses on post-hoc, concept-based methods \cite{poeta2023concept,holistic_concepts}. These methods have gained traction because they identify not only where a model focuses, but also what semantic features it is identifying \cite{fel2023craft,poeta2023concept}. Additionally, these features (referred to as concepts) have been shown to be more human-interpretable \cite{fel2023craft,plummer2020these}. Typically, concepts are used to explore node-concept associations (identifying where they are learned) or node-class associations (relating them to output classes). In this work, we extend these ideas to explain similarities in data comparisons. 

Explaining similarity has recently gained momentum \cite{opitz2025interpretable,eberle2020building,lin2021xcos,williford2020explainable,malkiel2022interpreting,10.1145/3726302.3729971, gam1}. For instance, Layer-Wise Relevance Propagation has been adapted to decompose similarity scores \cite{eberle2020building}. In facial recognition, XCos \cite{lin2021xcos} compares patch-wise cosine maps to network attention maps, while XFR \cite{williford2020explainable} estimates pixel contributions to similarity via synthetic in-painting. Integrated Jacobians \cite{moeller2023attribution,moeller2024approximate} extend Integrated Gradients for attributions between the two inputs of the similarity function.

However, few works utilize concepts to explain similarity \cite{plummer2020these,chen2023sim2word}. Relying on saliency maps of informative attributes: SANE \cite{plummer2020these} trains a model to predict the most explanatory saliency map, while Sim2Word \cite{chen2023sim2word} applies input-space masking to observe changes in similarity. The latter, however, is limited to counterfactual examples and perturbations in the raw image space. Furthermore, while these works focus on local, pairwise explanations (why two specific images are similar), our approach provides both local and group explanations. CSIM \cite{csim} extracts concepts with SAEs, enables explainability with vector arithmetic operations, and Human Similarity Steering. Only CSIM uses dictionary learning to extract concepts and additionally assumes they intrinsically explain image similarity. Our work makes the connection between concepts and image similarity concrete by perturbing in the embedding space.

Perturbations through masking are popular due to their model-agnostic nature, but they face significant challenges: perturbed inputs often fall off the data manifold, undermining the reliability of the explanation \cite{ood_exp_paper,nieradzik2025reliable}. Additionally, masking one concept may inadvertently mask another, leading to incorrect attribution \cite{chen2023sim2word}. We propose performing these perturbations within an embedding space and demonstrate that they minimize data distribution impact.

% need to somhow bring in the fact that we can also explain dissimilarity.

% However, we argue that more contributes to the similarity than simply stripes and all of these contributions are quite difficult to capture \textit{a priori} as the attribute labels are captured. 

%\label{subsec:relWork}

%I think this paragraph should be included in some capacity depending on the route we take in the paper. 
% At the moment, we know vector space perturbations perform on par with image space pertubations but not for dissimlar concepts
% We would like to venture into showing that the manifold problem with image space pertubations **m ight can be lessend by perturbing in the vector space,

%is there a way to show that vector space pertubations are more orthogonal in the sense that they do not overlap with others?

\section{Methodology} \label{sec:pipeline}
In this section, we describe how we obtain our novel pairwise and group explanations. We introduce 3 key steps which make up our method: (1) decompose a foundation model's embeddings with dictionary learning to generate concepts, (2) use the concepts to perturb the image embeddings and measure the change in the similarity function for pairs of images, and (3) group the explanations.

\compactSubSection{Fundamentals.} We assume a dataset $X$, an embedding function $g: X \rightarrow \mathbb{R}^d$, and a similarity function $f: \mathbb{R}^d \times \mathbb{R}^d \rightarrow \mathbb{R}$. For any $\mathbf{x} \in X$, the function $g(\mathbf{x})$ yields an activation vector. Let $\mathbf{A} \in \mathbb{R}^{n \times d}$ be the matrix of activations for $n$ samples, where a row $\mathbf{a}_k$ is the embedding of the $k$-th sample.
% Further, we assume that we are operating under the Linear Representation Hypothesis (LRH) \cite{Linear_hypothesis} such that neural network representations encode abstract, interpretable features as linearly accessible.

\compactSubSection{(1) Sparse Dictionary Learning.} We decompose $\mathbf{A}$ into a weight matrix $\mathbf{U} \in \mathbb{R}^{n \times c}$ and a concept dictionary $\mathbf{V} \in \mathbb{R}^{d \times c}$. This decomposition is solved by minimizing the Frobenius norm:
\begin{equation}
(\mathbf{U}, \mathbf{V}) = \arg \min_{\mathbf{U}, \mathbf{V}} \| \mathbf{A} - \mathbf{UV}^\top \|^2_F
\end{equation}\label{eq:dict_learning}
Thereby an individual activation $\mathbf{a}_k$ (a row in $\mathbf{A}$) is approximated by a linear combination of concepts $\mathbf{a}_k \approx \sum_{i=1}^c u_{k,i} \mathbf{v}_i^\top$, where $u_{k,i}$ is the coefficient in the $k$-th row and $i$-th column of $\mathbf{U}$, and $\mathbf{v}_i^\top$ is the $i$-th row of $\mathbf{V}^\top$. We utilize different SAE formulations for this decomposition: Top-K \cite{gao2025scaling}, JumpReLU \cite{rajamanoharan2024jumping}, and Vanilla SAE \cite{bricken2023monosemanticity}.

%need to show somehow in experiment why we use this version of the explanation and not double sided at the same time, or single sided. 

\compactSubSection{(2) Perturb the Embedding.} To explain the similarity between two activations $\mathbf{a}_i$ and $\mathbf{a}_j$, we measure the importance of concept $c_k$ by masking via a perturbation strategy:
%define a perturbation strategy that masks a specific concept $c_k$. The importance of concept $c_k$ is defined as:
% \begin{equation}
% \Delta c_i = \left( f(\mathbf{a}_1, \mathbf{a}_2) - f(\mathbf{a}_1/c_i, \mathbf{a}_2) \right) + \left( f(\mathbf{a}_1, \mathbf{a}_2) - f(\mathbf{a}_1, \mathbf{a}_2/c_i) \right)
% \end{equation}

%\begin{equation}
%\Delta c_i = \underbrace{\left( f(\mathbf{a}_1, \mathbf{a}_2) - f(\mathbf{a}_1/c_i, \mathbf{a}_2) \right)}_{\text{component 1}} + \underbrace{\left( f(\mathbf{a}_1, \mathbf{a}_2) - f(\mathbf{a}_1, \mathbf{a}_2/c_i) \right)}_{\text{component 2}}
%\Delta c_k = \left( f(\mathbf{a}_i, \mathbf{a}_j) - f(\mathbf{a}_i/c_k, \mathbf{a}_j) \right) + \left( f(\mathbf{a}_i, \mathbf{a}_j) - f(\mathbf{a}_i, \mathbf{a}_j/c_k) \right)
%\label{eq:exp_computation}
%\end{equation}
\begin{align}
    \localRel^{k}(\mathbf{a}_i, \mathbf{a}_j) =& \left( f(\mathbf{a}_i, \mathbf{a}_j) - f(\mathbf{a}_i/c_k, \mathbf{a}_j) \right) \nonumber \\ 
    +& \left( f(\mathbf{a}_i, \mathbf{a}_j) - f(\mathbf{a}_i, \mathbf{a}_j/c_k) \right)
    \label{eq:exp_computation}
\end{align}
where the masked activation $\mathbf{a}_i/c_k$ is calculated by removing the contribution of the $k$-th concept:
\begin{equation}
%\mathbf{a}_k/c_i = \mathbf{a}_k - u_{k,i} \mathbf{v}_i^\top
\mathbf{a}_i/c_k = \mathbf{a}_i - u_{i,k} \mathbf{v}_k^\top
\label{eq:perturbation}
\end{equation}
and $\localRel^{k}(\mathbf{a}_i, \mathbf{a}_j)$ is the $k$-th entry of the local explanation vector $\localRel(\mathbf{a}_i, \mathbf{a}_j) \in \mathbb{R}^{c}$ for $c$ concepts, 
%By repeating this process for all $c$ concepts, we obtain a vector of importance scores 
characterizing the shared and distinct attributes of the image pair. We note that by performing the perturbations for both inputs, we obtain a \textit{symmetric} explanation such that $\localRel(\mathbf{a}_i, \mathbf{a}_j) = \localRel(\mathbf{a}_j, \mathbf{a}_i)$. This effectively strikes a balance with previous literature \cite{plummer2020these,chen2023sim2word}, while also allowing the explanation to express the respective similarity contributions from both inputs.

\compactSubSection{(3) From Local to Group} Using the local explanations obtained above, we define an explanation for a group $G$:

\begin{equation}
\groupRel = \frac{1}{\binom{|G|}{2}} \sum_{\substack{i, j \in G \\ i < j}} \localRel(\mathbf{a}_i, \mathbf{a}_j) \label{eq:group_eq}
\end{equation}
Groups can be obtained in a multitude of ways, depending on the downstream task and the underlying data. For instance, in a retrieval setting, one could fix $i$ from equation \ref{eq:group_eq} 
as the query and compute $\groupRel$ over the data points of interest, such as a group of reference images.
If a semantic grouping exists, e.g., from labels, distance-based rankings, or prior human knowledge, these constitute natural groupings from which pairs can be formed and their explanations computed to understand their similarities and dissimilarities. 

% Note about how authors from \cite{plummer2020these} talk about one-sided versus two-sided explanations.
\begin{figure}[t]
    \centering
    
    \begin{subfigure}[b]{0.15\textwidth}
        \centering
        \includegraphics[width=\textwidth]{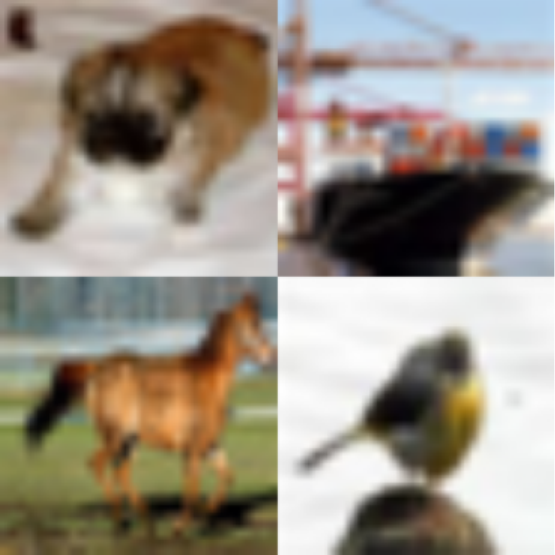}
        \caption{Original}
        \label{fig:pert_orig}
    \end{subfigure}
    \begin{subfigure}[b]{0.15\textwidth}
        \centering
        \includegraphics[width=\textwidth]{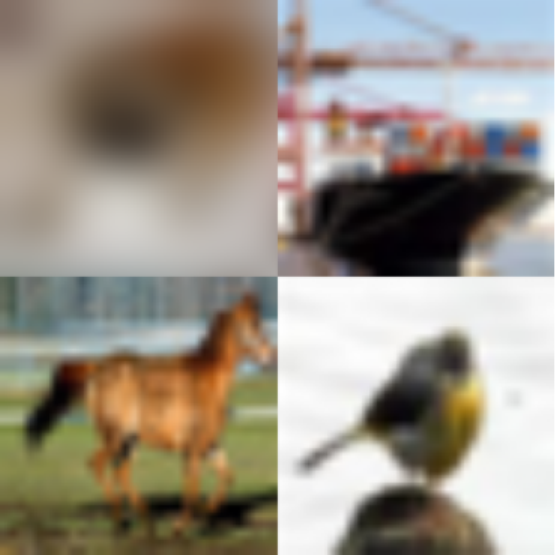}
        \caption{H. Blur}
        \label{fig:pert_blur}
    \end{subfigure}
    \begin{subfigure}[b]{0.15\textwidth}
        \centering
        \includegraphics[width=\textwidth]{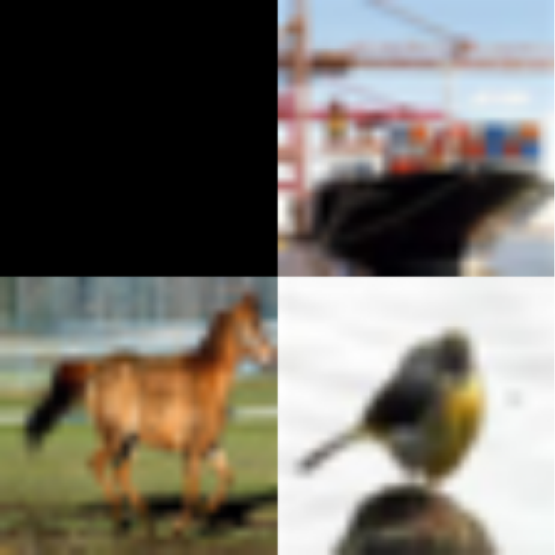}
        \caption{Black}
        \label{fig:pert_black}
    \end{subfigure}

    %\\[1ex]

    \begin{subfigure}[b]{0.15\textwidth}
        \centering
        \includegraphics[width=\textwidth]{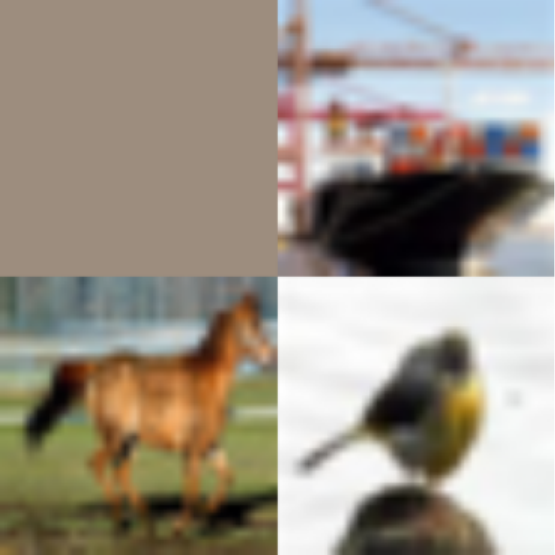}
        \caption{Q Avg.}
        \label{fig:pert_mean}
    \end{subfigure}
    \begin{subfigure}[b]{0.15\textwidth}
        \centering
        \includegraphics[width=\textwidth]{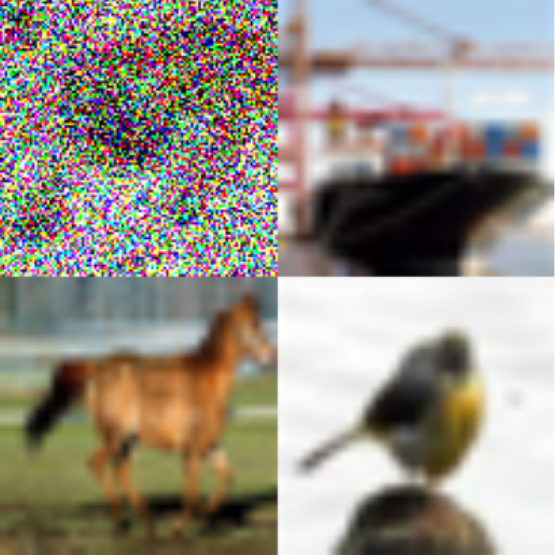}
        \caption{Random}
        \label{fig:pert_noise}
    \end{subfigure}
    \begin{subfigure}[b]{0.15\textwidth}
        \centering
        \includegraphics[width=\textwidth]{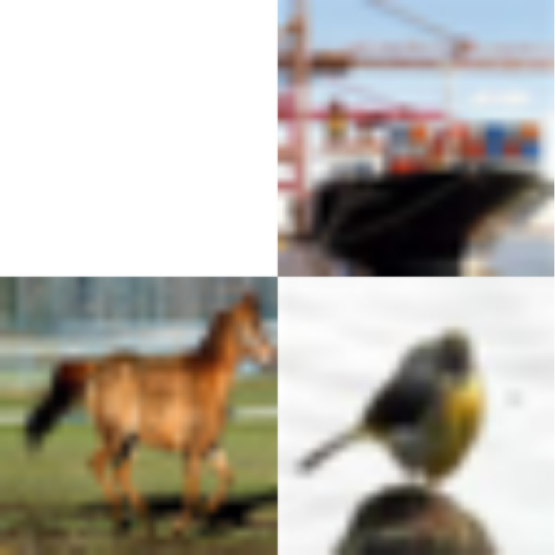}
        \caption{White}
        \label{fig:pert_white}
    \end{subfigure}
    
    \caption{Input perturbations on our synthetic
    dataset Multi-CIFAR-10 Collage applied to the top left quadrant.}
    \label{fig:all_perturbations}
\end{figure}

\compactSubSection{Visualizing the Resulting Concepts.} We visualize the concepts together with the concept importances: we depict the importances $\localRel(\mathbf{a}_i, \mathbf{a}_j)$ or $\groupRel$ in a barplot showing the most important ones; for illustration, we show a few image crops taken from the data which activate the according concept the strongest \cite{borowski2021exemplary} and additionally highlight which areas activate this concept the most, similarly as \cite{fel2025archetypal}. An example is given in Figure \ref{fig:sim_example} (c).

\section{Faithfulness Evaluation} 
\label{sec:experiments}

% we need to say here that first we show that our pertubations are faithful to the data distirbution and the metric. To show this we introduce a synthetic dataset. Our second experiment aims to show that our explanations indeed explain the similarity score. 

We validate our approach on synthetic and real-world data by examining the faithfulness of our latent perturbations and generated explanations. We first verify that the dictionaries successfully encode the target concepts. Next, we analyze the impact of perturbations on both the similarity function and the broader data distribution. Finally, we demonstrate that our explanations linearly recover the similarity function.

% To show the benefits of our approach, we examine the faithfulness of our latent perturbations and then the produced explanations using synthetic and real-world data. We begin by assessing how well the dictionaries capture the concepts. Then, we discuss the perturbation's effect on the similarity function. Next, we analyze the impact of various perturbations on the underlying data distribution.  Finally, we demonstrate that our produced explanations linearly recover the similarity function.
%first construct synthetic data to highlight the faithfulness of our latent perturbations with respect to the embeddings and a given similarity metric. 
%Next, we validate that the produced explanation vectors indeed explain the given metric. 

% \input{assets/collages}
\begin{figure}[t]
    \centering
    
    \begin{subfigure}[b]{0.22\textwidth}
        \centering
        \includegraphics[width=\textwidth]{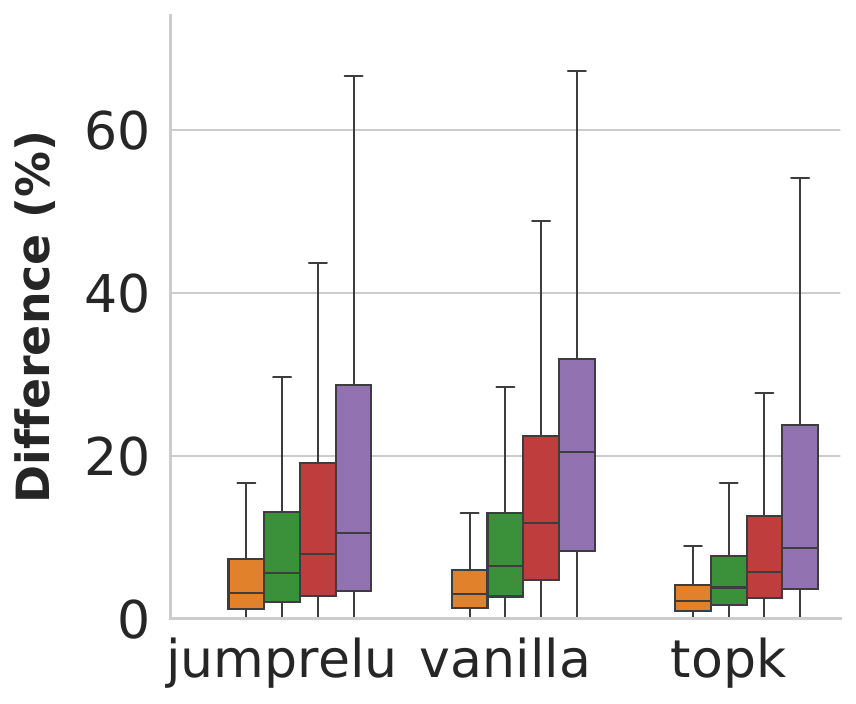}
        \caption{Cosine}
        \label{fig:cosine}
    \end{subfigure}
    %\begin{subfigure}[b]{0.15\textwidth}
    %    \centering
    %    \includegraphics[width=\textwidth]{assets/sae_pretrained_dot_product_activation_pct.pdf}
    %    \caption{Dot Product}
    %    \label{fig:dot_product}
    %\end{subfigure}
    \begin{subfigure}[b]{0.22\textwidth}
        \centering
        \includegraphics[width=\textwidth]{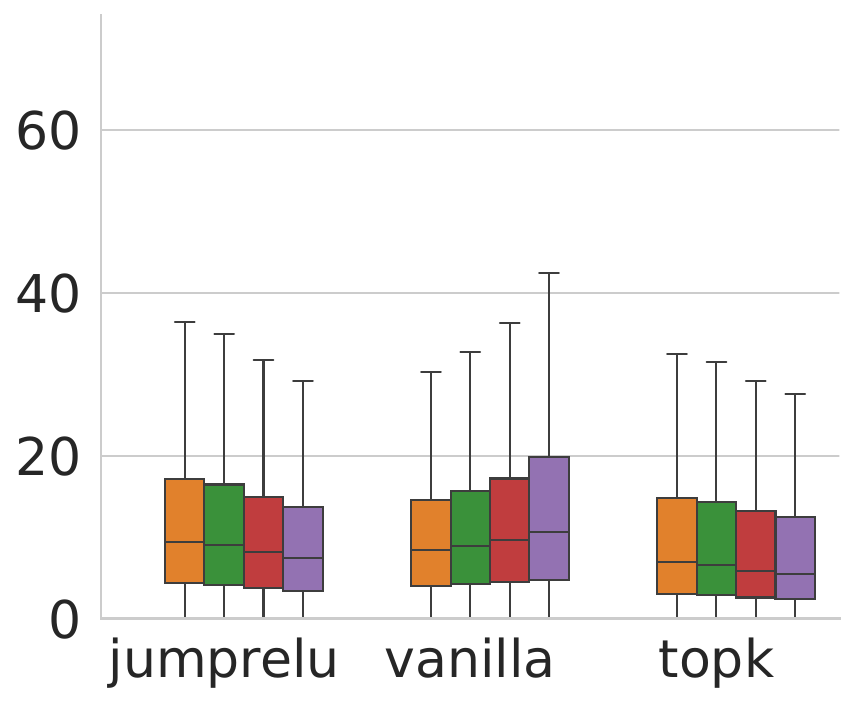}
        \caption{-Euclidean}
        \label{fig:euclidean}
    \end{subfigure}
   
    %%%%%%%%%%%%%%%%%%%%%%
    % Horizontal Legend
    %%%%%%%%%%%%%%%%%%%%%%
    \scalebox{0.8}{
    \begin{tikzpicture}
    \begin{axis}[
    hide axis,
    xmin=0, xmax=1,
    ymin=0, ymax=1,
    legend columns=5,
    legend cell align={left},
    legend style={
        font=\normalsize,
        draw=black!20,
        rounded corners=3pt,
        fill=white,
        nodes={inner sep=3pt},
        /tikz/every even column/.append style={column sep=0.6cm},
        title=No. of Image Difference(s) in Collage,
        title style={align=center},
        at={(0.5,1)},
    anchor=center
    },
]

    % --- Legend with Dots instead of Lines ---
    \definecolor{darkorange}{rgb}{1.0, 0.55, 0.0}
    \addlegendimage{darkorange, mark=*, only marks}
    \addlegendentry{2}

    \definecolor{darkspringgreen}{rgb}{0.09, 0.45, 0.27}
    \addlegendimage{darkspringgreen, mark=*, only marks}
    \addlegendentry{4}

    %\definecolor{amaranth}{rgb}{0.9, 0.17, 0.31}
    \definecolor{brightmaroon}{rgb}{0.76, 0.13, 0.28}
    %\definecolor{darkpastelred}{rgb}{0.76, 0.23, 0.13}
    %\definecolor{darkred}{rgb}{0.55, 0.0, 0.0}
    \addlegendimage{brightmaroon, mark=*, only marks}
    \addlegendentry{6}

    \definecolor{amethyst}{rgb}{0.6, 0.4, 0.8}
    \addlegendimage{amethyst, mark=*, only marks}
    \addlegendentry{8}

    \end{axis}
    \end{tikzpicture}}
  \caption{We report the average difference between the baseline similarity and the intervened similarity across hops and training scenarios.}
  \label{fig:sanity_check}
\end{figure}

\compactSubSection{Datasets.} We evaluate our perturbations and explanations performance on the real-world fashion (top body garment) \textit{VITON-HD}~\cite{Choi_2021_CVPR} dataset. Additionally, we construct the \textit{Multi-CIFAR-10 Collage} dataset (example in Figure \ref{fig:all_perturbations}), inspired by \cite{fel2025archetypal}. This dataset provides a ground truth for concept presence and mitigates the ambiguity often found in natural image datasets with poorly defined concept boundaries. Each sample $\mathbf{x} \in \mathcal{D}$ is generated by arranging four distinct images from the CIFAR-10 \cite{krizhevsky2009learning} dataset into a $2 \times 2$ grid. We define "concepts" in this context as the ten CIFAR-10 classes. Crucially, Multi-CIFAR-10 addresses limitations of image-space perturbations, where concepts may be co-localized, making a fair comparison to latent-space perturbations difficult. Furthermore, we intentionally limited intra-class variability by selecting one image per class, ensuring that similarity changes are cleanly attributable to input perturbations. %and that our synthetic ground truth behaves similarly to standard embedding space similarity metrics. 

\compactSubSection{Experiment Details.} Embeddings are obtained from six distinct vision backbones: DINOv2, DINOv3, ResNet50, ConvNeXt, ViT, and SigLIP. We repeat each experiment at least 10 times, recording from each dataset's test set. For all experiments involving Multi-CIFAR-10, we train Top-K, JumpReLU, and Vanilla SAEs on 2,000 randomly generated collages. On VITON-HD, we train on the default train split. More experimental details are available in the Appendix.  

%During training, the Vanilla and JumpReLU SAEs are optimized using a Mean Squared Error (MSE) loss combined with an $\ell_1$ sparsity penalty, whereas the Top-K SAE is trained solely using MSE. 

% For Experiment \ref{exp:ood_fid}, we construct a test set of 2,000 images and apply our perturbations to each image individually. on the complex dataset, we  In Experiment \ref{exp:metric_analysis}, we sample 5000 pairs and apply our perturbations.  To ensure statistical robustness, we repeat this entire process across 20 independent runs. Unless otherwise specified, the reported results represent an average across all evaluated SAE architectures and both embedding states (pre-trained and domain-adapted).

%\subsection{Faithfulness Evaluation} \label{exp:faithfulness}
% \input{assets/sim_metric_analysis}

\begin{table*}[]
    \centering
\begin{tabular}{
l l
S[table-format=2.2]@{\,/\,}S[table-format=2.2]
S[table-format=2.2]@{\,/\,}S[table-format=2.2]
S[table-format=2.2]@{\,/\,}S[table-format=2.2]
S[table-format=2.2]@{\,/\,}S[table-format=2.2]
S[table-format=2.2]@{\,/\,}S[table-format=2.2]
S[table-format=2.2]@{\,/\,}S[table-format=2.2]
}
\toprule

&
&
\multicolumn{12}{c}{$\mathcal{W}_1$ $\downarrow$ / OOD $\downarrow$}\\

\cmidrule(lr){3-14}

&
&
\multicolumn{2}{c}{\textbf{ConvNext}}
&
\multicolumn{2}{c}{\textbf{DINOv2}}
&
\multicolumn{2}{c}{\textbf{DINOv3}}
&
\multicolumn{2}{c}{\textbf{ResNet}}
&
\multicolumn{2}{c}{\textbf{SigLIP}}
&
\multicolumn{2}{c}{\textbf{ViT}}\\

\midrule

\multirow{7}{*}{\rotatebox{90}{Pretrained}}

& Black
&0.50&1.12
&0.45&1.39
&0.07&1.50
&0.04&1.32
&0.18&1.43
&0.59&1.48\\

& Gray
&0.50&1.08
&0.45&1.39
&0.07&1.48
&0.04&1.27
&0.18&1.43
&0.53&1.22\\

& Heavy Blur
&0.55&1.26
&0.36&1.06
&0.06&1.28
&0.04&1.27
&0.20&1.56
&0.61&1.52\\

& Quadrant Avg
&0.50&1.09
&0.44&1.35
&0.07&1.45
&0.04&1.26
&0.18&1.39
&0.54&1.29\\

& Random
&0.50&1.10
&0.55&1.71
&0.09&1.89
&0.06&2.08
&0.29&2.55
&0.53&1.22\\

& White
&0.50&1.10
&0.43&1.35
&0.07&1.43
&0.04&1.31
&0.18&1.43
&0.58&1.45\\

& Ours
&\textbf{0.46*}&\textbf{0.92*}
&\textbf{0.29*}&\textbf{0.65*}
&\textbf{0.06*}&\textbf{0.84*}
&\textbf{0.03*}&\textbf{0.86*}
&\textbf{0.14*}&\textbf{1.08*}
&\textbf{0.46*}&\textbf{0.76*}\\

\bottomrule
\end{tabular}
\caption{%\textbf{Synthetic Results.} 
    Average $\mathcal{W}_1$ and OOD results for 20 runs on Multi-CIFAR-10. Statistically significant results according to a one-sided Wilcoxon signed-rank test with Bonferroni correction for multiple comparisons are indicated by \textbf{*}.}
\label{tab:ood_perturbation_table}
\end{table*}

\begin{figure*}[h!]

\centering
\setlength{\tabcolsep}{2pt}
\renewcommand{\arraystretch}{1.2}

\newcommand{\myColW}{2.6}

% \begin{tabular}{c c c c c c c}
\begin{tabular}{m{0.75cm} >{\centering\arraybackslash} m{\myColW cm} >{\centering\arraybackslash} m{\myColW cm} >{\centering\arraybackslash} m{\myColW cm} >{\centering\arraybackslash} m{\myColW cm} >{\centering\arraybackslash} m{\myColW cm} >{\centering\arraybackslash} m{\myColW cm}}

% ---------------- Column Headers ----------------
 & \textbf{ConvNeXt} & \textbf{DINOv2} & \textbf{DINOv3} & \textbf{ResNet} & \textbf{SigLIP} & \textbf{ViT} \\%[0.5em]

% ---------------- FID Row ----------------
$\mathbf{\mathcal{W}_1}$
& \includegraphics[width=0.15\textwidth]{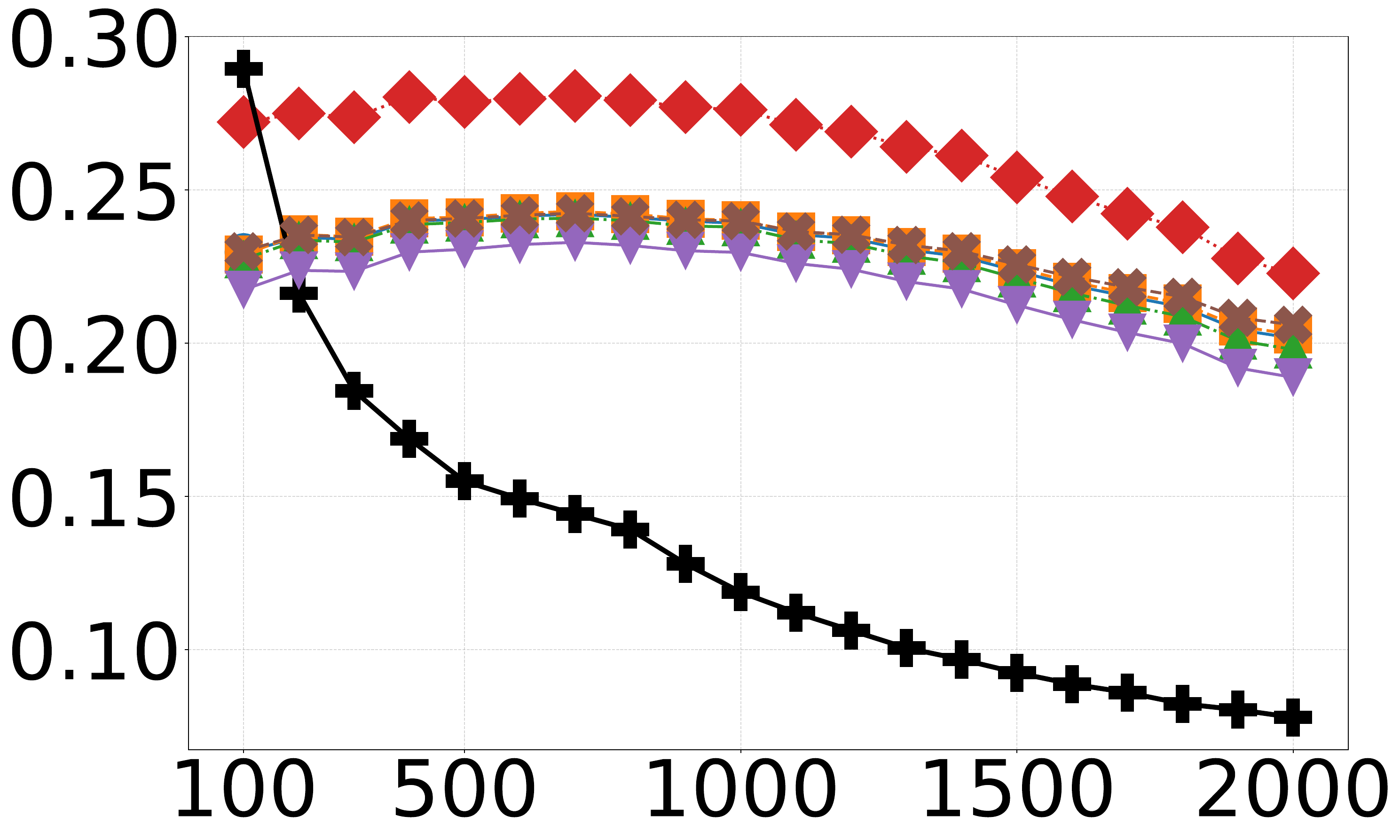}
& \includegraphics[width=0.15\textwidth]{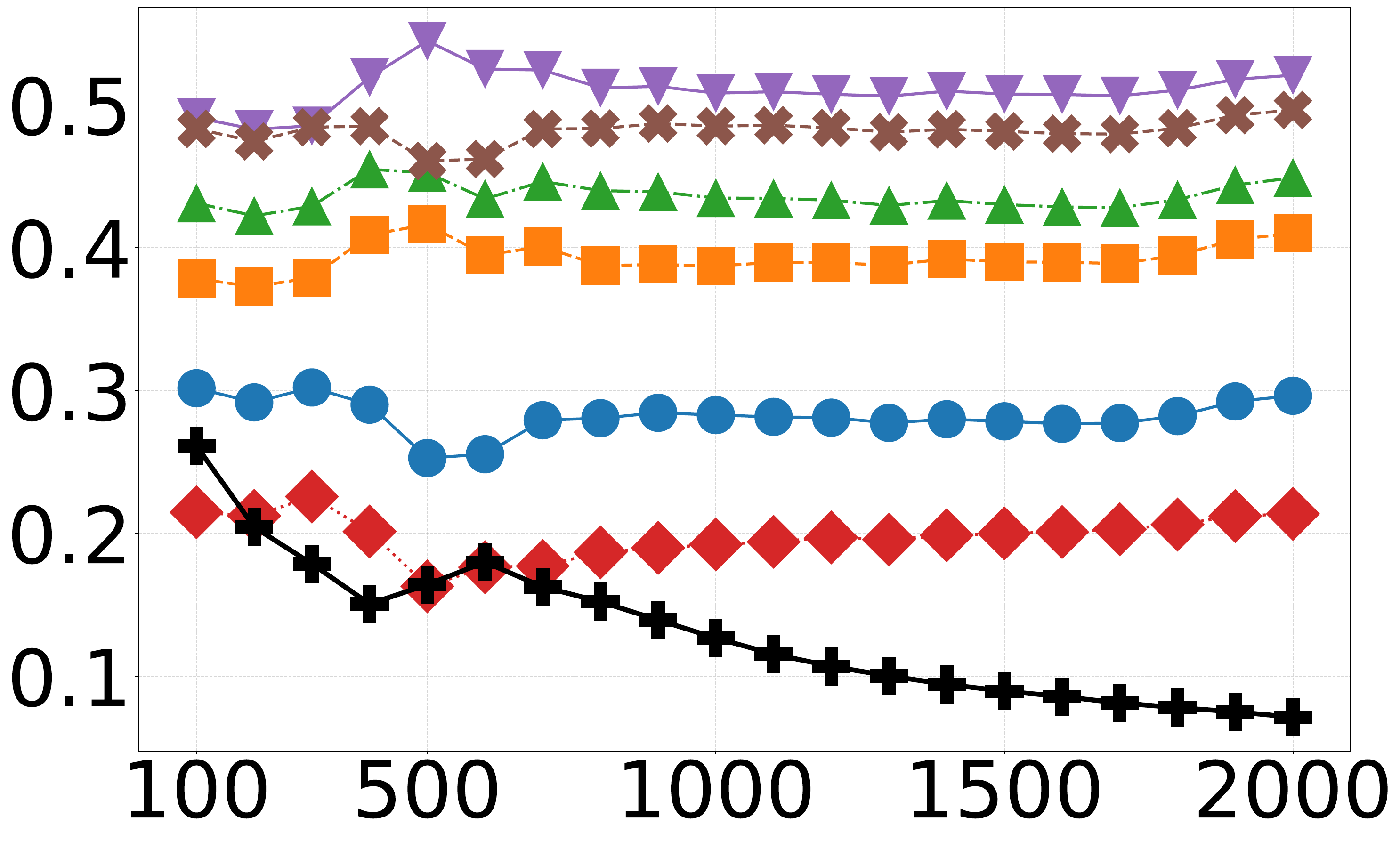}
& \includegraphics[width=0.15\textwidth]{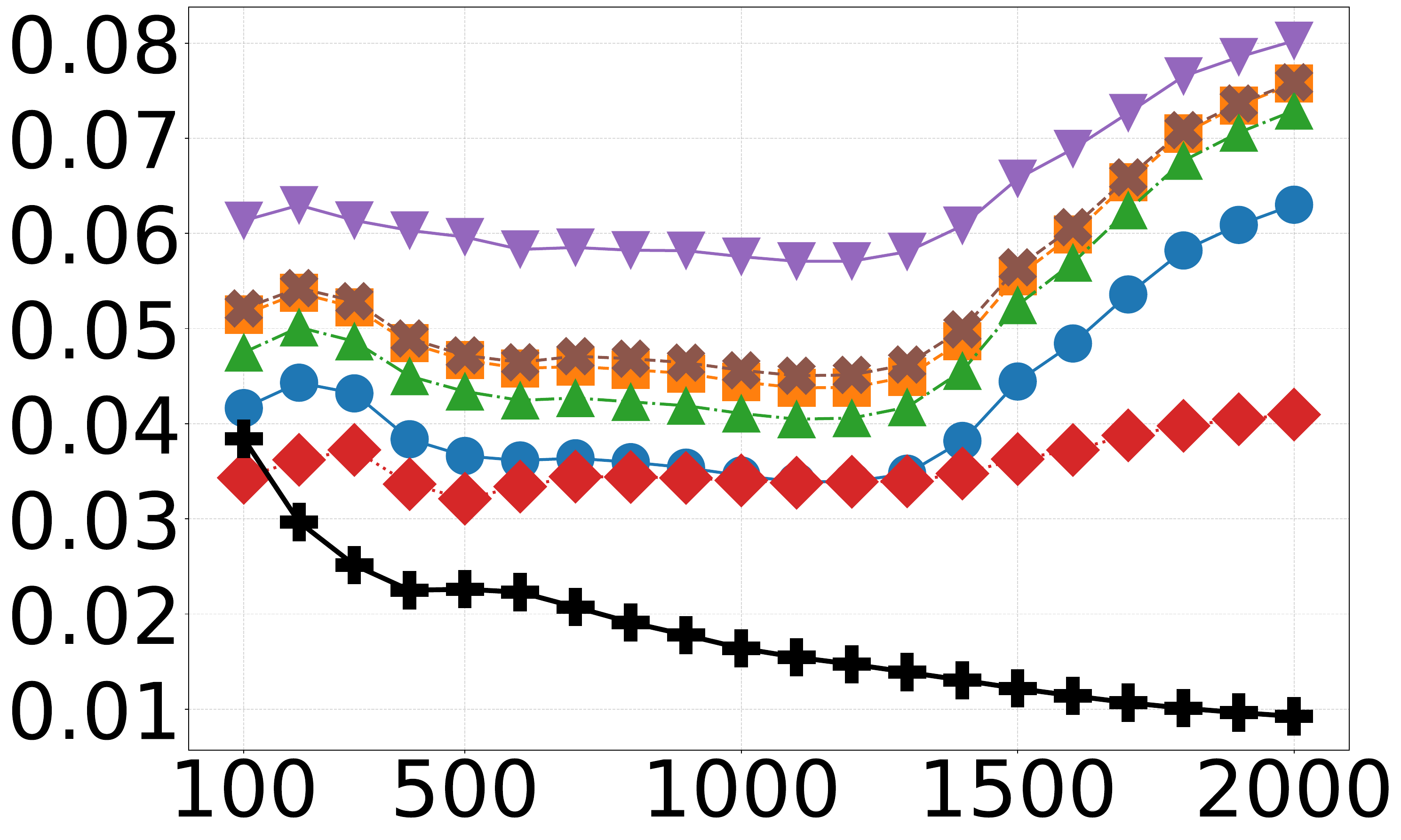}
& \includegraphics[width=0.15\textwidth]{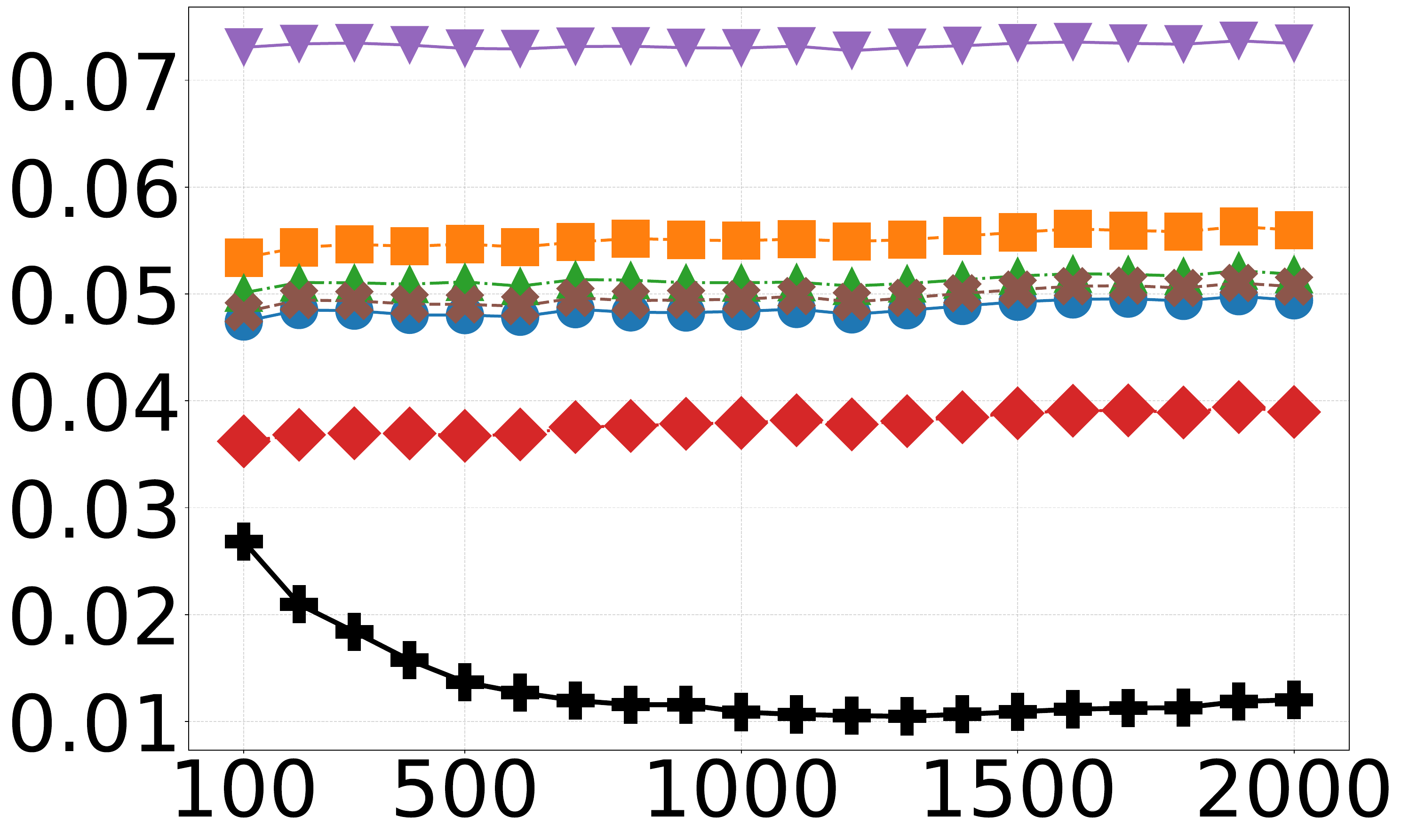}
& \includegraphics[width=0.15\textwidth]{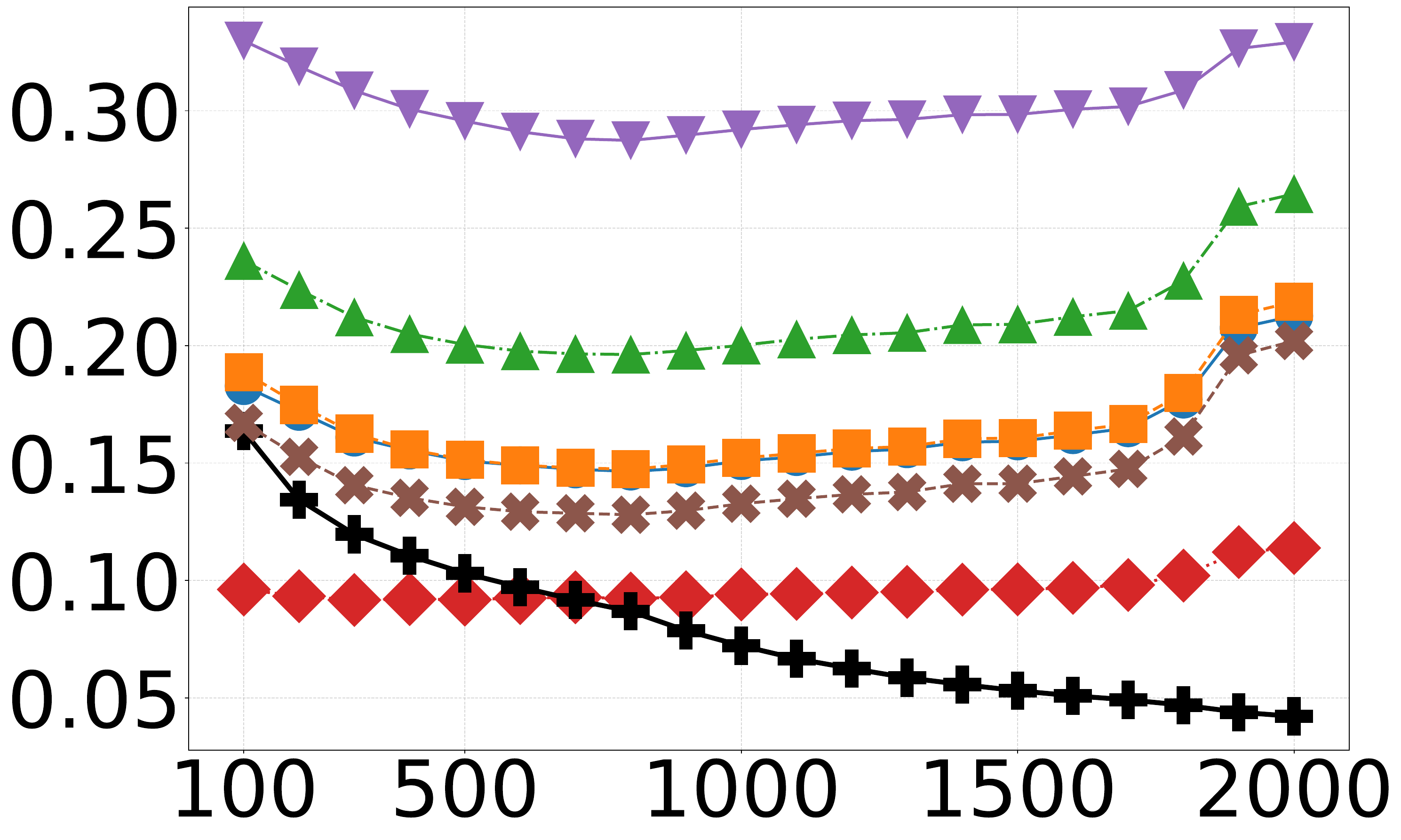}
& \includegraphics[width=0.15\textwidth]{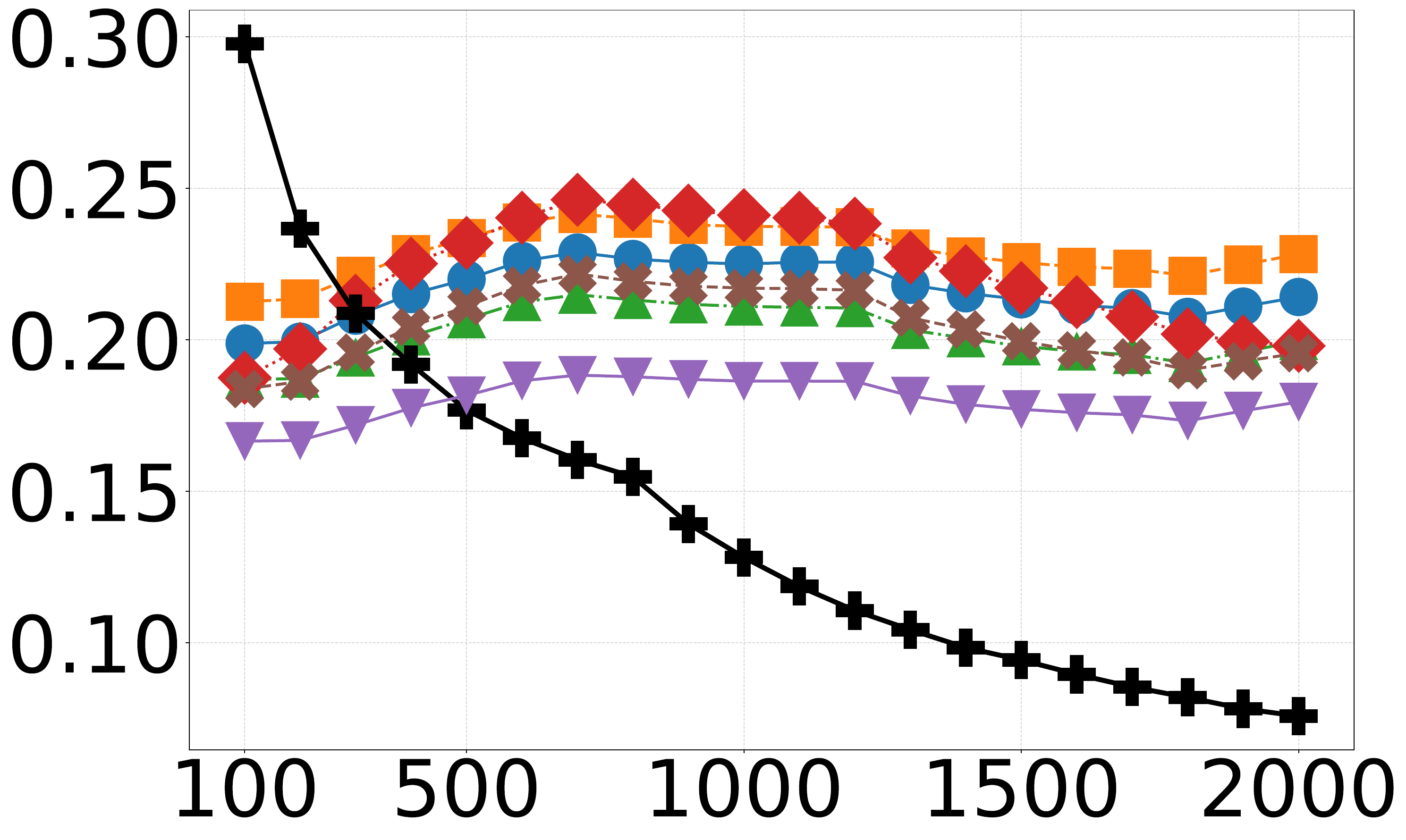}
\\%[1em]

% ---------------- OOD Row ----------------
OOD
& \includegraphics[width=0.15\textwidth]{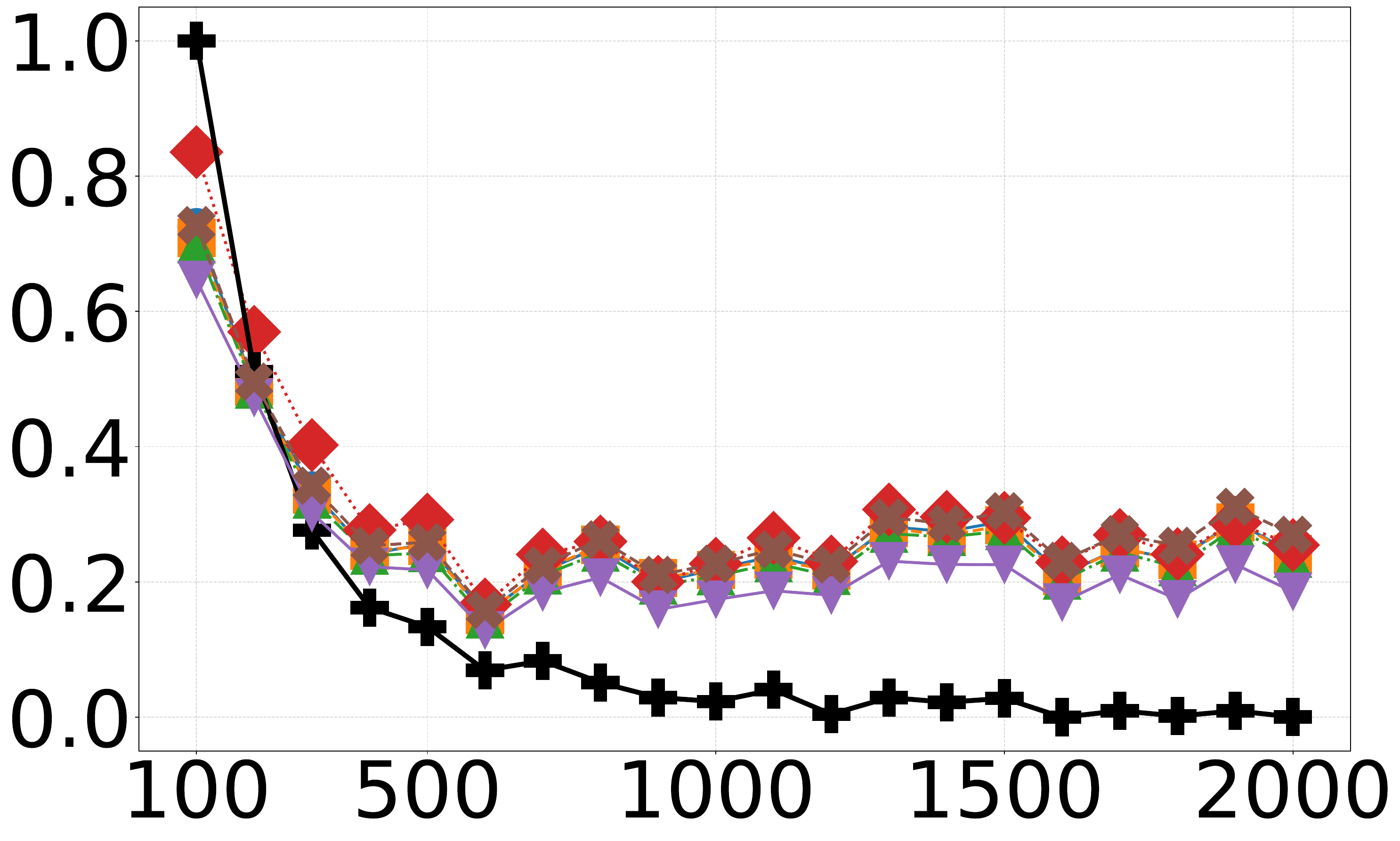}
& \includegraphics[width=0.15\textwidth]{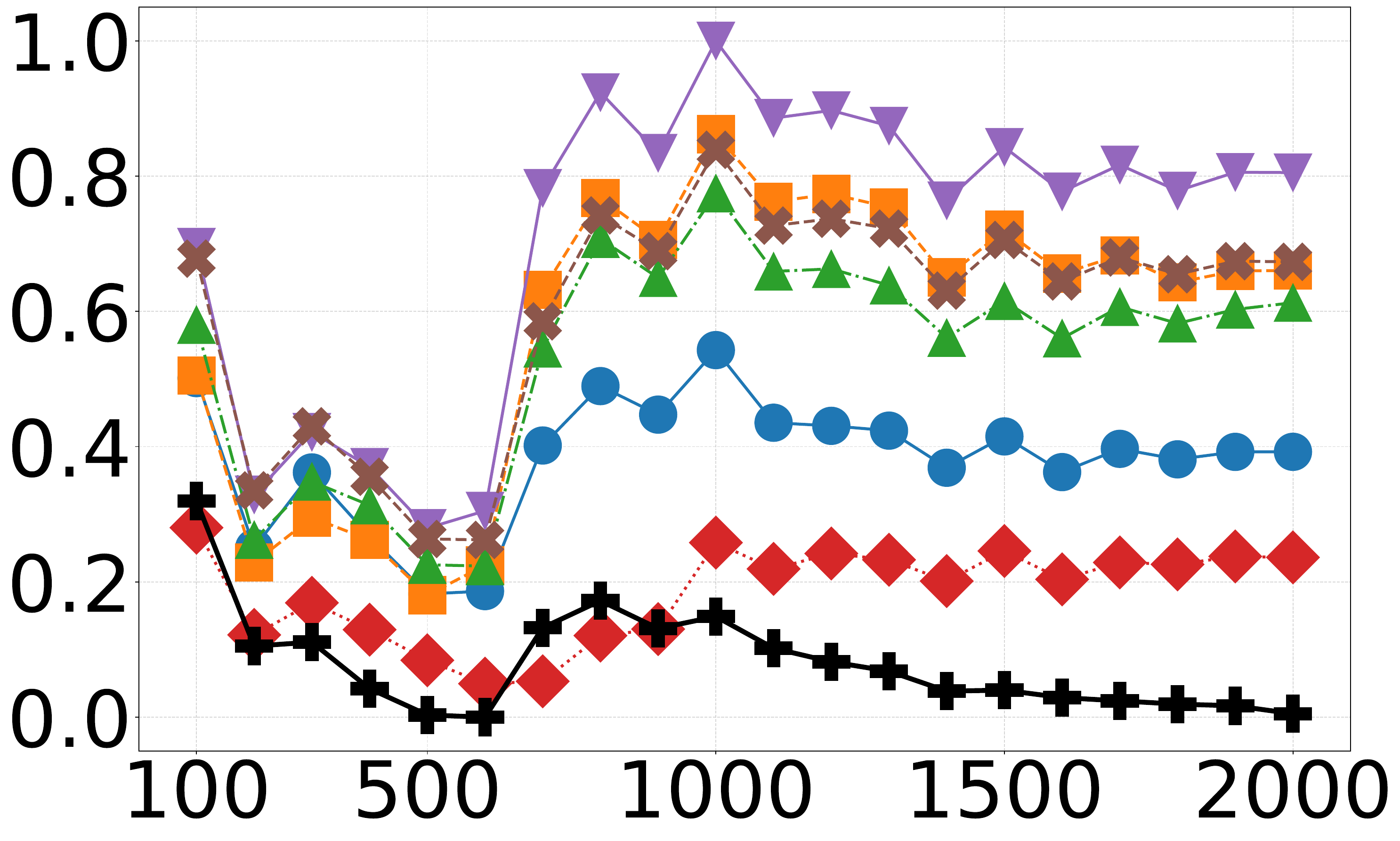}
& \includegraphics[width=0.15\textwidth]{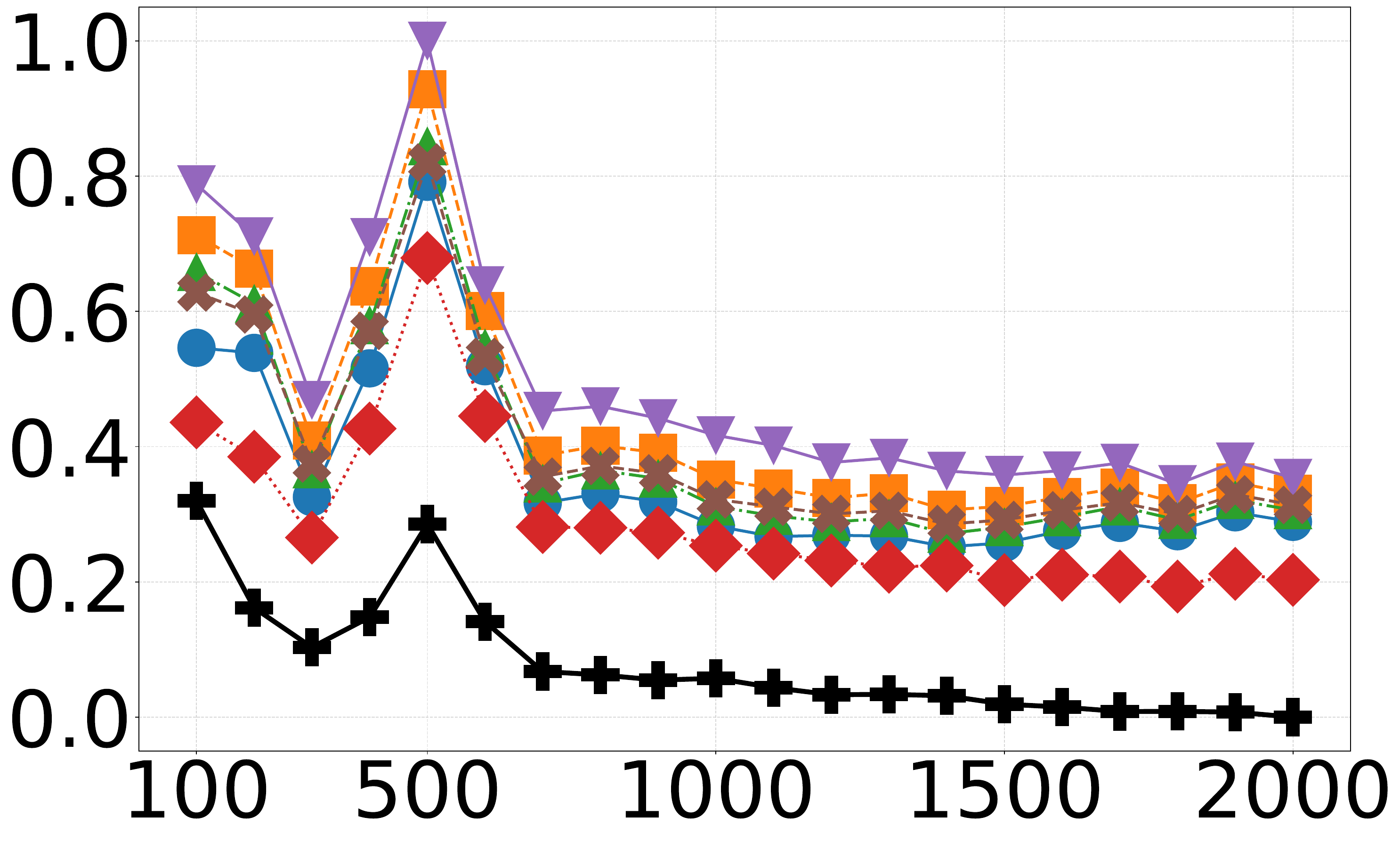}
& \includegraphics[width=0.15\textwidth]{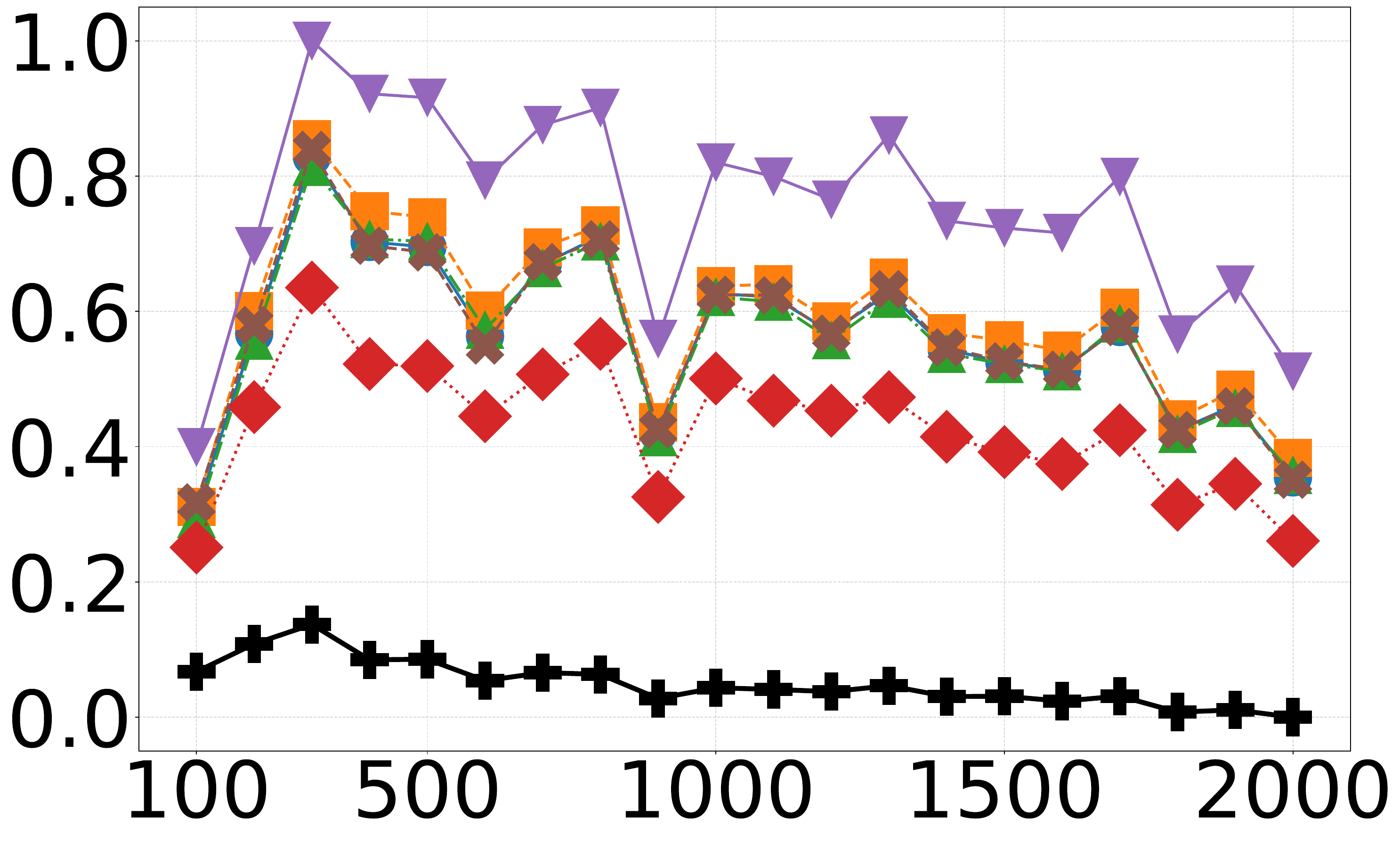}
& \includegraphics[width=0.15\textwidth]{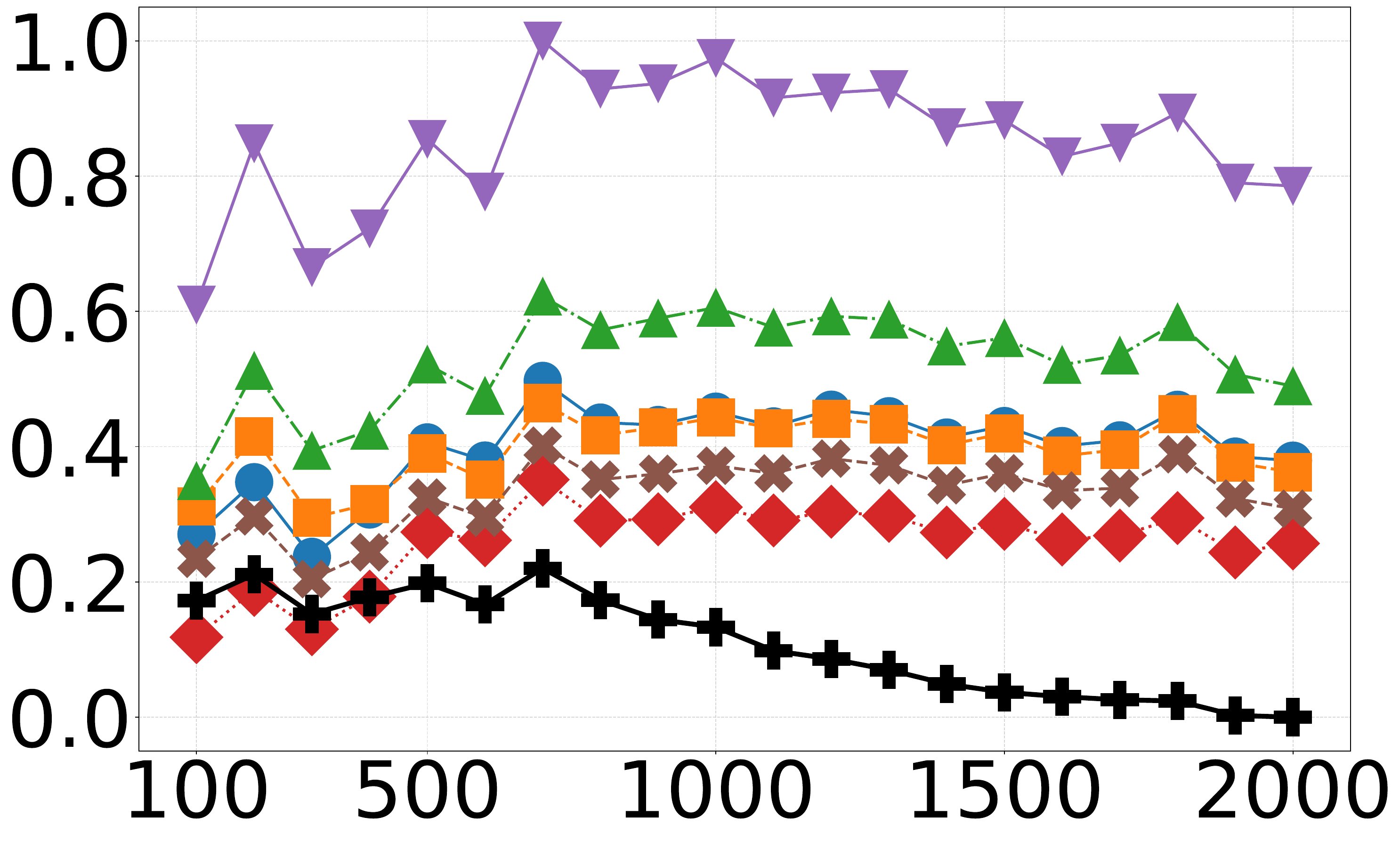}
& \includegraphics[width=0.15\textwidth]{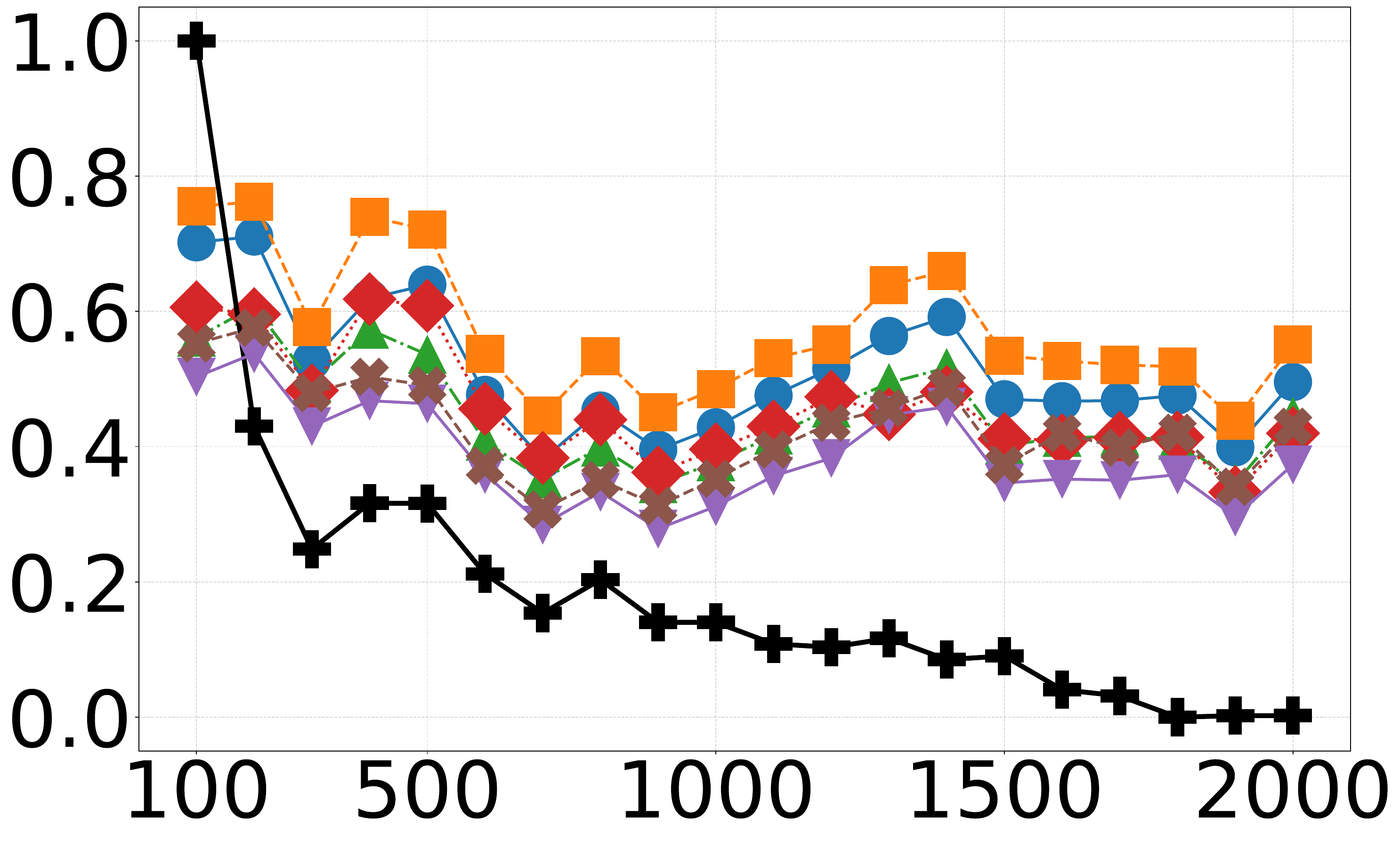}
\\

\end{tabular}

\vspace{0.5em}

\begin{tikzpicture}

% Title placed manually
% \node[anchor=east] at (-0.3,0.0) {\scriptsize \textbf{Perturbation Type}};

\begin{axis}[
    hide axis,
    xmin=0, xmax=1,
    ymin=0, ymax=1,
    legend columns=7,
    legend style={
        draw=black,
        fill=white,
        legend cell align=left,
        at={(0.5,0)},
        anchor=west
    }
]

\addlegendimage{only marks, mark=*, mark size=4pt, color=blue}
\addlegendentry{Avg}

\addlegendimage{only marks, mark=square*, mark size=4pt, color=orange}
\addlegendentry{Black}

\addlegendimage{only marks, mark=triangle*, mark size=5pt, color=green!60!black}
\addlegendentry{Gray}

\addlegendimage{only marks, mark=diamond*, mark size=5pt, color=red}
\addlegendentry{Blur}

\addlegendimage{
only marks,
mark=triangle*,
mark options={rotate=180},
mark size=5pt,
color=purple
}
\addlegendentry{Random}

\addlegendimage{only marks, mark=x, mark size=5pt, color=brown}
\addlegendentry{White}

\addlegendimage{only marks, mark=+, mark size=5pt, color=black}
\addlegendentry{Ours}

\end{axis}
\end{tikzpicture}

\caption{Average $\mathcal{W}_1$ and OOD scores (y-axis) over number of concepts (x-axis) for all pretrained models on VITON-HD. Curves illustrate various perturbation strategies. The y-axis is normalized to range from 0 to 1.
%\textbf{Complex Data Results} We report the effect of the number of concepts on the $\mathcal{W}_1$ and OOD scores averaged over all pretrained models, evaluated against 6 image perturbations. 
}
\label{fig:fid_ood_models}
\end{figure*}

\compactSubSection{Concept Verification.} \label{exp:metric_analysis} In this experiment, we verify that our latent perturbations preserve the symmetry of a similarity function on Multi-CIFAR-10, thereby encapsulating the present concepts. For two activations $\mathbf{a}_j, \mathbf{a}_k \in \mathbf{A}$, the similarity function $f$ is symmetric, i.e., $f(\mathbf{a}_j, \mathbf{a}_k) = f(\mathbf{a}_k, \mathbf{a}_j)$. Consider two collages, each differing by 1 concept. Let concept $c_1$ be the distinct concept present in $\mathbf{a}_j$ and $c_2$ be the distinct concept present in $\mathbf{a}_k$. We perform a semantic replacement to make $\mathbf{a}_j$ more similar to $\mathbf{a}_k$ by calculating:
\begin{equation}
\mathbf{a}_j' = \mathbf{a}_j - u_{j,1}\mathbf{v}_1 + u_{k,2}\mathbf{v}_2,
\quad
\mathbf{a}_k' = \mathbf{a}_k - u_{k,2}\mathbf{v}_2 + u_{j,1}\mathbf{v}_1, \nonumber
\end{equation}
for $\mathbf{a}_k$. If the learned concepts accurately capture the features, the similarity $f(\mathbf{a}_j', \mathbf{a}_k')$ should align closely with $f(\mathbf{a}_k, \mathbf{a}_j)$. We report the percent difference between
$f(\mathbf{a}_j', \mathbf{a}_k')$
and
$f(\mathbf{a}_k, \mathbf{a}_j)$ for two functions. Note: full score recovery is impossible since concepts do not encode spatial information.

The results in Figure~\ref{fig:sanity_check} indicate the trained dictionaries adequately represent the underlying concepts. The median percent difference across all SAE variants remains within $10\%$, that is $90\%$ recovery, of the original score for 2, 4, and 6 image differences for Cosine, and nearly all Euclidean evaluations. We also see that Top-$K$ SAEs consistently yield the lowest deviation across all swap counts. Furthermore, all models recover Cosine more effectively than Euclidean, suggesting the swapping preserves angular relationships better than distance. As expected, the deviation increases monotonically with the number of concept swaps for Cosine and for Vanilla SAEs for Euclidean. Overall, our results show that latent perturbations successfully preserve  \~{$90\%$} of a similarity function's symmetry, especially with low number of swaps and hence represent the underlying concepts.

\compactSubSection{Data Distributional Effect.} \label{exp:ood_fid} To evaluate the reliability of our latent perturbations, we measure their distributional impact by comparing input and latent perturbations with two complementary metrics: a 1-Wasserstein ($\mathcal{W}_1$) \cite{wasserstein} and an OOD score \cite{holistic_concepts,sun2022knnood}. These quantify: (i) the faithfulness of the perturbed representations relative to the original distribution, and (ii) the ability of the perturbation to generate points that remain within the data manifold, respectively. 
% this is only a repetition of the above
%A good perturbation strategy should yield low scores for both metrics, indicating the intervention more closely represents the original distribution and stays in the model's domain.

On our synthetic data, we apply a perturbation to a concept in a quadrant of the collage either in the input space or the latent space. For image space perturbations on VITON-HD, we find the top 10\% activating pixels based on the concept's heatmap and apply the perturbation only to them. 

Our results are summarized in Table \ref{tab:ood_perturbation_table} on Multi-CIFAR-10  and VITON-HD in Figure \ref{fig:fid_ood_models}. Notably, we observe in the table that, overall, for each pretrained model, our latent-space interventions consistently achieve lower OOD and $\mathcal{W}_1$ scores than raw input-space masking despite a favorable setting for the image-space perturbations. 
%The table also reveals interesting model properties, such as a heavy blur being more invasive than random noise for ConvNeXt and ViT, and vice versa, for all other models, calling for a model-specific selection of image-space perturbations. 
The charts display the averaged results over the pretrained models with the number of concepts on the x-axis. The charts demonstrate that latent perturbations starting at 200 concepts on average achieve lower $\mathcal{W}_1$ and OOD. The negative slope behavior is caused by greater concept granularity induced by the number of concepts, meaning that more concepts enable finer-grained features to be captured, while fewer induce more all-encompassing concepts. We do conclude that our latent concept perturbations represent the data distribution and produce fewer OOD embeddings than other common input perturbation methods.

\begin{table}[h]
\centering
\footnotesize
\setlength{\tabcolsep}{4pt}
%\caption{Average similarity drop ($\Delta$) and positive ratio across hops.}
\begin{tabular}{ll cccc}
\toprule
& & \multicolumn{4}{c}{\textbf{Number of image differences}} \\
\cmidrule(l){3-6}
\textbf{Metric} & & \textbf{0} & \textbf{2} & \textbf{4} & \textbf{6} \\
\midrule
\multirow{2}{*}{Cosine} & $\Delta$ & 0.067 & 0.097 & 0.140 & 0.209 \\
 & Ratio & 0.95$\pm${\scriptsize0.08} & 0.93$\pm${\scriptsize0.10} & 0.92$\pm${\scriptsize0.11} & 0.91$\pm${\scriptsize 0.13} \\
\midrule
\multirow{2}{*}{-Euclid.} & $\Delta$ & 0.543 & 0.177 & 0.106 & 0.076 \\
 & Ratio & 0.95$\pm${\scriptsize0.08} & 0.92$\pm${\scriptsize0.09} & 0.89$\pm${\scriptsize0.11} & 0.86$\pm${\scriptsize0.13} \\
\bottomrule
\end{tabular}
\caption{We report the average relative similarity drop $\Delta = \left( f(\mathbf{a}_i, \mathbf{a}_j) - f(\mathbf{a}_i/c_k, \mathbf{a}_j) \right) / f(\mathbf{a}_i, \mathbf{a}_j) $ and the number of times similarity decreases over the number of pairs for varying similarity levels (Ratio). Latent perturbations consistently induce a decrease in overall similarity.}
\label{tab:perturbation_effect}
\end{table}

% \begin{table}[h]
% \centering
% \footnotesize
% \setlength{\tabcolsep}{4pt}
% \begin{tabular}{ll cccc}
% \toprule
% & & \multicolumn{4}{c}{\textbf{Number of image differences}} \\
% \cmidrule(l){3-6}
% \textbf{Metric} & & \textbf{0} & \textbf{2} & \textbf{4} & \textbf{6} \\
% \midrule
% \multirow{2}{*}{Cosine} & $\Delta$ & 0.114 & 0.144 & 0.173 & 0.201 \\
%  & Ratio & 0.99$\pm$0.03 & 1.00$\pm$0.03 & 1.00$\pm$0.03 & 1.00$\pm$0.03 \\
% \midrule
% \multirow{2}{*}{Euclid.} & $\Delta$ & 8.692 & 5.286 & 4.088 & 3.453 \\
%  & Ratio & 0.99$\pm$0.03 & 0.99$\pm$0.04 & 0.98$\pm$0.05 & 0.98$\pm$0.05 \\
% \bottomrule
% \end{tabular}
% \caption{We report the average relative similarity drop ($\Delta$) and the number of times similarity decreases over the number of pairs for varying similarity levels. Latent perturbations consistently induce a decrease in overall similarity.}
% \label{tab:perturbation_effect}
% \end{table}

\compactSubSection{Sim.~Function Effect.} To assess whether perturbations carry meaningful similarity information, we use Multi-Cifar-10 to construct 3000 image pairs spanning varying similarity levels, then perturb only the shared concepts in one image of each pair. Table~\ref{tab:perturbation_effect} shows a consistent relative drop in similarity across no.\ of image differences, with similarity decreasing in nearly all pairs, confirming that our perturbations meaningfully impact the score and thus validating our choice to use score changes to estimate concept importance.

\compactSubSection{Linearly Recoverable.}
To quantify how well our explanation vectors encode the underlying similarity, we use an evaluation setting inspired by \cite{NEURIPS2024_56ed2bd1}: we
train a regression model given an explanation to predict the similarity score $f(\mathbf{a}_j, \mathbf{a}_k)$.
In particular, we sample 3,000 random pairs of embeddings ($80/20$ train/test split) and compute their similarity scores. For each pair, we generate an explanation vector $\localRel \in\mathbb{R}^{1500}$ and fit a linear regression to predict the corresponding score from this representation. Here, we use 1500 concepts based on their minimal distributional effect and only use TopK SAEs due to the Verification results. Table~\ref{tab:perf_baseline_1000_random} reports the average performance across all test pairs. We compare to other explanation vector baselines in the same setting,
%For comparison, we repeat the same regression protocol using explanation vectors produced by several baseline methods, 
including CSIM~\cite{csim} created by multiplying the concept activations, blur-based explanation vectors \cite{chen2023sim2word} analogous to ours except done by blurring, and vanilla gradient-based image attribution method, computed by adding the gradients of each image.
%taking the gradient with respect to each image and adding them to obtain one vector.

Our results show that our explanation vectors consistently recover the underlying similarity metrics more accurately than all competing methods. In particular, our approach achieves mean $R^2$ values above $0.87$ for both Cosine and Euclidean. These results demonstrate that our explanations preserve the information required to reconstruct the similarity between image pairs and faithfully encode the behavior of the underlying similarity function.

% To assess how well our explanations explains the underlying metric, we sample random pairs of image activations and record their similarity score for three different metrics. Then, we produce our explanation vectors $\localRel$ of each pair with 1000 concepts. Then we train a linear regression model on the explanation vectors with the similarity score as the target. We report the average results for 3000 pairs in Table \ref{tab:recovery_table}. Additionally, we train 4 more linear regressions  and compare against other baseline explantions methods including CAV arithmetic operations \cite{concept sim}, explanations vectors produced by blurring, and then a vanilla gradient-based image attribution. 

% Our results confirm that our produced explanations vectors recover the similarty metrics results in the activation space and out perform all of the baselines. Notably, our method obtains a perfect $R^2$ value for the Dot Product metric with means over .95 for the Euclidean and Cosine metrics. Hence, we can conclude that our explanation method correctly explains the similarity of image pairs. 

\begin{figure*}[t]
    \centering
    \includegraphics[trim={4.5cm 0 4.5cm 0}, width=.90\textwidth]{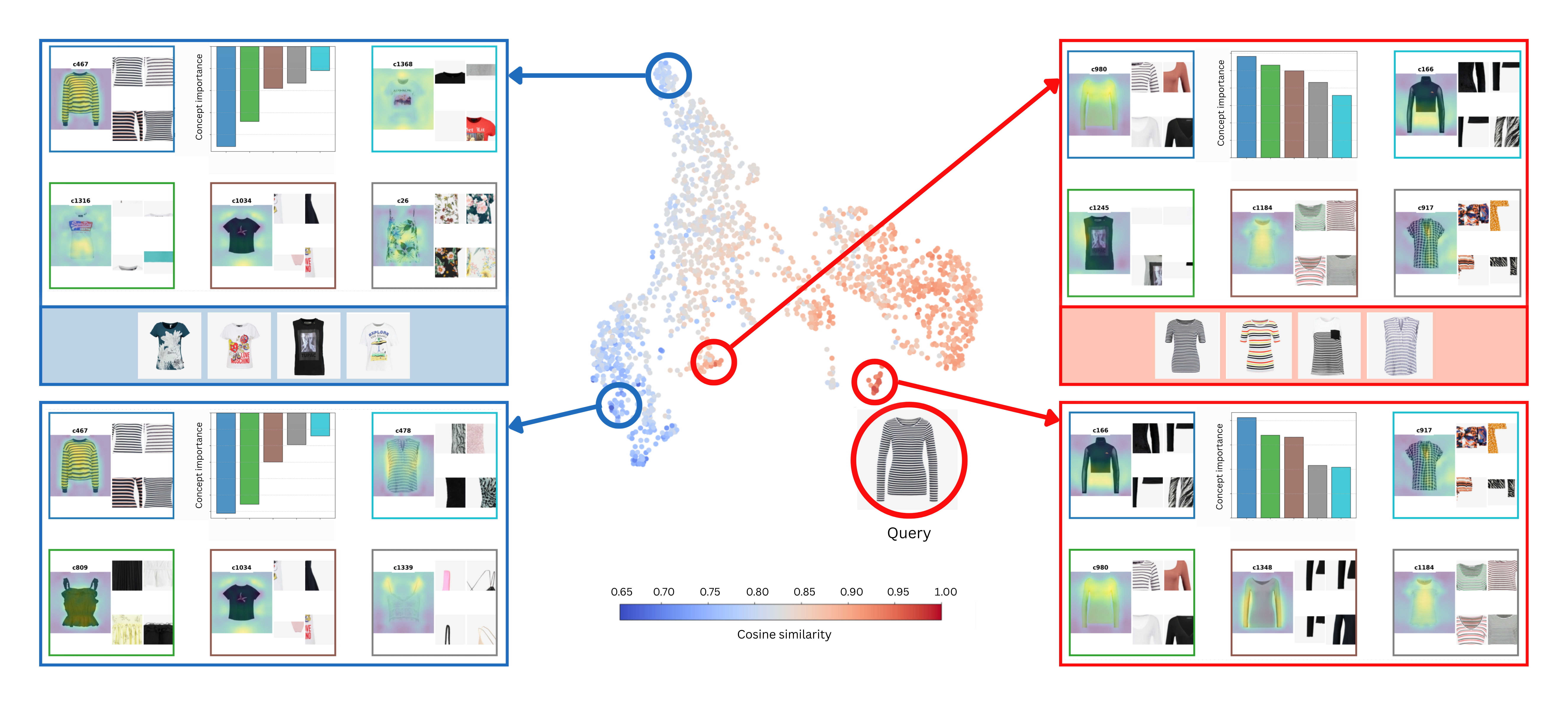} % Replace with your image file
    \caption{UMAP of the explanation vectors produced with respect to Query image, with red indicating high-similarity samples and blue low-similarity samples. 
    The boxes indicate an explanation for groups of images according to eq.~\eqref{eq:group_eq}. For similarity, we illustrate "similarity concepts" contributing to similarity because they are present in both the query and group images. For dissimilarity, we illustrate "dissimilarity concepts" which are not shared between the query and the group images.
    }
    \label{fig:group_q_example}
\end{figure*}

\begin{figure}[t]
    \centering
    \includegraphics[width=.45\textwidth]{assets/exemplar_retrieval.pdf}    
    \caption{Exemplar Retrieval Example (selecting images similar for the same reason as a given image pair). First column: query--reference pairs; second column: selection via close cosine similarity as a baseline; third column: selecting via our explanation.}
    \label{fig:exemplar_retrieval}
\end{figure}

\section{Case Study of Explanations}

In this section, we present our generated explanations, explore a group explanation, and demonstrate their usefulness. We begin by comparing our approach to local explanations found in the literature. Next, we present a novel application of our local explanations in the context of retrieval. Finally, we generate group explanations first by analyzing a retrieval setting and secondly by visualizing embeddings associated with a specific concept and explaining the formed groupings. 

\begin{table}[ht]
\footnotesize
  \centering
  \setlength{\tabcolsep}{2.5pt}
  \begin{tabular}{lcc}
    \toprule
    \textbf{Method} & \textbf{Cosine}  & \textbf{-Euclidean}\\
    \midrule\textbf{}
    CSIM
    & 0.88 $\pm$ {\scriptsize 0.06} / 0.04 $\pm$ {\scriptsize 0.01} & 0.85 $\pm$ {\scriptsize 0.07} / 1.65 $\pm$ {\scriptsize 0.42} \\
    % CAV Subtraction & -0.29 $\pm$ {\scriptsize 0.07} / 0.02 $\pm$ {\scriptsize 0.00} & -0.29 $\pm$ {\scriptsize 0.06} / 35005.82 $\pm$ {\scriptsize 1543.78} & -0.27 $\pm$ {\scriptsize 0.07} / 36.08 $\pm$ {\scriptsize 2.06} \\
    Blur
    & 0.73 $\pm$ {\scriptsize 0.04} / 0.05 $\pm$ {\scriptsize 0.00} & 0.82 $\pm$ {\scriptsize 0.02} / 1.66 $\pm$ {\scriptsize 0.09} \\

    Gradient 
    & 0.74 $\pm$ {\scriptsize 0.03} / 0.05 $\pm$ {\scriptsize 0.00} & 0.63 $\pm$ {\scriptsize 0.04} / 2.51 $\pm$ {\scriptsize 0.15} \\
    
    Ours
    & \textbf{0.94*} $\pm$ {\scriptsize 0.03} / \textbf{0.02*} $\pm$ {\scriptsize 0.00} & \textbf{0.87*} $\pm$ {\scriptsize 0.07} / \textbf{1.59} $\pm$ {\scriptsize 0.38} \\
  
    \bottomrule
  \end{tabular}
  \caption{ We report the averaged performance ($R^2 \uparrow$ / RMSE $\downarrow$) over all vision models over 10 runs. We employ 1500 concepts for our method and compare against 4 other explanation techniques. Statistically significant results according to a one-sided, non-paired Wilcoxon rank-sum test with Bonferroni correction for multiple tests are indicated by \textbf{*}.}
  \label{tab:perf_baseline_1000_random}
\end{table}

\compactSubSection{Setup.}
We use DINOv3~\cite{2025dinov3} as the frozen backbone and train a TopK SAE~\cite{gao2025scaling} on activations extracted from the test set of VITON-HD~\cite{Choi_2021_CVPR}. For the explanations, we utilize a dictionary of $1500$ concepts. Notably, we do not use any label information from this dataset; therefore, all insights regarding the formed groups are derived entirely from our explanations.

\compactSubSection{Local Explanation Comparison.}\label{sec:local_exp} 
Here, we focus on explaining \textit{why} a given query and reference image are similar. As shown in Figure \ref{fig:sim_example} (top), our query is a striped shirt, and the reference is a striped blouse. On the left, we display an explanation inspired by previous work~\cite{plummer2020these,chen2023sim2word}, consisting of a heatmap and an attribute label. This baseline may capture the primary similarity driver between the two garments (assuming initial human definition of the concept) but cannot account for all factors of similarity. 

Our approach calculates importance scores across all concepts present in either image (Figure \ref{fig:sim_example} displays the top 5). The most prominent is Concept c1184, highlighted in the red box, where a heatmap localizes the concept on the query, 
%(Note: in the other explanations, the top-activating image will be displayed), 
while the right side displays the four patches that most activate this concept. Because these patches depict stripes and the heatmap aligns with this pattern, we designate Concept c1184 as the "stripe concept" for this dataset. Other noteworthy concepts include c980 (the neckline), c917 (shirt patterns), and 713 (curves). Concept c1245 appears to represent a contrast between the background and a given garment. Since DinoV3 is not fine-tuned in the fashion domain, we expect to find such concepts, which may affect similarity. By combining similarity importance scores with precise heatmaps, our local approach provides a comprehensive explanation: it pinpoints \textit{where} the images are similar, identifies \textit{what} causes the similarity, and quantifies the importance of each factor.

Notably, our explanations enable designer action based on the provided insights, such as selecting data for fine-tuning to reduce the influence of concept c1245, or removing its contribution to the similarity entirely.

\compactSubSection{Exemplar Retrieval.} 
Inspired by conditional similarity learning~\cite{lim2026clay, hsieh2025focallensinstructiontuningenables}, we introduce a new task that we call \emph{Exemplar Retrieval}. The goal is to retrieve images that are similar for the same underlying reasons as a given query--reference pair. Unlike conditional similarity, where the conditioning attribute is explicitly provided (e.g., color or shape), our setting assumes no such supervision. Instead, we are given a pair of images similar in some way and seek additional samples that share the concepts responsible for that similarity.
A natural baseline is to retrieve image pairs with close similarity scores. However, a score does not provide information about \emph{why} two images are similar. To address this, we retrieve exemplars using our explanation vectors. Given a pair, we compute its explanation vector and rank candidate pairs according to the dot product between their explanation vectors and that of the original pair. 
% Since the explanation vectors encode the concepts responsible for similarity, this retrieval procedure returns samples that share the same semantic factors rather than merely exhibiting comparable similarity scores.

Figure~\ref{fig:exemplar_retrieval} demonstrates the advantage of this approach for four example pairs. Across all these, retrieval based on similar cosine similarity (baseline) often matches images with similar overall appearance while failing to preserve the semantic reason for similarity. In contrast, retrieval using our explanation vectors consistently recovers images with the same concepts. For example, our method retrieves striped shirts rather than plain (Row~1), dark-colored shirts instead of lighter alternatives (Row~2), shirts with similar necklines (Row~3), and shirts with comparable sleeve styles (Row~4). These results suggest that explanation vectors capture the semantic factors driving similarity, enabling exemplar retrieval conditioned by concepts.
Additionally, we highlight this as a new actionable task derived directly from our explanations.
% This is also an example of a new actionable task that can be approached using our proposed explanations.

\begin{figure*}[t]
    \centering
    \includegraphics[width=0.9\textwidth]{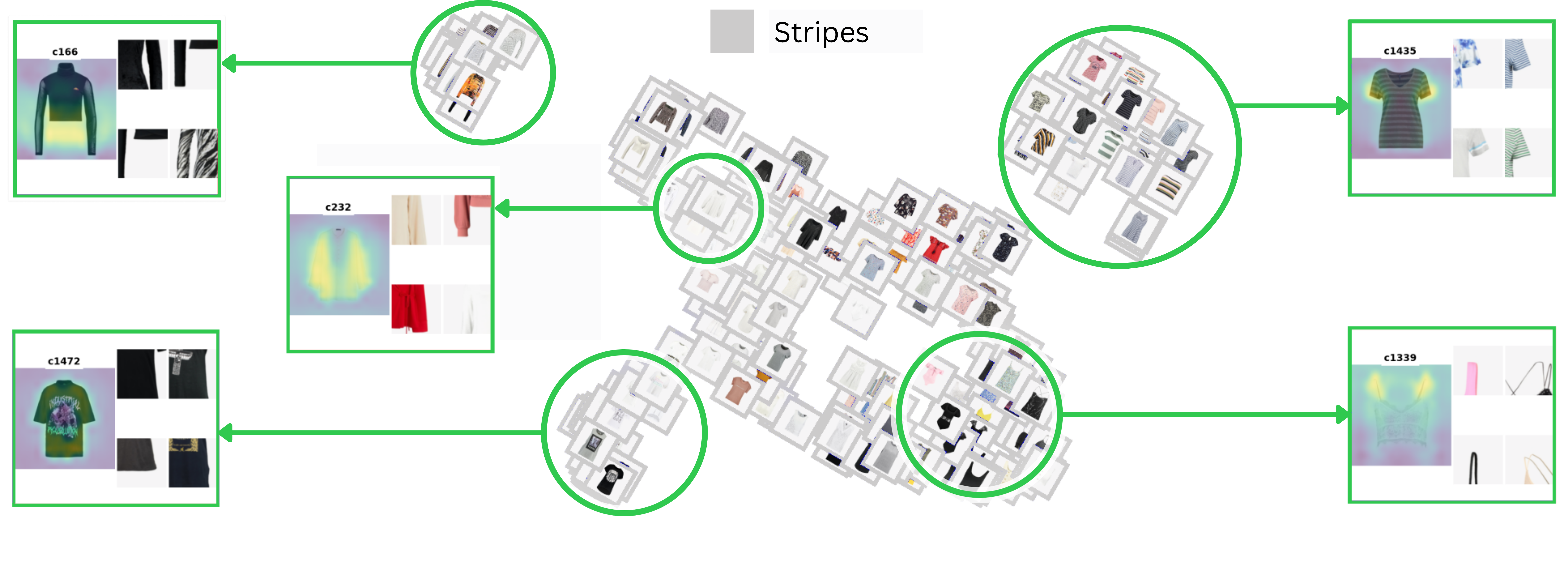}    
    \caption{UMAP of DINOv3 embeddings with respect to cosine distance for 250 images that activate the stripes concept the most.
    For particular clusters we compute pairwise 
    similarity explanations and identify the most dominant concept with respect to cosine similarity and display it in boxes.}
    \label{fig:group_example}
\end{figure*}

\compactSubSection{Investigating Groups with Respect to a Query.}
We consider a retrieval setting where we rank our dataset for a given query. We generate explanation vectors for each reference image with respect to the query, and then apply UMAP~\cite{umap} to them to identify emergent groups. In Figure \ref{fig:group_q_example}, the color in the UMAP projection represents a spectrum: red indicates more similar references, blue more dissimilar references, and white/gray neutral cases. 

The visualization reveals that similar explanations naturally congregate. We highlight a few particularly interesting clusters. For instance, Group 1 contains similar samples.% yet it resides surprisingly close to dissimilar samples in the projection. 
A closer look at our generated explanations reveals that the neckline, contrast, and stripe concepts drive the similarity for this cluster relative to the query. To verify this, we draw a random sample from this group (bottom left of Figure \ref{fig:group_q_example}), which indeed consists of various striped t-shirts, confirming the validity of our explanation. 

We apply the same analysis to Group 2, but focus on \textit{dissimilarity}. We argue that for dissimilar samples, understanding \textit{why} they differ is more informative than forcing a similarity comparison. The concepts contributing most to this group's dissimilarity are Concept 467 (striped long-sleeve shirt) and Concept 1316 (t-shirt with a central print). Given that our query is a striped long-sleeve shirt, and the random sample from Group 2 (top right) contains only printed t-shirts, this explanation is logical. This demonstrates our method's capacity to elucidate group relations relative to a specific query.
Accordingly, our explanations enable the retrieval user to retrieve groups with similar characteristics.

% Inspired from conditional similarity \cite{lim2026clay, hsieh2025focallensinstructiontuningenables}, we create a task which we call Exemplar-Conditioned Retrieval, which aims to find samples that reflect the same reasons for similarity. In conditional similarity, a known property is conditioned such that similar samples with respect to this property is returned. Our task is different since we do not have explicit properties to condition on; all we know is that a pair of images are similar for some reason and we seek to find other samples which reflect the reason for similarity. Normally, one might search for samples with a similar score in order to find more samples that have the same reasons for similarity. 

% As seen in Figure \ref{fig:exemplar_retrieval}, the samples retreived with cosine do not always reflect the reasons for similarity. Using our explanation vectors, we compute the dot product the vector produced by the given query and reference and the other explanation vectors and select the maximum. In the first row of Figure \ref{fig:exemplar_retrieval}, we see that our method returns shirts with stripes while cosine returns solid shirts. Row 2 shows that our exp;lanation vectors return only dark shirts like the query and reference. Row 3 returns shirts with a similar necklines and 4 returns samples with sleeves.
 
\compactSubSection{Similar Groups within a Concept.}\label{sec:group_exploration} 
Finally, we highlight the ability to interpret similarity within the DINOv3 embedding space. To do this, we select 250 samples that most strongly activate the "stripe" concept and project them using UMAP (Figure \ref{fig:group_example}). Several distinct groups immediately emerge. To validate, we isolate each group, compute all pairwise explanations for its members, and display the dominant concept.% found. 

Interestingly, this process uncovers distinct sub-categories, such as printed t-shirts (Concept 1472) and garments with straps (Concept 1339). We also observe a clear structural divide: short-sleeve striped shirts congregate in one area, while long-sleeve variants cluster on the opposite side of the UMAP space. This illustrates how our explanation framework can disentangle and map out the nuanced variations within our data, despite the absence of human-annotated labels.
Also, it enables the model designer to take actions to, e.g., remove unfavorable clustering.

\section{Limitations and Future Work}\label{sec:lim}
Although our concept-based explanations improve upon the baselines and provide useful insights, our method is not without limitations. First, 
our method implicitly assumes an embedding space can be reconstructed linearly using CAVs \cite{Linear_hypothesis}, which has recently been called into question \cite{hindupur2025projectingassumptionsdualitysparse,bhalla2026sparseautoencoderscaptureconcept}. Hence, concepts may miss more abstract or relational similarities. We view the linearity of concepts as a first-order approximation \cite{fel2026into} rather than a claim of completeness. SAE stability is of additional concern; we address it by using penultimate-layer activations \cite{holistic_concepts}. Also, we include our hyperparameter results in our Appendix, which report average recovery accuracy; most achieve stable recovery. Additionally, CAV extraction is inherently data-dependent but can generalize to explain similarities in alternative datasets in clear transfer learning scenarios \cite{csim}. Furthermore, for foundation models lacking targeted domain adaptation, ubiquitous "background" concepts (e.g., Concept 1245 in Figure \ref{fig:sim_example}) do appear prominently in the importance rankings, reflecting potentially unfavorable downstream task behavior.
%This may reflect the underlying biases of the pre-trained model, 
Mitigation of these concepts could enhance human interpretability, e.g., with an additional normalization based on concept uniformity.% could effectively filter them out.

% Additionally, CAV extraction is inherently data-dependent; the resulting concepts may only generalize to explain similarities in alternative datasets in clear transfer learning scenarios. 
%Additionally, we inherit the projection assumptions of the SAEs included in the above experiments; specifically, the Vanilla and JumpRelu make linear separability assumptions of the concept, while the TopK relies on angular separability\cite{hindupur2025projectingassumptionsdualitysparse}, meaning the extracted concepts may miss more abstract or relational similarities. 
Regarding human interpretability, while a result in the Appendix, offers an intuition for the invasiveness of our conceptual interventions, the pixel representation of these embedding shifts remains unknown. A promising extension would be to integrate diffusion models \cite{yang2023diffusurvey} to visualize the effects of our perturbations directly in the image space. 
Finally, while the usefulness of concepts for users has been evaluated in user studies (e.g., \cite{fel2023craft}) and how well objects can be recovered \cite{fel2025archetypal}, additional user studies can further assess human alignment with such explanations, in particular for similarities. %We similarly defer to future work.
%Finally, since human interpretable concepts are not a guarantee, a crucial next step for validating our explainability framework is a comprehensive user study to assess human alignment of explanations, which we similarly defer to future work.

\section{Conclusion}\label{sec:conclusion}
In this work, we propose a novel explainability framework for image similarity based on automatically extracted Concept Activation Vectors (CAVs). By decomposing a network's embeddings into semantic concepts, we can perturb these interpretable components directly within the embedding space to systematically study their effect on a given model and an associated similarity metric. To isolate these effects without confounding other visual features, we introduce the Multi-CIFAR-10 dataset. First, we show that the SAE dictionaries represent concepts by preserving similarity symmetry. Then, we demonstrate that latent concept perturbations remain closer to the original embedding distribution and yield fewer OOD samples compared to standard input-space perturbations. Further, we confirm the predictive power associated with our explanations enables linear recovery of the original similarity scores. Finally, we present a case study utilizing our framework to generate diverse explanations, ranging from pairwise concept importances to similarity localization via heatmaps, introduce the novel Exemplar Retrieval task, and produce valid, actionable group explanations.

\bibliography{aaai2027}

@article{burkart2021XAIsurvey,
  title={A survey on the explainability of supervised machine learning},
  author={Burkart, Nadia and others},
  journal={Journal of Artificial Intelligence Research},
  volume={70},
  pages={245--317},
  year={2021}
}

@article{poeta2023concept,
  title={Concept-based explainable artificial intelligence: A survey},
  author={Poeta, Eleonora and Ciravegna, Gabriele and Pastor, Eliana and Cerquitelli, Tania and Baralis, Elena},
  journal={ACM Computing Surveys},
  year={2023},
  publisher={ACM New York, NY}
}

@InProceedings{Choi_2021_CVPR,
    author    = {Choi, Seunghwan and Park, Sunghyun and Lee, Minsoo and Choo, Jaegul},
    title     = {VITON-HD: High-Resolution Virtual Try-On via Misalignment-Aware Normalization},
    booktitle = {CVPR},
    year      = {2021},
}

@misc{hsieh2025focallensinstructiontuningenables,
      title={FocalLens: Instruction Tuning Enables Zero-Shot Conditional Image Representations}, 
      author={Cheng-Yu Hsieh and Pavan Kumar Anasosalu Vasu and Fartash Faghri and Raviteja Vemulapalli and Chun-Liang Li and Ranjay Krishna and Oncel Tuzel and Hadi Pouransari},
      year={2025},
      eprint={2504.08368},
      archivePrefix={arXiv},
      primaryClass={cs.CV},
      url={https://arxiv.org/abs/2504.08368}, 
}

@inproceedings{lim2026clay,
  title     = {CLAY: Conditional Visual Similarity Modulation in Vision-Language Embedding Space},
  author    = {Lim, Sohwi and Hyoseok, Lee and Park, Jungjoon and Oh, Tae-Hyun},  
  booktitle = {Proceedings of the IEEE/CVF Conference on Computer Vision and Pattern Recognition (CVPR)},
  year      = {2026}
  }

@article{yang2023diffusurvey,
  title={Diffusion models: A comprehensive survey of methods and applications},
  author={Yang, Ling and Zhang, Zhilong and Song, Yang and Hong, Shenda and Xu, Runsheng and Zhao, Yue and Zhang, Wentao and Cui, Bin and Yang, Ming-Hsuan},
  journal={ACM Computing Surveys},
  volume={56},
  number={4},
  pages={1--39},
  year={2023},
  publisher={ACM New York, NY, USA}
}

@inproceedings{
gao2025scaling,
title={Scaling and evaluating sparse autoencoders},
author={Leo Gao and Tom Dupre la Tour and Henk Tillman and Gabriel Goh and Rajan Troll and Alec Radford and Ilya Sutskever and Jan Leike and Jeffrey Wu},
booktitle={The Thirteenth International Conference on Learning Representations},
year={2025},
url={https://openreview.net/forum?id=tcsZt9ZNKD}
}

@inproceedings{10.1145/3726302.3729971,
author = {Heuss, Maria and others},
title = {RankingSHAP - Faithful Listwise Feature Attribution Explanations for Ranking Models},
year = {2025},
isbn = {9798400715921},
publisher = {Association for Computing Machinery},
address = {New York, NY, USA},
url = {https://doi.org/10.1145/3726302.3729971},
doi = {10.1145/3726302.3729971},
booktitle = {Proceedings of the 48th International ACM SIGIR Conference on Research and Development in Information Retrieval},
pages = {381–391},
numpages = {11},
location = {Padua, Italy},
series = {SIGIR '25}
}

@inproceedings{fel2023craft,
  title={Craft: Concept recursive activation factorization for explainability},
  author={Fel, Thomas and Picard, Agustin and Bethune, Louis and Boissin, Thibaut and Vigouroux, David and Colin, Julien and Cad{\`e}ne, R{\'e}mi and Serre, Thomas},
  booktitle={Proceedings of the IEEE/CVF Conference on Computer Vision and Pattern Recognition},
  pages={2711--2721},
  year={2023}
}

@inproceedings{plummer2020these,
  title={Why do these match? explaining the behavior of image similarity models},
  author={Plummer, Bryan A and Vasileva, Mariya I and Petsiuk, Vitali and Saenko, Kate and Forsyth, David},
  booktitle={European Conference on Computer Vision},
  pages={652--669},
  year={2020},
  organization={Springer}
}

@article{eberle2020building,
  title={Building and interpreting deep similarity models},
  author={Eberle, Oliver and B{\"u}ttner, Jochen and Kr{\"a}utli, Florian and M{\"u}ller, Klaus-Robert and Valleriani, Matteo and Montavon, Gr{\'e}goire},
  journal={IEEE Transactions on Pattern Analysis and Machine Intelligence},
  volume={44},
  number={3},
  pages={1149--1161},
  year={2020},
  publisher={IEEE}
}

@article{lin2021xcos,
  title={xcos: An explainable cosine metric for face verification task},
  author={Lin, Yu-Sheng and Liu, Zhe-Yu and Chen, Yu-An and Wang, Yu-Siang and Chang, Ya-Liang and Hsu, Winston H},
  journal={ACM Transactions on Multimedia Computing, Communications, and Applications (TOMM)},
  volume={17},
  number={3s},
  pages={1--16},
  year={2021},
  publisher={ACM New York, NY}
}

@inproceedings{williford2020explainable,
  title={Explainable face recognition},
  author={Williford, Jonathan R and others},
  booktitle={European conference on computer vision},
  pages={248--263},
  year={2020},
  organization={Springer}
}

@misc{Linear_hypothesis,
      title={From Flat to Hierarchical: Extracting Sparse Representations with Matching Pursuit}, 
      author={Valérie Costa  and others},
      year={2025},
      eprint={2506.03093},
      archivePrefix={arXiv},
      primaryClass={cs.LG},
}

@inproceedings{
fel2026into,
title={Into the Rabbit Hull: From Task-Relevant Concepts in {DINO} to Minkowski Geometry},
author={Thomas Fel and Binxu Wang and Michael A. Lepori and Matthew Kowal and Andrew Lee and Randall Balestriero and Sonia Joseph and Ekdeep Singh Lubana and Talia Konkle and Demba E. Ba and Martin Wattenberg},
booktitle={The Fourteenth International Conference on Learning Representations},
year={2026},
url={https://openreview.net/forum?id=pbbRmnFkjG}
}

@misc{bhalla2026sparseautoencoderscaptureconcept,
      title={Do Sparse Autoencoders Capture Concept Manifolds?}, 
      author={Usha Bhalla and Thomas Fel and Can Rager and Sheridan Feucht and Tal Haklay and Daniel Wurgaft and Siddharth Boppana and Matthew Kowal and Vasudev Shyam and Jack Merullo and Atticus Geiger and Ekdeep Singh Lubana},
      year={2026},
      eprint={2604.28119},
      archivePrefix={arXiv},
      primaryClass={cs.LG},
      url={https://arxiv.org/abs/2604.28119}, 
}

@inproceedings{gam1,
author = {Barkan, Oren and Armstrong, Omri and Hertz, Amir and Caciularu, Avi and Katz, Ori and Malkiel, Itzik and Koenigstein, Noam},
title = {GAM: Explainable Visual Similarity and Classification via Gradient Activation Maps},
year = {2021},
isbn = {9781450384469},
publisher = {Association for Computing Machinery},
address = {New York, NY, USA},
url = {https://doi.org/10.1145/3459637.3482430},
doi = {10.1145/3459637.3482430},
booktitle = {Proceedings of the 30th ACM International Conference on Information \& Knowledge Management},
pages = {68–77},
numpages = {10},
location = {Virtual Event, Queensland, Australia},
series = {CIKM '21}
}

@techreport{krizhevsky2009learning,
  title={Learning multiple layers of features from tiny images},
  author={Krizhevsky, Alex and Nair, Vinod and Hinton, Geoffrey},
  year={2009},
  institution={University of Toronto}
}

@article{chen2023sim2word,
  title={Sim2word: Explaining similarity with representative attribute words via counterfactual explanations},
  author={Chen, Ruoyu and Li, Jingzhi and Zhang, Hua and Sheng, Changchong and Liu, Li and Cao, Xiaochun},
  journal={ACM Transactions on Multimedia Computing, Communications and Applications},
  volume={19},
  number={6},
  pages={1--22},
  year={2023},
  publisher={ACM New York, NY}
}

@misc{umap,
      title={UMAP: Uniform Manifold Approximation and Projection for Dimension Reduction}, 
      author={Leland McInnes and others},
      year={2020},
      eprint={1802.03426},
      archivePrefix={arXiv},
      primaryClass={stat.ML}
}

@inproceedings{ood_exp_paper,
 author = {Hase, Peter and others},
 booktitle = {Advances in Neural Information Processing Systems},
 editor = {M. Ranzato and A. Beygelzimer and Y. Dauphin and P.S. Liang and J. Wortman Vaughan},
 pages = {3650--3666},
 publisher = {Curran Associates, Inc.},
 title = {The Out-of-Distribution Problem in Explainability and Search Methods for Feature Importance Explanations},
 volume = {34},
 year = {2021}
}

@inproceedings{holistic_concepts,
 author = {Fel, Thomas and Boutin, Victor and B\'{e}thune, Louis and Cadene, Remi and Moayeri, Mazda and And\'{e}ol, L\'{e}o and Chalvidal, Mathieu and Serre, Thomas},
 booktitle = {Advances in Neural Information Processing Systems},
 editor = {A. Oh and T. Naumann and A. Globerson and K. Saenko and M. Hardt and S. Levine},
 pages = {54805--54818},
 publisher = {Curran Associates, Inc.},
 title = {A Holistic Approach to Unifying Automatic Concept Extraction and Concept Importance Estimation},
 volume = {36},
 year = {2023}
}

@article{sun2022knnood,
  title={Out-of-distribution Detection with Deep Nearest Neighbors},
  author={Sun, Yiyou and Ming, Yifei and Zhu, Xiaojin and Li, Yixuan},
  journal={ICML},
  year={2022}
}

@book{wasserstein,
year = {2009},
publisher = {Springer},
author = {Villani, Cédric},
address = {Berlin},
booktitle = {Optimal transport old and new},
series = {Grundlehren der mathematischen Wissenschaften = A series of comprehensive studies in mathematics, 338},
isbn = {9783540710509},
language = {eng},
title = {Optimal transport old and new },
url = {http://search.ebscohost.com/login.aspx?direct=true&scope=site&db=nlebk&db=nlabk&AN=261958},
}

@inproceedings{malkiel2022interpreting,
  title={Interpreting bert-based text similarity via activation and saliency maps},
  author={Malkiel, Itzik and Ginzburg, Dvir and Barkan, Oren and Caciularu, Avi and Weill, Jonathan and Koenigstein, Noam},
  booktitle={Proceedings of the ACM Web Conference 2022},
  pages={3259--3268},
  year={2022}
}

@inproceedings{
  borowski2021exemplary,
  title={Exemplary Natural Images Explain {CNN} Activations Better than State-of-the-Art Feature Visualization},
  author={Judy Borowski and Roland Simon Zimmermann and Judith Schepers and Robert Geirhos and Thomas S. A. Wallis and Matthias Bethge and Wieland Brendel},
  booktitle={International Conference on Learning Representations},
  year={2021},
  url={https://openreview.net/forum?id=QO9-y8also-}
}

@inproceedings{moeller2024approximate,
  title={Approximate attributions for off-the-shelf Siamese transformers},
  author={Moeller, Lucas and others},
  booktitle={Proceedings of the 18th Conference of the European Chapter of the Association for Computational Linguistics (Volume 1: Long Papers)},
  pages={2059--2071},
  year={2024}
}

@inproceedings{moeller2023attribution,
  title={An attribution method for Siamese encoders},
  author={Moeller, Lucas and others},
  booktitle={Proceedings of the 2023 Conference on Empirical Methods in Natural Language Processing},
  pages={15818--15827},
  year={2023}
}

@inproceedings{opitz2025interpretable,
  title={Interpretable text embeddings and text similarity explanation: A survey},
  author={Opitz, Juri and Moeller, Lucas and Michail, Andrianos and Pad{\'o}, Sebastian and Clematide, Simon},
  booktitle={Proceedings of the 2025 Conference on Empirical Methods in Natural Language Processing},
  pages={22314--22330},
  year={2025}
}

@article{mohan2023deep,
  title={Deep metric learning for computer vision: A brief overview},
  author={Mohan, Deen Dayal and Jawade, Bhavin and Setlur, Srirangaraj and Govindaraju, Venu},
  journal={Handbook of Statistics},
  volume={48},
  pages={59--79},
  year={2023},
  publisher={Elsevier}
}

@inproceedings{tan2019learning,
  title={Learning similarity conditions without explicit supervision},
  author={Tan, Reuben and Vasileva, Mariya I and Saenko, Kate and Plummer, Bryan A},
  booktitle={Proceedings of the IEEE/CVF international conference on computer vision},
  pages={10373--10382},
  year={2019}
}

@InProceedings{Hsiao_2018_CVPR,
author = {Hsiao, Wei-Lin and Grauman, Kristen},
title = {Creating Capsule Wardrobes From Fashion Images},
booktitle = {Proceedings of the IEEE Conference on Computer Vision and Pattern Recognition (CVPR)},
month = {June},
year = {2018}
}

@inproceedings{radenovic2018revisiting,
  title={Revisiting oxford and paris: Large-scale image retrieval benchmarking},
  author={Radenovi{\'c}, Filip and Iscen, Ahmet and Tolias, Giorgos and Avrithis, Yannis and Chum, Ond{\v{r}}ej},
  booktitle={Proceedings of the IEEE conference on computer vision and pattern recognition},
  pages={5706--5715},
  year={2018}
}

@article{chen2022deep,
  title={Deep learning for instance retrieval: A survey},
  author={Chen, Wei and Liu, Yu and Wang, Weiping and Bakker, Erwin M and Georgiou, Theodoros and Fieguth, Paul and Liu, Li and Lew, Michael S},
  journal={IEEE Transactions on Pattern Analysis and Machine Intelligence},
  volume={45},
  number={6},
  pages={7270--7292},
  year={2022},
  publisher={IEEE}
}

@article{sun2014deep,
  title={Deep learning face representation by joint identification-verification},
  author={Sun, Yi and Chen, Yuheng and Wang, Xiaogang and Tang, Xiaoou},
  journal={Advances in neural information processing systems},
  volume={27},
  year={2014}
}

@article{gururaj2024comprehensive,
  title={A comprehensive review of face recognition techniques, trends, and challenges},
  author={Gururaj, HL and Soundarya, BC and Priya, S and Shreyas, J and Flammini, Francesco},
  journal={IEEE Access},
  volume={12},
  pages={107903--107926},
  year={2024},
  publisher={IEEE}
}

@misc{hindupur2025projectingassumptionsdualitysparse,
      title={Projecting Assumptions: The Duality Between Sparse Autoencoders and Concept Geometry}, 
      author={Sai Sumedh R. Hindupur and Ekdeep Singh Lubana and Thomas Fel and Demba Ba},
      year={2025},
      eprint={2503.01822},
      archivePrefix={arXiv},
      primaryClass={cs.LG},
      url={https://arxiv.org/abs/2503.01822}, 
}

@misc{2025dinov3,
      title={DINOv3}, 
      author={Sim{\'e}oni, Oriane and Vo, Huy V and Seitzer, Maximilian and Baldassarre, Federico and Oquab, Maxime and others},
      year={2025},
      eprint={2508.10104},
      archivePrefix={arXiv},
      primaryClass={cs.CV},
      note={\url{https://doi.org/qpjq}}, 
}

@inproceedings{csim, 
author={Roberts, Isaac and Velioglu, Riza and Ashram, Inaam and Hermes, Luca and Hammer, Barabara},booktitle = {ESANN 2026 proceedings}, location = {Bruges (Belgium) and online}, publisher = {Ciaco - i6doc.com}, title = {{But Are These Images Conceptually Similar?}}, year = {2026}, }

@article{nieradzik2025reliable,
  title={Reliable evaluation of attribution maps in cnns: A perturbation-based approach},
  author={Nieradzik, Lars and others},
  journal={International journal of computer vision},
  volume={133},
  number={5},
  pages={2392--2409},
  year={2025},
  publisher={Springer}
}

@inproceedings{fel2025archetypal,
  title={Archetypal SAE: Adaptive and Stable Dictionary Learning for Concept Extraction in Large Vision Models},
  author={Fel, Thomas and Lubana, Ekdeep Singh and Prince, Jacob S and Kowal, Matthew and Boutin, Victor and Papadimitriou, Isabel and Wang, Binxu and Wattenberg, Martin and Ba, Demba E and Konkle, Talia},
  booktitle={International Conference on Machine Learning},
  pages={16543--16572},
  year={2025},
  organization={PMLR}
}

@article{rajamanoharan2024jumping,
  title={Jumping ahead: Improving reconstruction fidelity with jumprelu sparse autoencoders},
  author={Rajamanoharan, Senthooran and Lieberum, Tom and Sonnerat, Nicolas and Conmy, Arthur and Varma, Vikrant and Kram{\'a}r, J{\'a}nos and Nanda, Neel},
  journal={arXiv preprint arXiv:2407.14435},
  year={2024}
}

@article{bricken2023monosemanticity,
       title={Towards Monosemanticity: Decomposing Language Models With Dictionary Learning},
       author={Bricken, Trenton and Templeton, Adly and Batson, Joshua and Chen, Brian and Jermyn, Adam and Conerly, Tom and Turner, Nick and Anil, Cem and Denison, Carson and Askell, Amanda and Lasenby, Robert and Wu, Yifan and Kravec, Shauna and Schiefer, Nicholas and Maxwell, Tim and Joseph, Nicholas and Hatfield-Dodds, Zac and Tamkin, Alex and Nguyen, Karina and McLean, Brayden and Burke, Josiah E and Hume, Tristan and Carter, Shan and Henighan, Tom and Olah, Christopher},
       year={2023},
       journal={Transformer Circuits Thread},
       note={https://transformer-circuits.pub/2023/monosemantic-features/index.html}
    }

@article{longo2024explainable,
  title={Explainable Artificial Intelligence (XAI) 2.0: A manifesto of open challenges and interdisciplinary research directions},
  author={Longo, Luca and Brcic, Mario and Cabitza, Federico and Choi, Jaesik and Confalonieri, Roberto and Del Ser, Javier and Guidotti, Riccardo and Hayashi, Yoichi and Herrera, Francisco and Holzinger, Andreas and others},
  journal={Information Fusion},
  volume={106},
  pages={102301},
  year={2024},
  publisher={Elsevier}
}

@inproceedings{
moshkovitz2026position,
title={Position: Explainability Research Must Prioritize Foundations over Ad-hoc Methods},
author={Michal Moshkovitz and Suraj Srinivas and Lesia Semenova and Nave Frost and Cyrus Rashtchian and Valentyn Boreiko and Shichang Zhang and Himabindu Lakkaraju and Cynthia Rudin and Jennifer Wortman Vaughan},
booktitle={Forty-third International Conference on Machine Learning Position Paper Track},
year={2026},
url={https://openreview.net/forum?id=HaduBh7EaL}
}

@inproceedings{NEURIPS2024_56ed2bd1,
 author = {Zimmermann, Roland S. and others},
 booktitle = {Advances in Neural Information Processing Systems},
 doi = {10.52202/079017-1535},
 editor = {A. Globerson and L. Mackey and D. Belgrave and A. Fan and U. Paquet and J. Tomczak and C. Zhang},
 pages = {48448--48483},
 publisher = {Curran Associates, Inc.},
 title = {Measuring Per-Unit Interpretability at Scale Without Humans},
 url = {https://proceedings.neurips.cc/paper_files/paper/2024/file/56ed2bd15b66f709cd81cb1aaa0496b9-Paper-Conference.pdf},
 volume = {37},
 year = {2024}
}


\clearpage
\FloatBarrier
\appendix

\section{Supplementary Materials}

\subsection{Hardware}
For all necessary training in the experiments, we used NVIDIA A40 (46 GB) GPUs and 120 CPUs. The installed CUDA version is 12.6 with NVIDIA driver 560.35.05. All experiments were run using Python 3.11. Our implementation is built on PyTorch 2.10 and torchvision 0.25, with
concept extraction performed via the \texttt{overcomplete} library (0.3.0) and pretrained backbones loaded through \texttt{timm} (1.0.24). Evaluation and analysis rely on scikit-learn (1.8.0), SciPy (1.17.0), NumPy (2.3.5), POT (0.9.6) for the Wasserstein distance, and \texttt{umap-learn} (0.5.11) for dimensionality reduction. A complete list of pinned dependencies is provided in the accompanying code.

\section{Theory}

We demonstrate that under the Linear Representation Hypothesis \cite{Linear_hypothesis} and when $f$ is a linear dot product, the sum of the concept activations is exactly twice the computed dot product.
\begin{theorem}[Completeness Theorem]
Let \( f(\mathbf{x}, \mathbf{y}) = \mathbf{x}^\top \mathbf{y} \) denote the similarity function. 
Assume that each activation vector \( \mathbf{a}_k \in \mathbb{R}^d \) admits an exact reconstruction from a dictionary 
\( \mathbf{V} = [\mathbf{v}_1, \dots, \mathbf{v}_c] \in \mathbb{R}^{d \times c} \) with coefficients 
\( \mathbf{u}_k \in \mathbb{R}^c \), i.e.,
\[
\mathbf{a}_k = \sum_{i=1}^{c} u_{k,i}\,\mathbf{v}_i , \qquad k \in \{1,2\}.
\]
For a given concept \( i \), define the perturbed activation obtained by removing its contribution as
\[
\mathbf{a}_k^{(-i)} \coloneqq \mathbf{a}_k - u_{k,i}\,\mathbf{v}_i .
\]
Define the importance of concept \( i \) for the similarity between \( \mathbf{a}_1 \) and \( \mathbf{a}_2 \) as
\[
\Delta c_i
=
\Big( f(\mathbf{a}_1,\mathbf{a}_2) - f(\mathbf{a}_1^{(-i)},\mathbf{a}_2) \Big)
+
\Big( f(\mathbf{a}_1,\mathbf{a}_2) - f(\mathbf{a}_1,\mathbf{a}_2^{(-i)}) \Big).
\]
Then the importances satisfy the completeness property
\[
\sum_{i=1}^{c} \Delta c_i
=
2\,f(\mathbf{a}_1,\mathbf{a}_2).
\]
\end{theorem}

\begin{proof}
Using bilinearity of the dot product, the first term becomes
\[
\begin{aligned}
f(\mathbf{a}_1,\mathbf{a}_2)
-
f(\mathbf{a}_1^{(-i)},\mathbf{a}_2)
&=
\mathbf{a}_1^\top \mathbf{a}_2
-
(\mathbf{a}_1 - u_{1,i}\mathbf{v}_i)^\top \mathbf{a}_2 \\
&=
u_{1,i}\,\mathbf{v}_i^\top \mathbf{a}_2 .
\end{aligned}
\]
Similarly,
\[
\begin{aligned}
f(\mathbf{a}_1,\mathbf{a}_2)
-
f(\mathbf{a}_1,\mathbf{a}_2^{(-i)})
&=
\mathbf{a}_1^\top \mathbf{a}_2
-
\mathbf{a}_1^\top (\mathbf{a}_2 - u_{2,i}\mathbf{v}_i) \\
&=
u_{2,i}\,\mathbf{a}_1^\top \mathbf{v}_i .
\end{aligned}
\]
Hence
\[
\Delta c_i
=
u_{1,i}\,\mathbf{v}_i^\top \mathbf{a}_2
+
u_{2,i}\,\mathbf{v}_i^\top \mathbf{a}_1 .
\]
Summing over all concepts,
\[
\begin{aligned}
\sum_{i=1}^{c} \Delta c_i
&=
\sum_{i=1}^{c}
\Big(
u_{1,i}\,\mathbf{v}_i^\top \mathbf{a}_2
+
u_{2,i}\,\mathbf{v}_i^\top \mathbf{a}_1
\Big) \\
&=
\left( \sum_{i=1}^{c} u_{1,i}\mathbf{v}_i \right)^\top \mathbf{a}_2
+
\left( \sum_{i=1}^{c} u_{2,i}\mathbf{v}_i \right)^\top \mathbf{a}_1 .
\end{aligned}
\]
By the reconstruction identities \( \sum_{i=1}^{c} u_{k,i}\mathbf{v}_i = \mathbf{a}_k \), we obtain
\[
\sum_{i=1}^{c} \Delta c_i
=
\mathbf{a}_1^\top \mathbf{a}_2
+
\mathbf{a}_2^\top \mathbf{a}_1
=
2\,\mathbf{a}_1^\top \mathbf{a}_2
=
2\,f(\mathbf{a}_1,\mathbf{a}_2),
\]
which completes the proof.
\end{proof}

\begin{figure*}[htbp]
    \centering
    
    % --- Top Row: Images ---
    \begin{subfigure}[b]{0.32\textwidth}
        \centering
        \includegraphics[width=\textwidth]{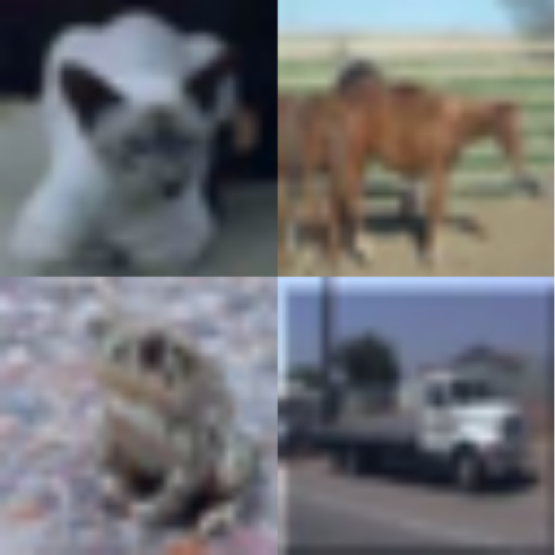}
        \caption{Query}
        \label{fig:discord_query}
    \end{subfigure}
    \hfill
    \begin{subfigure}[b]{0.32\textwidth}
        \centering
        \includegraphics[width=\textwidth]{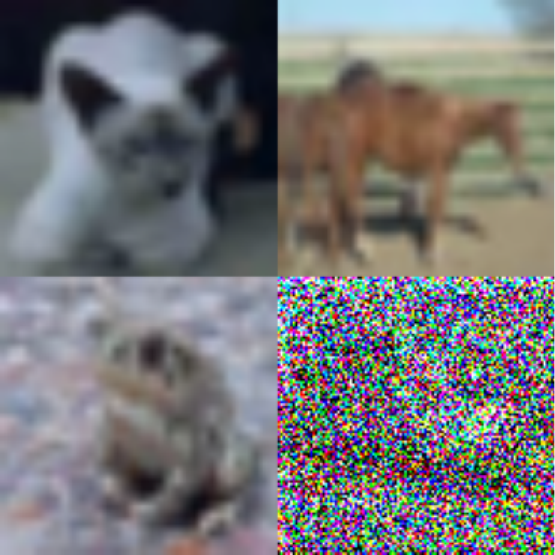}
        \caption{Masked Query}
        \label{fig:discord_masked}
    \end{subfigure}
    \hfill
    \begin{subfigure}[b]{0.32\textwidth}
        \centering
        \includegraphics[width=\textwidth]{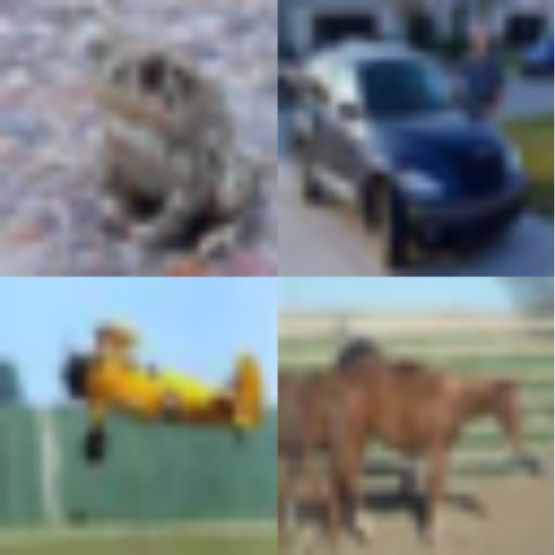}
        \caption{Reference}
        \label{fig:discord_ref}
    \end{subfigure}
    
    % \vspace{1.5em} % Spacing between images and table
    
    % % --- Bottom Row: Table ---
    % \resizebox{\textwidth}{!}{
    %     \begin{tabular}{|l|c|c|c|c|c|c|c|c|c|}
    %         \hline
    %         & \textbf{concept} & \textbf{activation} & \textbf{reconst.} & \textbf{Quad Avg} & \textbf{Blur} & \textbf{Black} & \textbf{White} & \textbf{Random} & \textbf{Gray} \\
    %         \hline
    %         \textbf{$\Delta c$} & 
    %         \cellcolor{red!20}-0.133 & 
    %         \cellcolor{green!20}0.096 & 
    %         \cellcolor{red!20}0.035 & 
    %         \cellcolor{red!20}0.114 & 
    %         \cellcolor{green!20}0.063 & 
    %         \cellcolor{green!20}0.140 & 
    %         \cellcolor{green!20}0.130 & 
    %         \cellcolor{green!20}0.190 & 
    %         \cellcolor{green!20}0.124 \\
    %         \hline
    %     \end{tabular}
    % }/
    
    \caption{Dissimilarity is difficult to track due to underlying similarities.}
    \label{fig:discordant_example_1}
\end{figure*}
\begin{table*}[t]
\centering
\scriptsize
\setlength{\tabcolsep}{4pt}
\renewcommand{\arraystretch}{1.1}

\begin{tabular}{l l l c c c c c c}
\toprule
 &  &  & Batch & Epochs & LR & Top-$k$ & Avg Accuracy $\pm$ std\\
Backbone & SAE & Setting & Size &  &  &  &  & \\
\midrule

\multirow{6}{*}{\rotatebox{90}{ConvNeXt}} & Vanilla & Pretrained & 4 & 20 & 0.005 & -- & $0.99775 \pm 0.00134$ \\
 & Vanilla & Domain-adapted & 64 & 20 & 0.005 & -- & $0.99338 \pm 0.01505$ \\
 & JumpReLU & Pretrained & 4 & 10 & 0.005 & -- & $0.95239 \pm 0.04034$ \\
 & JumpReLU & Domain-adapted & 4 & 10 & 0.005 & -- & $0.97501 \pm 0.03001$ \\
 & TopK & Pretrained & 32 & 20 & 0.005 & 4 & $0.99472 \pm 0.01466$ \\
 & TopK & Domain-adapted & 64 & 20 & 0.005 & 4 & $1.00000 \pm 0.00000$ \\
\midrule
\multirow{6}{*}{\rotatebox{90}{DINOv2}} & Vanilla & Pretrained & 32 & 20 & 0.005 & -- & $0.99407 \pm 0.01084$ \\
 & Vanilla & Domain-adapted & 32 & 20 & 0.005 & -- & $0.99586 \pm 0.01131$ \\
 & JumpReLU & Pretrained & 4 & 10 & 0.005 & -- & $0.96998 \pm 0.03184$ \\
 & JumpReLU & Domain-adapted & 4 & 10 & 0.005 & -- & $0.96027 \pm 0.04334$ \\
 & TopK & Pretrained & 4 & 20 & 0.0005 & 4 & $0.98352 \pm 0.03601$ \\
 & TopK & Domain-adapted & 8 & 20 & 0.005 & 4 & $1.00000 \pm 0.00000$ \\
\midrule
\multirow{6}{*}{\rotatebox{90}{DINOv3}} & Vanilla & Pretrained & 8 & 20 & 0.005 & -- & $0.99658 \pm 0.00596$ \\
 & Vanilla & Domain-adapted & 16 & 20 & 0.005 & -- & $0.99995 \pm 0.00008$ \\
 & JumpReLU & Pretrained & 4 & 20 & 0.005 & -- & $0.95282 \pm 0.03881$ \\
 & JumpReLU & Domain-adapted & 8 & 10 & 0.001 & -- & $0.94698 \pm 0.03874$ \\
 & TopK & Pretrained & 4 & 20 & 0.005 & 4 & $0.98793 \pm 0.01168$ \\
 & TopK & Domain-adapted & 16 & 10 & 0.005 & 4 & $1.00000 \pm 0.00000$ \\
\midrule
\multirow{6}{*}{\rotatebox{90}{ResNet}} & Vanilla & Pretrained & 4 & 10 & 0.005 & -- & $0.97063 \pm 0.03295$ \\
 & Vanilla & Domain-adapted & 8 & 20 & 0.0005 & -- & $0.99790 \pm 0.00593$ \\
 & JumpReLU & Pretrained & 4 & 10 & 0.005 & -- & $0.97268 \pm 0.02713$ \\
 & JumpReLU & Domain-adapted & 8 & 20 & 0.005 & -- & $0.95580 \pm 0.03753$ \\
 & TopK & Pretrained & 16 & 20 & 0.005 & 4 & $0.98214 \pm 0.03399$ \\
 & TopK & Domain-adapted & 64 & 10 & 0.005 & 4 & $1.00000 \pm 0.00000$ \\
\midrule
\multirow{6}{*}{\rotatebox{90}{SigLIP}} & Vanilla & Pretrained & 8 & 20 & 0.005 & -- & $0.97924 \pm 0.02433$ \\
 & Vanilla & Domain-adapted & 32 & 20 & 0.005 & -- & $0.99681 \pm 0.01019$ \\
 & JumpReLU & Pretrained & 4 & 20 & 0.005 & -- & $0.95403 \pm 0.04056$ \\
 & JumpReLU & Domain-adapted & 4 & 20 & 0.005 & -- & $0.95005 \pm 0.03885$ \\
 & TopK & Pretrained & 4 & 20 & 0.001 & 4 & $0.90137 \pm 0.05037$ \\
 & TopK & Domain-adapted & 16 & 20 & 0.0005 & 4 & $1.00000 \pm 0.00000$ \\
\midrule
\multirow{6}{*}{\rotatebox{90}{ViT}} & Vanilla & Pretrained & 4 & 20 & 0.005 & -- & $0.99522 \pm 0.00874$ \\
 & Vanilla & Domain-adapted & 64 & 20 & 0.005 & -- & $0.99805 \pm 0.00869$ \\
 & JumpReLU & Pretrained & 8 & 10 & 0.005 & -- & $0.96600 \pm 0.03460$ \\
 & JumpReLU & Domain-adapted & 4 & 10 & 0.005 & -- & $0.96841 \pm 0.03069$ \\
 & TopK & Pretrained & 32 & 20 & 0.005 & 4 & $0.99805 \pm 0.00925$ \\
 & TopK & Domain-adapted & 32 & 10 & 0.005 & 4 & $1.00000 \pm 0.00000$ \\
\bottomrule
\end{tabular}
\caption{Best SAE hyperparameters for each backbone.}
\label{tab:hyper_results}
\end{table*}
\begin{figure*}[t]
\centering
\setlength{\tabcolsep}{2pt}
\renewcommand{\arraystretch}{1.2}

\newcommand{\myColW}{2.6}

%\begin{tabular}{c c c c c c c}
\begin{tabular}{m{0.7cm} >{\centering\arraybackslash} m{\myColW cm} >{\centering\arraybackslash} m{\myColW cm} >{\centering\arraybackslash} m{\myColW cm} >{\centering\arraybackslash} m{\myColW cm} >{\centering\arraybackslash} m{\myColW cm} >{\centering\arraybackslash} m{\myColW cm}}

% ---------------- Column Headers ----------------
 & \textbf{ConvNeXt} & \textbf{DINOv2} & \textbf{DINOv3} & \textbf{ResNet} & \textbf{SigLIP} & \textbf{ViT} \\%[0.5em]

% ---------------- FID Row ----------------
\textbf{$\mathcal{W}_1$}
& \includegraphics[width=0.15\textwidth]{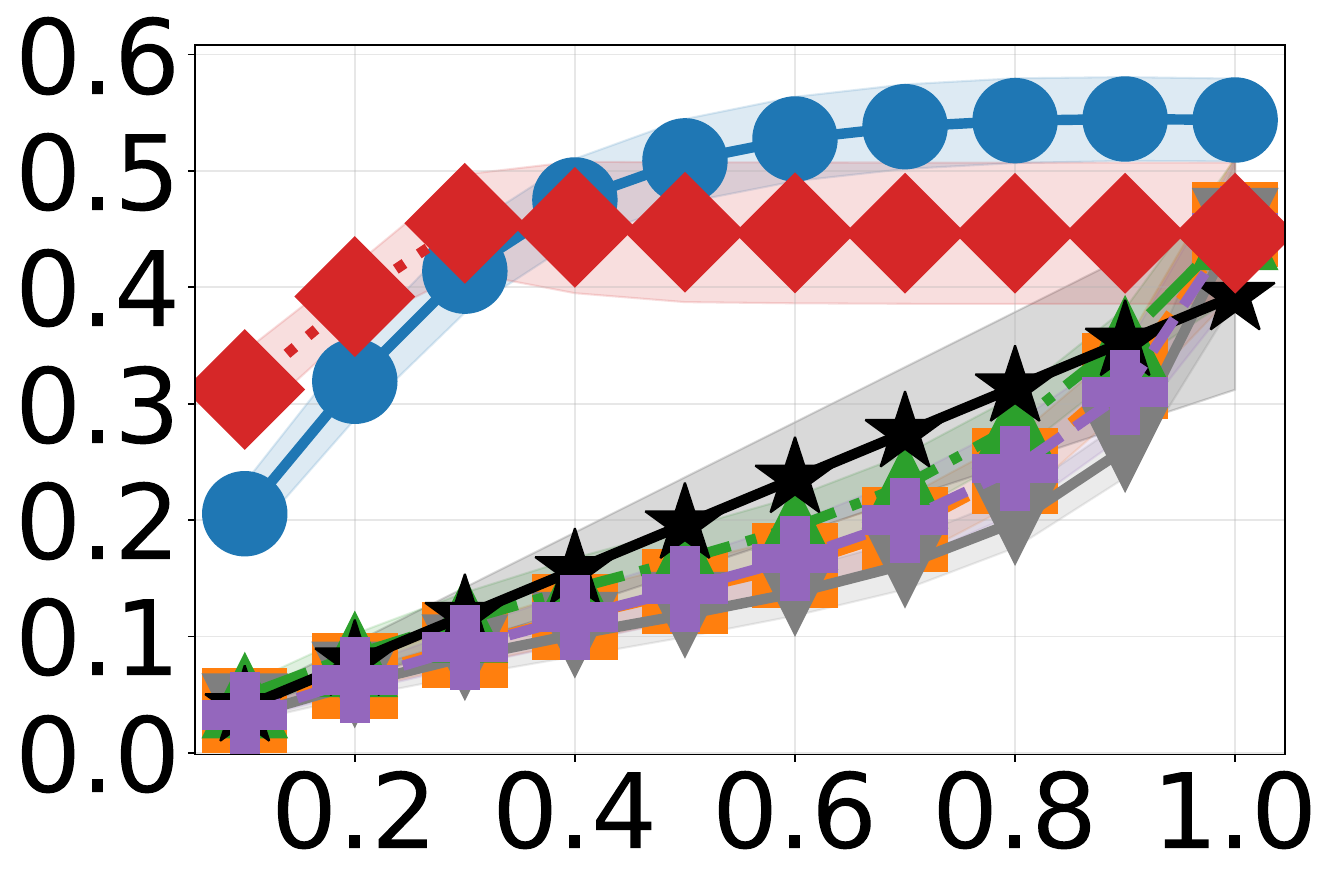}
& \includegraphics[width=0.15\textwidth]{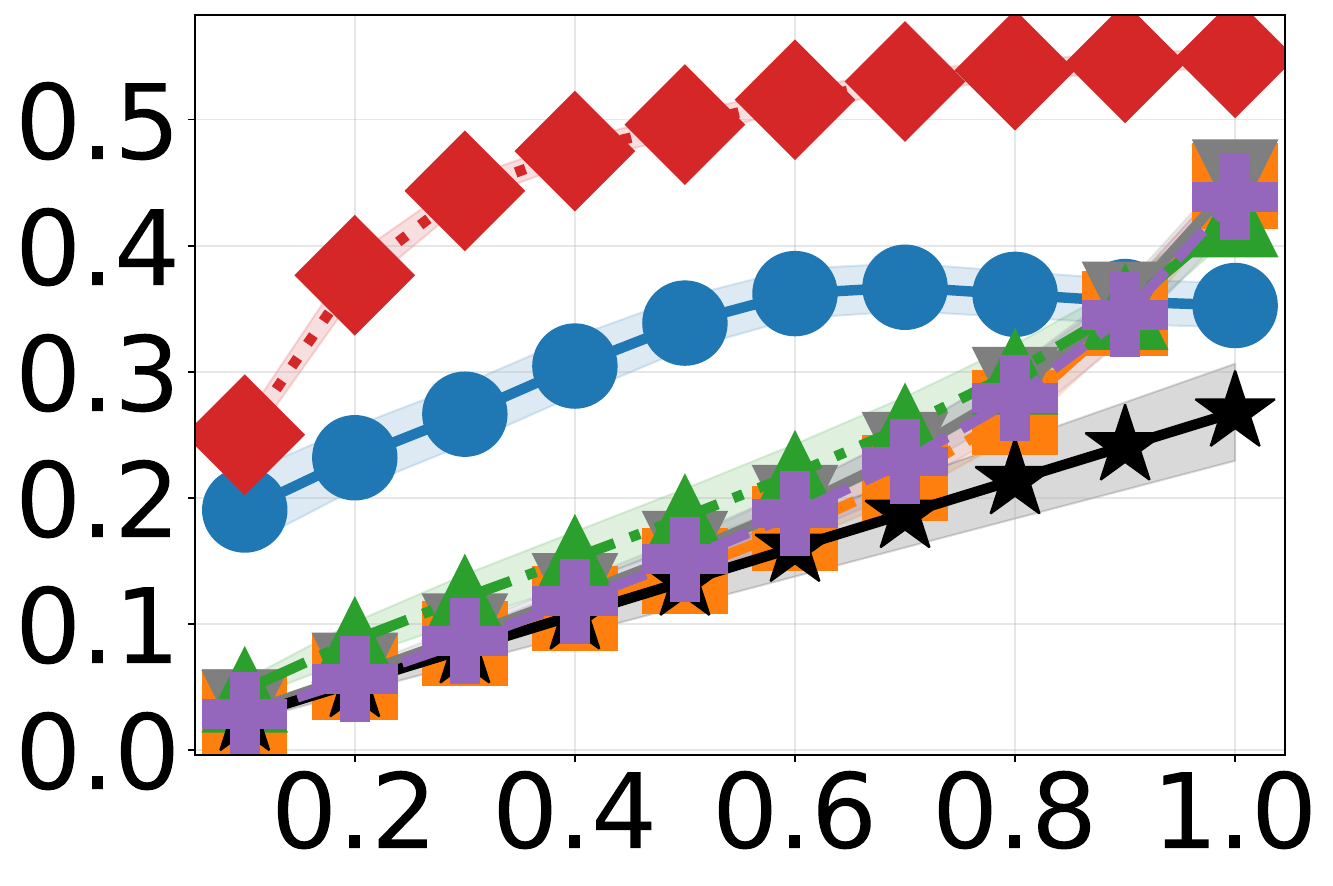}
& \includegraphics[width=0.15\textwidth]{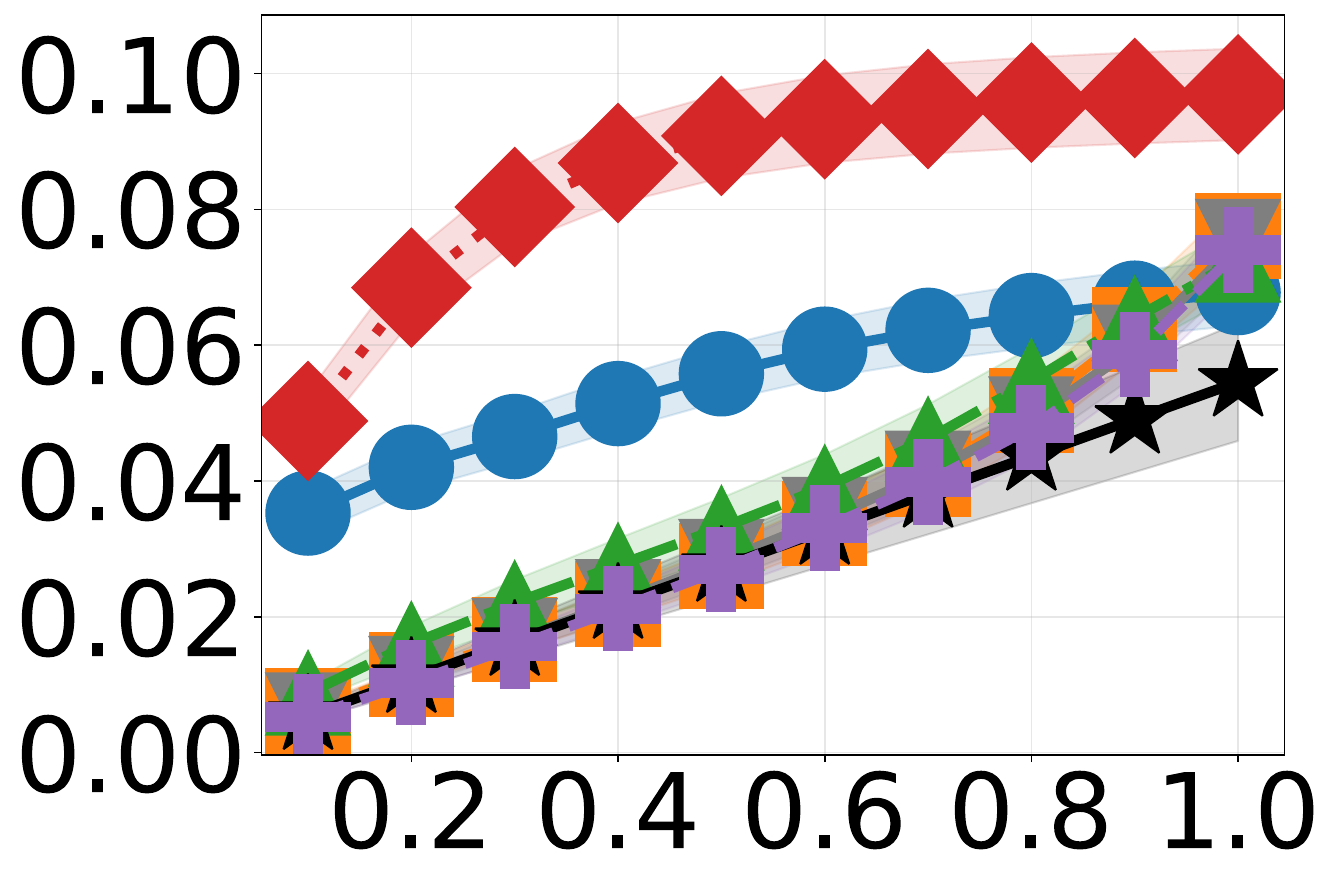}
& \includegraphics[width=0.15\textwidth]{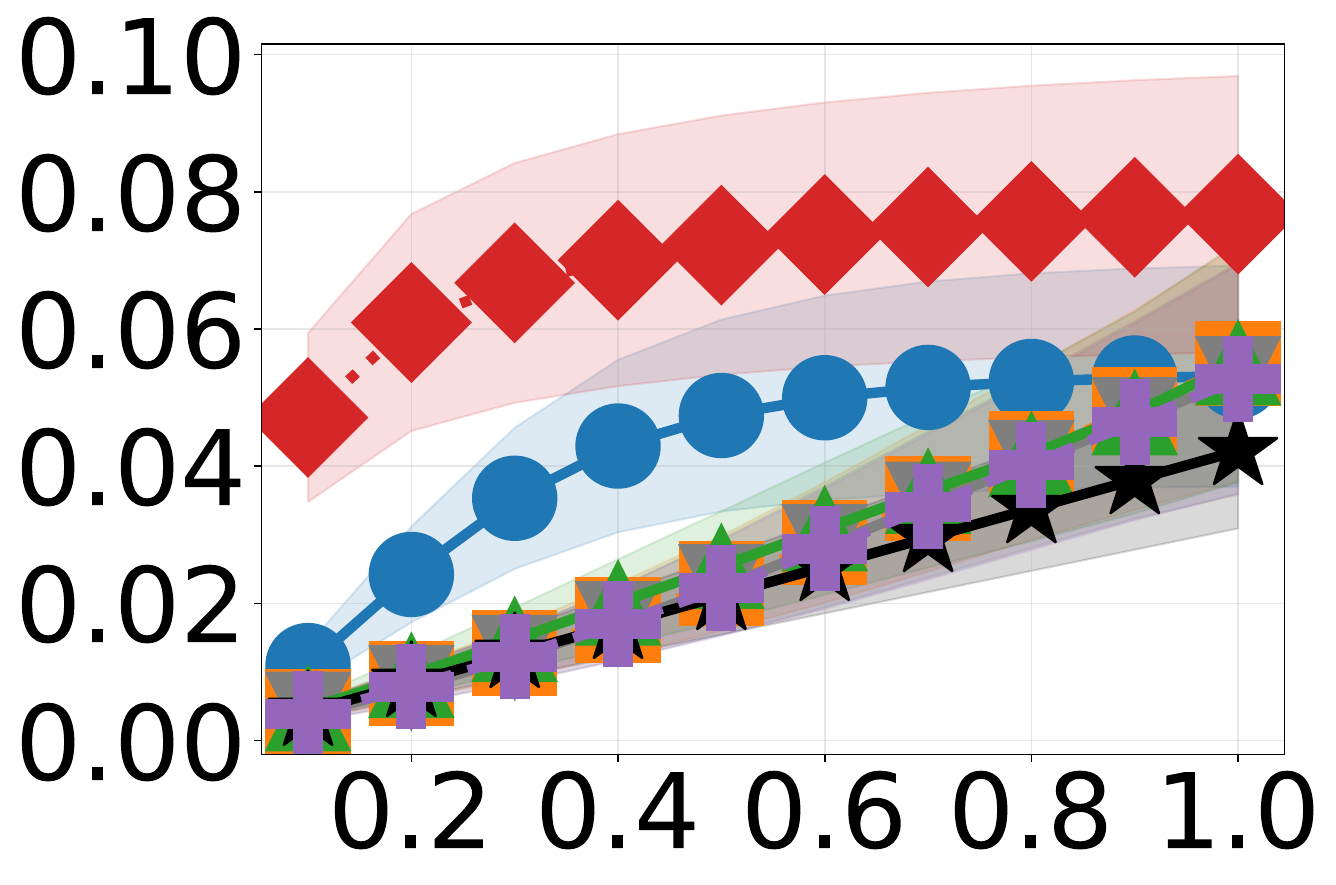}
& \includegraphics[width=0.15\textwidth]{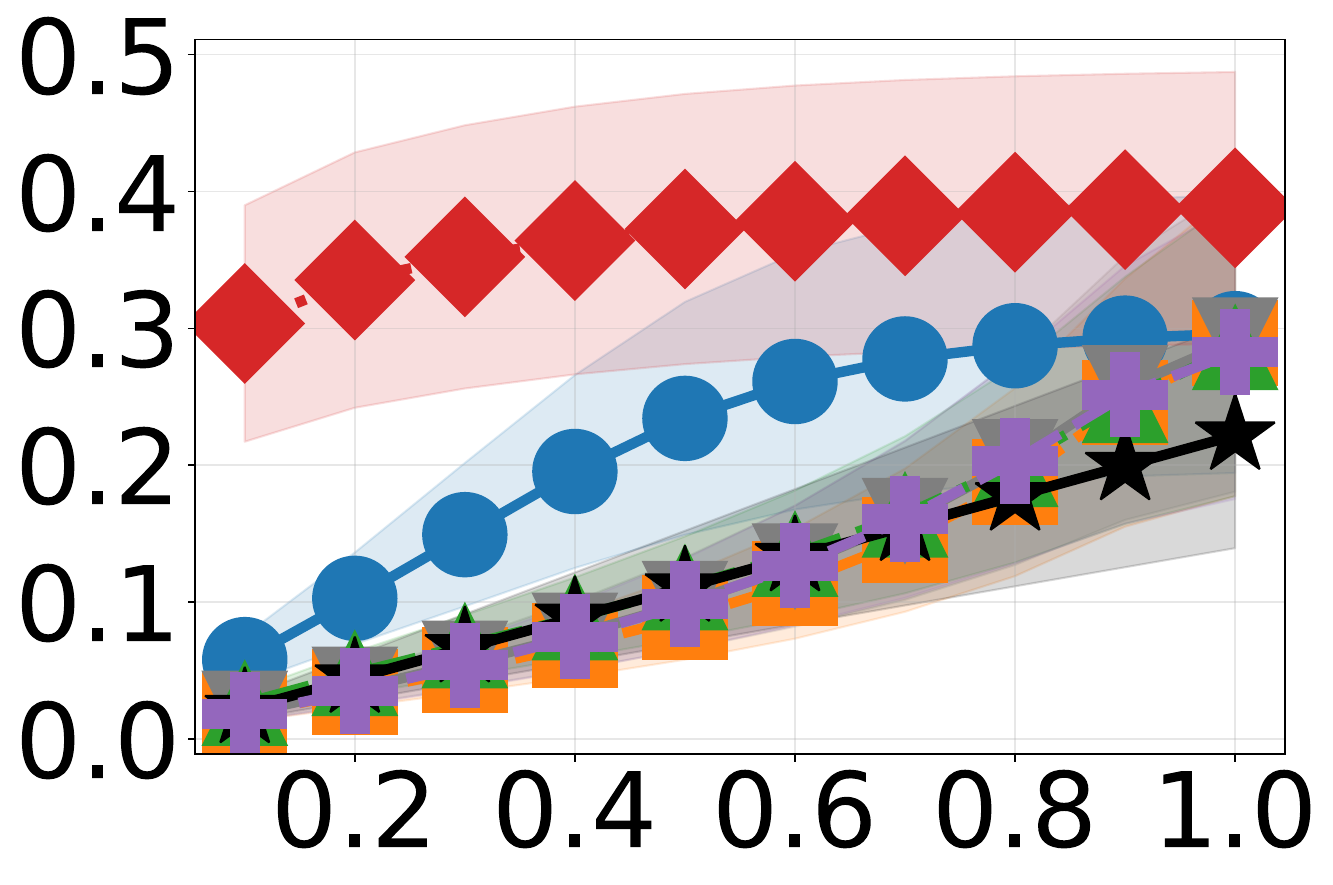}
& \includegraphics[width=0.15\textwidth]{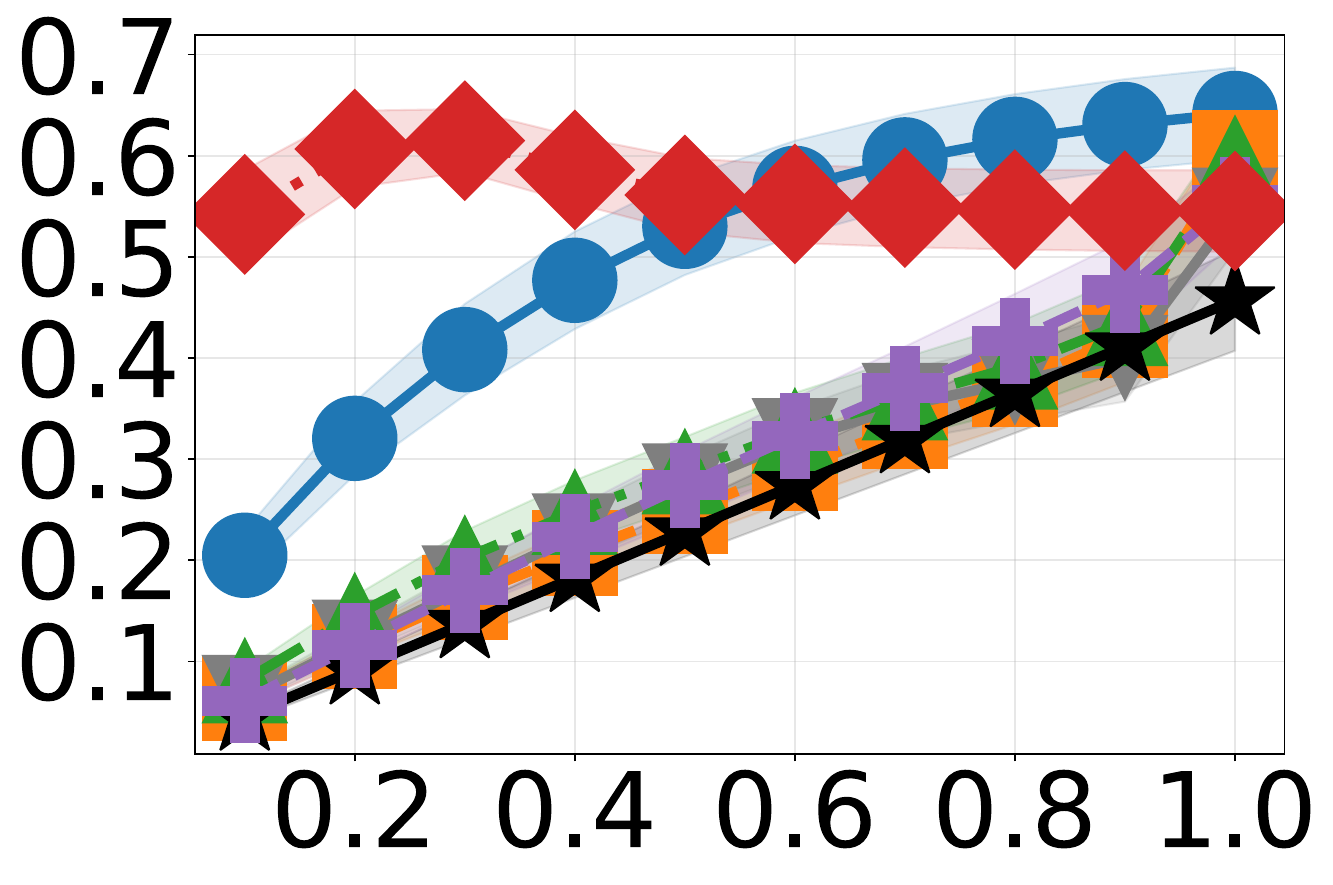}
\\%[1em]

% ---------------- OOD Row ----------------
\textbf{OOD}
& \includegraphics[width=0.15\textwidth]{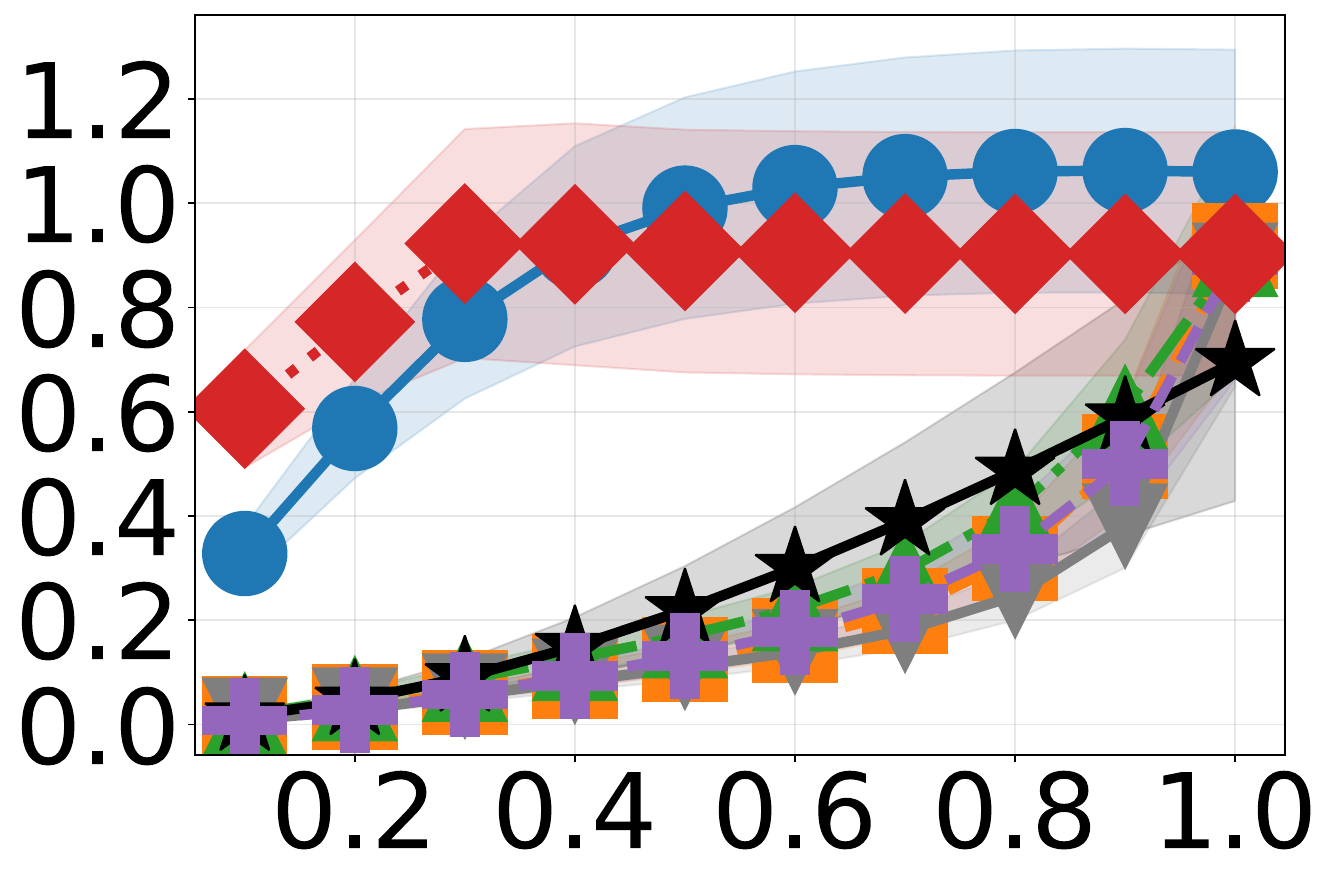}
& \includegraphics[width=0.15\textwidth]{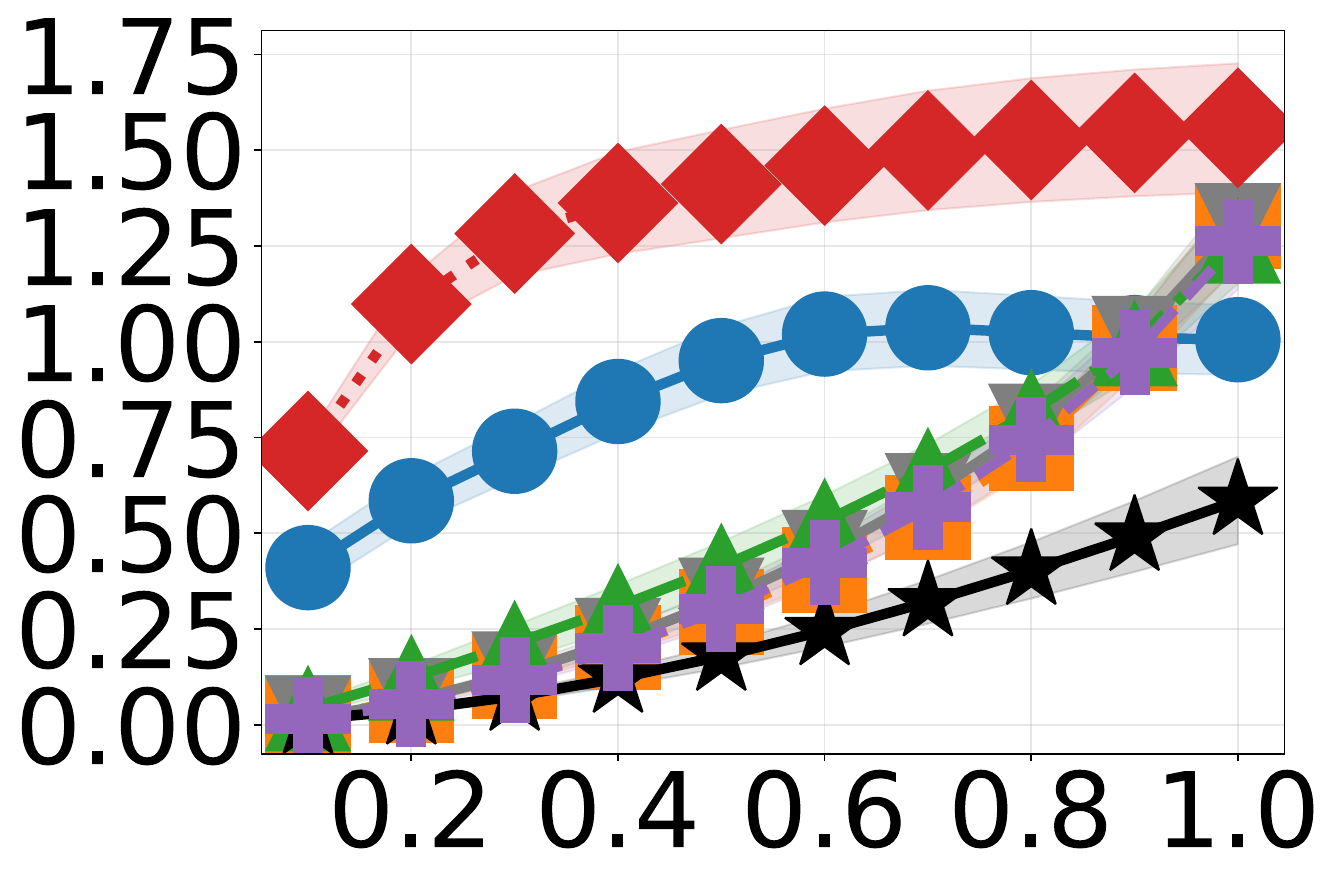}
& \includegraphics[width=0.15\textwidth]{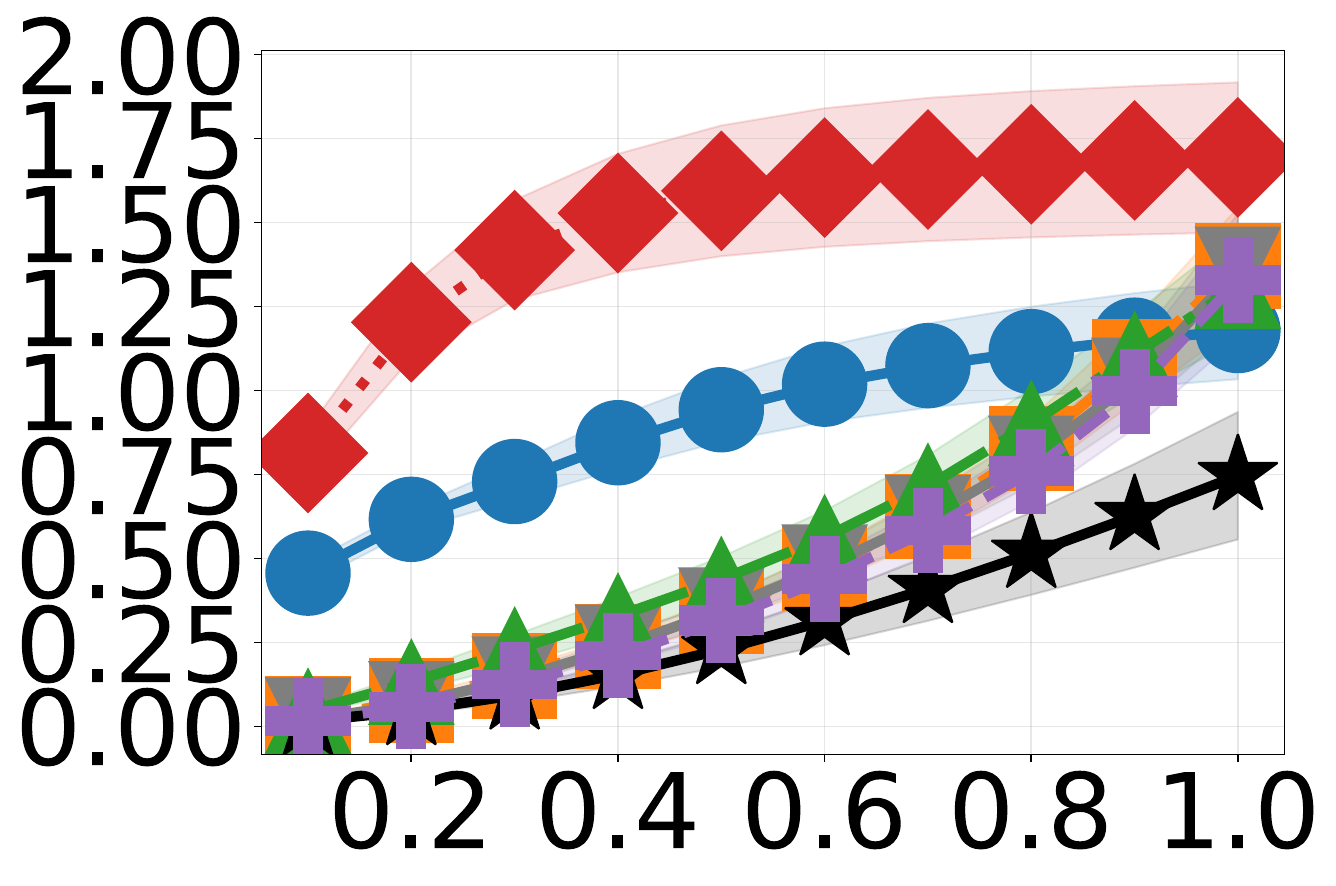}
& \includegraphics[width=0.15\textwidth]{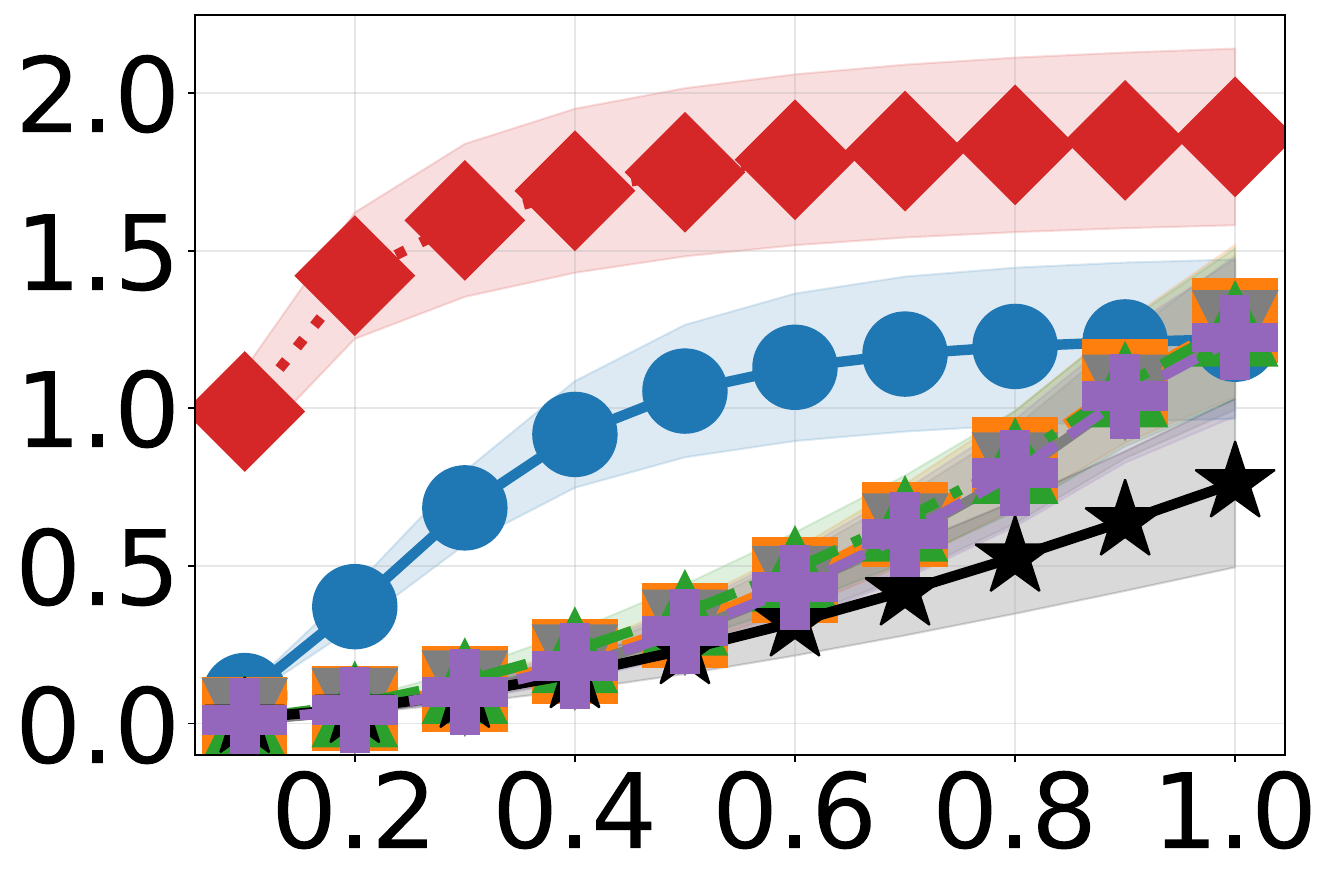}
& \includegraphics[width=0.15\textwidth]{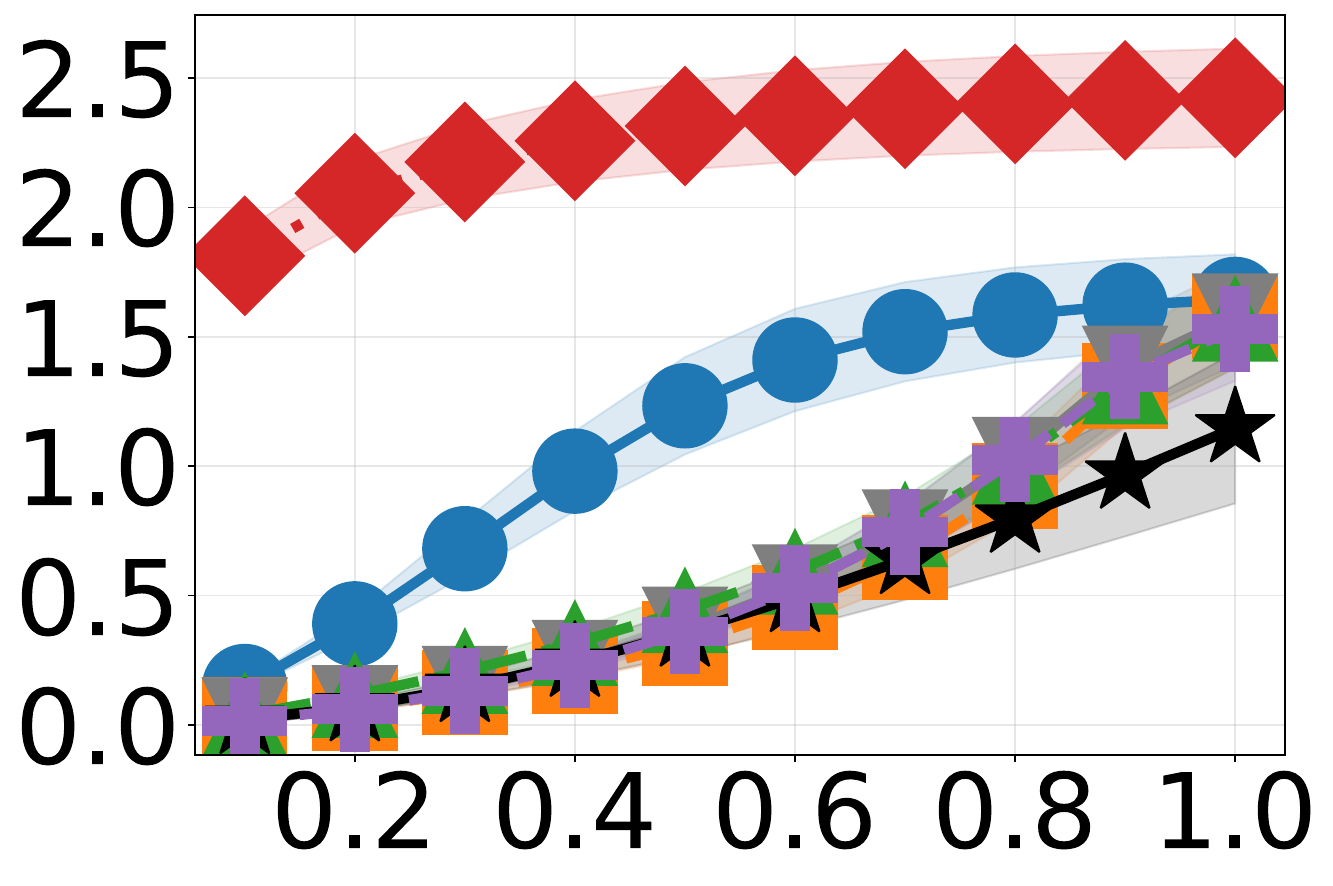}
& \includegraphics[width=0.15\textwidth]{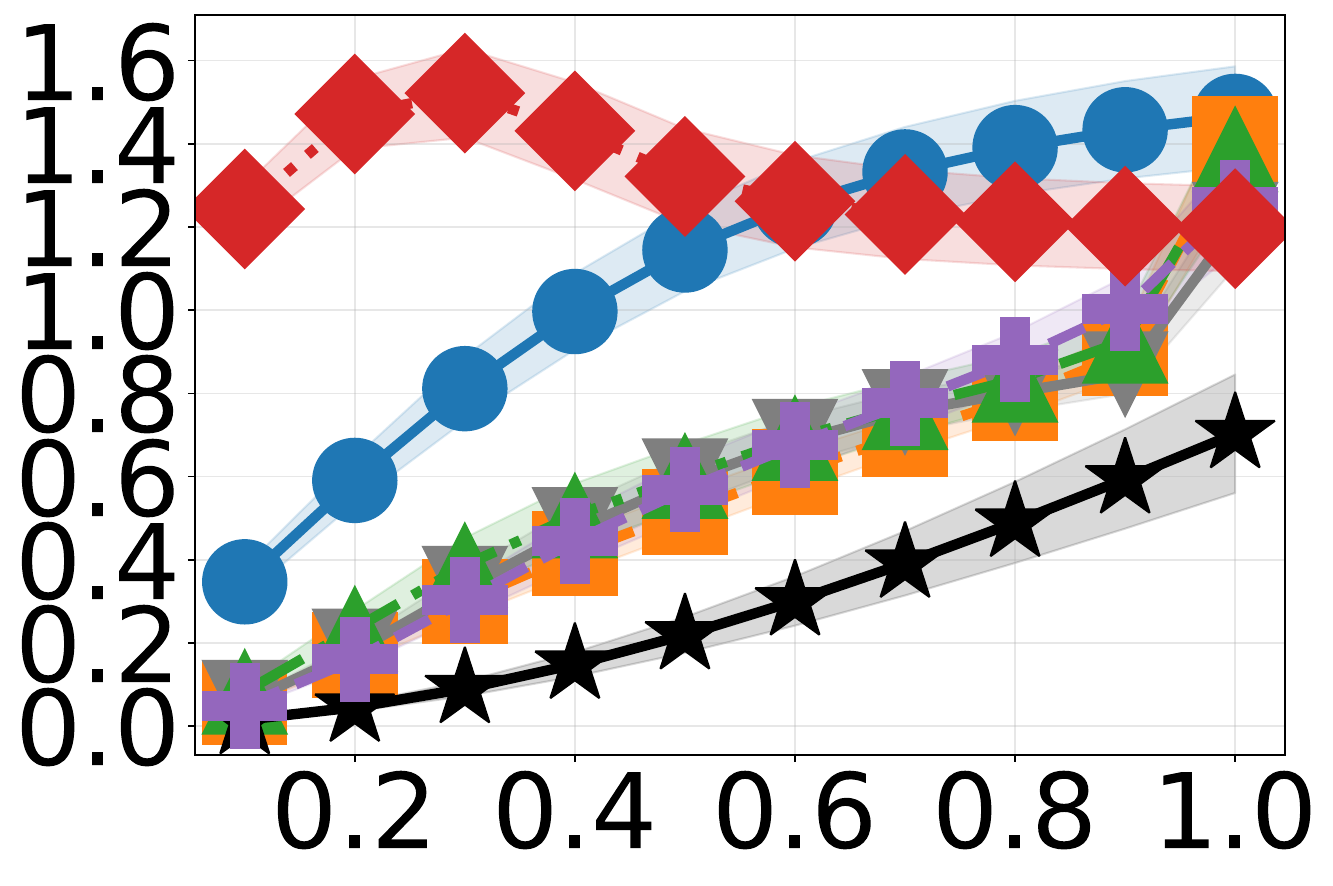}
\\

\end{tabular}

\vspace{0.5em}

% ---------------- Legend ----------------
\begin{tikzpicture}
\begin{axis}[
    hide axis,
    xmin=0, xmax=1,
    ymin=0, ymax=1,
    legend columns=7, % Adjusted to 4 so the 7 items wrap nicely into two rows
    legend style={
        draw=black,
        fill=white,
        legend cell align=left,
        title={Perturbation Type},
        at={(0.5,1.0)},
        anchor=center
    }
]

\addlegendimage{only marks, mark=*, mark size=4pt, color=blue}
\addlegendentry{Blur}

\addlegendimage{only marks, mark=square*, mark size=4pt, color=orange}
\addlegendentry{Darken}

\addlegendimage{only marks, mark=triangle*, mark size=5pt, color=green!60!black}
\addlegendentry{Lighten}

\addlegendimage{only marks, mark=diamond*, mark size=5pt, color=red}
\addlegendentry{Random}

% Gray with a downward pointing triangle (rotated 180 degrees)
\addlegendimage{only marks, mark=triangle*, mark options={rotate=180}, mark size=5pt, color=gray}
\addlegendentry{Gray}

% Quad_Avg with a thick plus sign
\addlegendimage{only marks, mark=+, mark options={line width=2pt}, mark size=5pt, color=purple}
\addlegendentry{Quad\_Avg}

% Ours updated to have both the line and the star marker
\addlegendimage{mark=star, mark size=5pt, line width=2pt, color=black}
\addlegendentry{Ours}

\end{axis}
\end{tikzpicture}

\caption{We illustrate the impact of perturbation severity on $\mathcal{W}_1$ and OOD scores across architectures. The score for our perturbation approach in the embedding space using concepts is indicated by the black line. Rows correspond to fidelity metrics, and columns to model architecture.}
\label{fig:fid_ood_models_ablation}
\end{figure*}
\begin{table*}[t]
\centering
\scriptsize
\setlength{\tabcolsep}{2pt}
\renewcommand{\arraystretch}{1.1}

\begin{tabular}{
l l
S[table-format=2.2]@{\,/\,}S[table-format=2.2]
S[table-format=2.2]@{\,/\,}S[table-format=2.2]
S[table-format=2.2]@{\,/\,}S[table-format=2.2]
S[table-format=2.2]@{\,/\,}S[table-format=2.2]
S[table-format=2.2]@{\,/\,}S[table-format=2.2]
S[table-format=2.2]@{\,/\,}S[table-format=2.2]
}
\toprule
 &  & \multicolumn{12}{c}{$\mathcal{W}_1$ $\downarrow$ / OOD $\downarrow$} \\
\cmidrule(lr){3-14}
 & 
 & \multicolumn{2}{c}{Convnext} & \multicolumn{2}{c}{Dino2} & \multicolumn{2}{c}{Dino3} & \multicolumn{2}{c}{Resnet} & \multicolumn{2}{c}{Siglip} & \multicolumn{2}{c}{Vit} \\
\midrule

\multirow{7}{*}{\rotatebox{90}{Domain-adapted}} & Black & 0.41 & 0.72 & 0.45 & 1.22 & 0.08 & 1.24 & 0.07 & 1.23 & 0.40 & 1.72 & 0.62 & 1.35 \\
 & Gray & 0.40 & 0.68 & 0.45 & 1.21 & 0.08 & 1.24 & 0.07 & 1.21 & 0.40 & 1.73 & 0.56 & 1.19 \\
 & Heavy Blur & 0.54 & 0.86 & 0.34 & 0.95 & 0.07 & 1.08 & 0.07 & 1.17 & 0.40 & 1.73 & 0.67 & 1.41 \\
 & Quadrant Avg & 0.40 & 0.69 & 0.44 & 1.18 & 0.08 & 1.21 & 0.07 & 1.19 & 0.39 & 1.67 & 0.57 & 1.23 \\
 & Random & 0.40 & 0.70 & 0.55 & 1.40 & 0.10 & 1.50 & 0.10 & 1.65 & 0.49 & 2.29 & 0.56 & 1.17 \\
 & White & 0.40 & 0.70 & 0.42 & 1.18 & 0.08 & 1.19 & 0.07 & 1.23 & 0.39 & 1.71 & 0.62 & 1.32 \\
 & Ours & \textbf{0.32} & \textbf{0.47} & \textbf{0.24} & \textbf{0.52} & \textbf{0.05} & \textbf{0.65} & \textbf{0.05} & \textbf{0.68} & \textbf{0.30} & \textbf{1.21} & \textbf{0.45} & \textbf{0.64} \\

\bottomrule
\end{tabular}
\caption{Comparison of $\mathcal{W}_1$ and OOD scores in Activation Space for a Domain-Adapted Embedding.}
\label{tab:ood_perturbations_DA}
\end{table*}

\section{Dataset Details}
\subsection{Dataset} The Multi-CIFAR-10 dataset provides a ground truth for concept presence and overlap, mitigating the ambiguity often found in natural image datasets where concept boundaries are poorly defined. Each sample $\mathbf{x} \in \mathcal{D}$ is generated by arranging four distinct images from the CIFAR-10 dataset into a $2 \times 2$ grid. To ensure compatibility with standard pre-trained vision backbones, we upscale each $32 \times 32$ CIFAR-10 sample to $112 \times 112$ pixels, resulting in a final collage of $224 \times 224$ pixels. 

We define "concepts" in this context as the ten CIFAR-10 object classes (e.g., airplane, bird, car). During generation, we select ten images from each class and randomly sample 4 images to be placed into the quadrants. We enforce that no image is repeated within a single collage to ensure concept uniqueness. This ensures that similarity changes are directly attributable to the input perturbations and not due to other sources such as class variance.

\subsection{Controlling Similarity} A critical component of our setup is the controlled generation of image pairs $(\mathbf{x}_i, \mathbf{x}_j)$ to simulate varying degrees of semantic similarity. A pair of collages is maximally dissimilar if they share no images and maximally similar if they share all images. On this note, we do assume \textbf{permutation invariance}, considering only the "bag of concepts" present regardless of their spatial quadrant. This controlled variation allows us to quantitatively assess the effect of latent concept perturbations against common input perturbations: heavy blur, black pixels, average of the area, random pixels, and white pixels; see Figure \ref{fig:all_perturbations} for an example of the resulting data set and these perturbations.

\subsection{SAE Concept Fixing}\label{subsec:details}
To ensure that no \textbf{polysemanticity} occurs, we restrict our experiments to SAEs where each component is responsible for a single specific concept. Specifically, we identify an atom as representing a concept if its activation yields the maximum classification accuracy for that concept. To measure the stability of this process, we record the average accuracy. These results are located in Table 2 of the Appendix. To select the best models for each SAE type, a grid search is conducted where we vary batch size, learning rate, and number of epochs. This allows us to directly compare the effects of concept intervention since each atom maps to a unique semantic class.

\subsection{Domain-Adapted Embedding}
We evaluate a domain-adapted variant of each embedding model via metric learning. Given the original embeddings $\mathbf{A} \in \mathbb{R}^{N \times d}$, we learn a reweighting matrix $\mathbf{M} \in \mathbb{R}^{d \times d}$ to obtain transformed embeddings $\tilde{\mathbf{A}} = \mathbf{A}\mathbf{M}$. After applying row-wise $\ell_2$ normalization to yield $\hat{\mathbf{A}}$, we compute the induced cosine similarity matrix $\mathbf{S} = \hat{\mathbf{A}}\hat{\mathbf{A}}^\top$. 

To establish a semantic ground truth, we define a target similarity matrix $\mathbf{T}$ based on the Jaccard index over the multi-hot concept matrix $\mathbf{Y}$:
$$
\mathbf{T}_{ij} = \frac{(\mathbf{Y}\mathbf{Y}^\top)_{ij}}{\|\mathbf{Y}_i\|_1 + \|\mathbf{Y}_j\|_1 - (\mathbf{Y}\mathbf{Y}^\top)_{ij}}.
$$
Alignment quality is then evaluated over all unordered pairs $i<j$ using Mean Squared Error (MSE) and the Pearson correlation coefficient ($\rho$):
$$
\mathrm{MSE} = \mathbb{E}_{i<j}\!\left[(\mathbf{S}_{ij}-\mathbf{T}_{ij})^2\right], \quad \rho = \mathrm{corr}(\mathbf{S}_{ij}, \mathbf{T}_{ij}).
$$
We train for 10 epochs with a learning rate of $10^{-3}$ and a batch size of 32.

\section{SAE Hyperparameter Selection}
To identify optimal training configurations for each SAE variant, we performed a grid search over key hyperparameters. We evaluated three SAE architectures, TopK SAE, JumpReLU SAE, and Vanilla SAE, across six pretrained embedding models: ResNet, SigLIP, DINOv2, DINOv3, ViT, and ConvNeXt. The number of learned concepts was fixed at 10 (with top-$k$ set to 4 for the TopK SAE), and the search space covered batch sizes in \{4, 8, 16, 32, 64\}, training epochs in \{10, 20\}, and learning rates in \{$10^{-4}$, $5 \times 10^{-4}$, $10^{-3}$, $5 \times 10^{-3}$\}. All SAEs were trained with the Adam optimizer; TopK SAEs used a mean-squared error loss augmented with a dead-code reactivation penalty, while Vanilla SAEs used mean squared error with an $\ell_1$ sparsity penalty ($\lambda = 10^{-4}$). For each configuration, we generated 3,000 composite grid images from CIFAR-10, split into 2,000 training and 1,000 test samples, and repeated the entire procedure 50 times with distinct random seeds to account for dataset variability. Each repetition was run in both the domain-adapted and pre-trained embedding settings. Configurations were evaluated using the average class identifiability accuracy, which measures, for each of the 10 classes, how well a single learned concept can distinguish that class via a threshold on the concept activation. The best hyperparameters can be observed in Table \ref{tab:hyper_results}.

\section{Experimental Details}
For all experiments involving the collage dataset, we train Top-K, JumpReLU, and Vanilla Sparse Autoencoders (SAEs) on 2,000 randomly generated collages. Activations for these autoencoders are extracted from six distinct vision backbones: DINOv2, DINOv3, ResNet50, ConvNeXt, ViT, and SigLIP. During training, SAE variants used the same loss terms as in the hyperparameter tuning and used the best parameters obtained over the search.

For the first three experiments of Faithfulness Eval, we construct a test set of 2,000 images and apply our perturbations to each image individually. We repeat this entire process across 20 independent runs. In the case of OOD, the reported results represent an average across all evaluated SAE architectures and only the pretrained embeddings. For completeness, we also include the Domain-Adapted results in Table \ref{tab:ood_perturbations_DA}. We observe that when we explicitly arrange our embeddings by our synthetic concepts, the overall invasiveness of each perturbation reduces, especially that of our latent space perturbations. 

In the Linear Recoverable Experiment, we repeat the experiment process across 10 independent runs for 100, 500, 1000, 1500, and 2000 concepts and only for the Top-k SAE given its superior performance in the Concept Verification section. We report only the results for our method with 1500 concepts, since in the OOD Analysis, performance was relatively stable thereafter. We report the full results for each method (those that depended on the number of concepts) in Tables \ref{tab:sim_recover_appendix1} and \ref{tab:sim_recover_appendix2}. Our method performed the best for all numbers of concepts except 2000. We do note that the results obtained for this experiment were tests over different pairings, hence our choice to use the Wilcoxon rank-sum test. The final version will unify them.

\section{Dissimilarity Behavior}

In Figure \ref{fig:discordant_example_1}, our query (Fig. \ref{fig:discord_query}) and reference (Fig. \ref{fig:discord_ref}) share two concepts and differ by two concepts. We perturb the "truck" concept (shown masked in Fig. \ref{fig:discord_masked}) and record the similar change across various perturbation methods. A truck and a car share many lower-level visual features (e.g., wheels, windows, chassis). We noticed that removing the truck causes the overall similarity between the two collages to drop, despite being distinct, dissimilar high-level concepts; the truck and the car contribute to the pair's similarity. Because high-level concepts can be viewed as collections of lower-level features, isolating a concept's exact contribution, pertaining to dissimilar concepts, becomes more complex due to this underlying feature overlap. For this reason, we did not include analysis of concepts contributing to dissimilarity in our Perturbation Effect on Similarity experiment.

It is important to note that this pronounced feature overlap is partly an artifact of our synthetic and controlled setup, where we manually constrained the dictionary to only 10 atoms. In practice, the dictionary size would typically be set much higher, allowing the learned concepts to represent more disentangled, granular features rather than broad, overlapping classes. However, determining the optimal number of concepts remains a notoriously difficult hyperparameter to tune. Consequently, the feature entanglement observed in this example—where distinct high-level concepts share underlying visual foundations—is a realistic challenge likely to manifest in broader applications.

\section{OOD Full Experimental Results}

\subsection{OOD Metric Details}
Formally, we report the 1-Wasserstein \cite{wasserstein} distance, $\mathcal{W}_1$, between the empirical distribution of original activations $\mathbf{A}$ (denoted $\mu_a$) and the empirical distribution of the perturbed activations (denoted $\mu_p$), written as $ \mathcal{W}_1(\mu_a, \mu_p)$. Complementary to this, the OOD score measures the local plausibility of individual perturbed points. We adapt Deep-KNN \cite{sun2022knnood} by comparing $k$-nearest neighbor distances in the original activations to those of perturbed embeddings after $\ell_2$ normalization. Specifically, we compute the mean closest neighbor distance within the original embeddings and its standard deviation, apply our perturbation, and recompute the mean closest neighbor distance to the original neighbors.

\subsection{Latent Perturbation Intuition}\label{exp:perturbation_intuition}

During our analysis, we observed that latent concept perturbations are less invasive to a foundation model's embedding space than standard input space perturbations. However, the interventions applied to the collages rendered them completely unidentifiable, resulting in minimal feature preservation. This extreme degradation is not strictly necessary; many input perturbations, such as blurring, can be calibrated to retain some degree of the underlying visual features. To address this disparity, we introduce an intensity parameter, $\alpha \in [0, 1]$, which explicitly controls the strength of the perturbation. At $\alpha = 0.1$, the majority of the original features are preserved, whereas $\alpha = 1.0$ recovers the setting evaluated in our initial experiments. By systematically varying $\alpha$ in increments of $0.1$, we aim to analyze and contextualize the behavior of our concept perturbations relative to standard input space perturbations. Figure \ref{fig:fid_ood_models} illustrates the impact of these scaled perturbations, with the top row detailing the $\mathcal{W}_1$ and the bottom row displaying the OOD scores across all models.

\begin{table*}[ht]
  \centering
  
  % --- FIRST TABLE ---
  \begin{tabular}{lcccc}
    \toprule
    Explanation Method & 100 & 500 & 1000 & 1500 \\
    \midrule
    CAV Multiplication & 0.86 / 0.84 & 0.88 / 0.85 & -866.35 / -898.69 & -6778.55 / -5838.94 \\
    CAV Subtraction & -0.05 / -0.05 & -0.28 / -0.26 & -0.87 / -0.84 & -3.01 / -3.10 \\
    Blur (ours) & 0.81 / 0.87 & 0.83 / 0.89 & 0.80 / 0.87 & 0.73 / 0.82 \\
    Blur Q Only & 0.66 / 0.73 & 0.62 / 0.71 & 0.43 / 0.57 & -0.33 / 0.04 \\
    Blur R Only & 0.66 / 0.73 & 0.62 / 0.71 & 0.45 / 0.58 & -0.48 / -0.01 \\
    Ours & 0.90 / 0.93 & 0.96 / 0.97 & 0.97 / 0.95 & 0.94 / 0.87 \\
    Ours Q only & 0.76 / 0.81 & 0.90 / 0.91 & 0.88 / 0.86 & 0.30 / -0.41 \\
    Ours R only & 0.76 / 0.81 & 0.90 / 0.90 & 0.90 / 0.86 & 0.36 / -0.09 \\
    Ours (both) & 0.93 / 0.85 & 0.96 / 0.97 & 0.96 / 0.98 & 0.91 / 0.94 \\
    \bottomrule
  \end{tabular}
  \caption{Averaged $R^2$ performance on Linear Recoverability experiment (Excluding 2000 Concepts). Values reported as Cosine / Euclidean.}
  \label{tab:sim_recover_appendix1}

  \vspace{2em} % Adds vertical space between the two tables

  % --- SECOND TABLE ---
  \begin{tabular}{lc}
    \toprule
    Explanation Method & 2000 \\
    \midrule
    CAV Multiplication & -419202.28 / -268497.29 \\
    CAV Subtraction & -175.95 / -183.79 \\
    Blur (ours) & 0.66 / 0.77 \\
    Blur Q Only & -0.66 / -0.23 \\
    Blur R Only & -0.71 / -0.31 \\
    Ours & -14.81 / -230.06 \\
    Ours Q only & -113.48 / -561.86 \\
    Ours R only & -24.75 / -93.08 \\
    Ours (both) & -0.92 / -23.95 \\
    \bottomrule
  \end{tabular}
  \caption{Averaged $R^2$ performance on Linear Recoverability experiment (2000 Concepts Only). Values reported as Cosine / Euclidean.}
  \label{tab:sim_recover_appendix2}
  
\end{table*}

\end{document}